\documentclass[11pt]{article}

\usepackage{theapa, rawfonts}
\usepackage{latexsym}

\usepackage{enumerate}
\usepackage{paralist}
\usepackage{multirow}
\usepackage{multicol}

\usepackage[ruled,linesnumbered]{algorithm2e}

\usepackage{amsmath}
\usepackage{amsthm}
\usepackage{amssymb}
\usepackage{mathtools}
\usepackage{bbm}
\usepackage{bm}
\usepackage{nicefrac}
\usepackage[noabbrev]{cleveref} 
\usepackage{thmtools}
\usepackage{thm-restate} 

\newtheorem{theorem}{Theorem}
\newtheorem{lemma}{Lemma}
\newtheorem{definition}{Definition} 

\newtheorem{remark}{Remark}
\newtheorem{corollary}{Corollary}

\newtheorem{example}{Example}

\usepackage{booktabs}

\usepackage{comment}
\long\def\comment#1{}

\newcommand{\sel}{\ensuremath{\mathit{sel}}}

\newcommand{\leanparagraph}[1]{\smallskip\noindent\textbf{#1}.}

\newcommand{\NP}{\textmd{\rm NP}}

\newcommand{\coNP}{\textrm{coNP}}
\newcommand{\Pol}{\textmd{\rm P}}
\newcommand{\FPNPPar}{\textmd{\rm FP}$^\NP_\|$}
\newcommand{\PNPPar}{\Pol$^\NP_\|$}

\newcommand{\naf}{\mathbf{\mathop{not}\,}}

\usepackage{url}
\usepackage{color}

\newcounter{myenumctr}

\newenvironment{myenumerate}{\begin{list}{ (\arabic{myenumctr})\ }{\usecounter{myenumctr}
\setlength{\leftmargin}{0pt}
\setlength{\itemsep}{0pt}
\setlength{\itemindent}{1.3\labelwidth}}}
{\end{list}}

\definecolor{darkgreen}{rgb}{0,0.7,0}
\definecolor{darkblue}{rgb}{0,0,0.7}
\definecolor{darkred}{rgb}{0.7,0,0}

\newcommand{\rev}[2]{{\color{blue}#2}}

\newif\ifdraft\draftfalse

\newif\ifinlineref\inlinereffalse
\newif\iffinal\finalfalse
\newif\ifextended\extendedfalse
\newif\ifdotikz\dotikzfalse

\ifdraft
\long\def\comment#1{{\bf\color{blue}#1}}
\else
\renewcommand{\comment}[1]{}
\fi

\ifdotikz
\usepackage{tikz}
\pgfrealjobname{well-justified-FLP}
\usetikzlibrary{shapes,arrows,backgrounds,chains,%
matrix,patterns,arrows,decorations.pathmorphing,decorations.pathreplacing,%
positioning,fit,calc,decorations.text,shadows%
}
\else
\long\def\beginpgfgraphicnamed#1#2\endpgfgraphicnamed{\includegraphics{#1}}
\fi

\allowdisplaybreaks

\begin{document}

\title{~\\[-2\baselineskip]
Refining Gelfond’s Rationality Principle: 
Towards More  \\
Comprehensive Foundational Principles 
for Answer Set Semantics\thanks{This article is a  
significantly extended version of a paper 
presented at IJCAI-2022 \cite{ShenEiter22}.}
} 

\author{Yi-Dong Shen$^1$ and Thomas Eiter$^2$\\
$^1$ 
Institute of Software,
Chinese Academy of Sciences, Beijing, China\\
E-mail: ydshen@ios.ac.cn \\
$^2$ Institute of Logic and Computation, Vienna University of Technology (TU Wien)\\ 
Favoritenstra{\ss}e 9-11, A-1040 Vienna, Austria \\
E-mail: eiter@kr.tuwien.ac.at
}

\date{}

\maketitle            

\begin{abstract}
Non-monotonic logic programming 
is the basis for a declarative problem
solving paradigm known as answer set programming (ASP), 
where solutions to
problems are encoded by intended models (answer sets) of a logic program.
Departing from the seminal definition by \citeA{GL88} 
for simple normal logic programs, various answer set semantics have
been proposed for extensions such as disjunctive logic programs and
epistemic logic programs. In the latter case, the semantics
consists of sets of
answer sets called world views.
In general, it seems infeasible
to formally prove 
whether a proposal
defines answer sets/world views
that correspond exactly to the solutions
of any problem intuitively represented by a logic program.
It is thus necessary to develop some 
general principles and use them as a baseline
to intuitively compare and assess different answer set
  resp.\ world view semantics.
Towards such a baseline, we consider two
important questions: (1) Should the minimal model property,  
constraint monotonicity 
and foundedness as defined in the literature
be mandatory conditions for
an answer set semantics in general?
(2) If not, what other properties could be considered 
as alternative principles for answer set semantics? 
We address the two questions with several contributions.
First, we use examples to demonstrate that requiring
minimal models, constraint monotonicity
and foundedness as mandatory conditions 
may exclude expected answer sets 
for some simple disjunctive programs and
world views for some epistemic specifications. 
Second, we evolve the {\em Gelfond answer set (GAS)
principles} \cite{Gelfond2008} for answer set construction
by refining the Gelfond's rationality principle
to well-supportedness, minimality  
w.r.t.\ negation by default, and minimality  
w.r.t.\ epistemic negation.
The principle of well-supportedness guarantees
that every answer set is constructible from if-then  
rules obeying a level mapping
and is thus free of circular justification,
while the two minimality principles 
ensure that the formalism
minimizes knowledge both at the level
of answer sets and of world views.
We thus propose to consider the three
refined GAS principles as alternative 
principles for answer set semantics
in general and for answer set and world view
construction in particular.
Third, to embody the refined GAS principles,
we extend the notion of well-supportedness  
substantially to answer sets 
and world views, respectively.
Fourth, we propose to define answer set semantics
in terms of the refined GAS principles,
called rational answer set and rational
world view semantics, respectively.
Fifth, we use the refined GAS principles
as an alternative baseline 
to intuitively assess the existing answer set semantics
whether they satisfy them 
(are compliant with them) or even fully embody them.
Finally, we analyze the computational complexity of 
well-supportedness and the rational answer set and world view
semantics, revealing them as expressive hosts for
problem solving.
\bigskip
~
\end{abstract}

\section{Introduction}
\label{int}

Initially conceived as an approach to reasoning under
 incomplete information, non-monotonic logic 
 programming has evolved
 into a declarative problem solving paradigm known as
answer set programming (ASP) 
\cite{MT99,Nie99,Lif02}, 
which is
specifically oriented towards modeling  
combinatorial search problems arising in 
areas of AI such as planning, reasoning about 
actions, diagnosis,
and beyond \cite{gelfond-kahl2014,bookLifschitz19}.
Broadly, the idea of ASP is to represent problems by logic
programs whose meaning/semantics is defined by some intended logic models 
called {\em stable models} \cite{GL88} or {\em answer sets} \cite{GL91},
or by some intended epistemic models called {\em world views} \cite{Gelfond91},
which correspond to the solutions of a problem. 

In the very beginning of the ASP evolution, 
only a class of logic programs, called {\em simple normal programs},
were considered \cite{GL88}, which consist of rules of the 
form
\begin{equation}
\label{normalLP}
A \leftarrow B_1 \wedge \cdots \wedge B_n \wedge \neg C_1 \wedge \cdots \wedge \neg C_k
\end{equation}
where $\leftarrow$ is an if-then rule operator,
$A$ is an atom or $\bot$, and the $B_i$'s and $C_j$'s are atoms.\footnote{In \cite{GL88}, ``,'' instead of
  ``$\wedge$'' was used for logical conjunction.}  Intuitively, rule
(\ref{normalLP}) says that 
{\em if} the body condition $B_1 \wedge \cdots
\wedge B_n \wedge \neg C_1 \wedge \cdots \wedge \neg C_k$ 
has been inferred to be true by any other rules in a logic program,
{\em then} the head $A$ can
be inferred \cite{JanhunenN16}. A
rule is a {\em constraint} when its head is $\bot$.

\citeA{GL88} defined the seminal answer set semantics,
called {\em GL$_{nlp}$-semantics},
for simple normal programs
by means of a program transformation. 
For a simple normal program $\Pi$ and an interpretation $I$, 
the {\em GL-reduct} $\Pi^I$ of $\Pi$ w.r.t.\ $I$ is
obtained from $\Pi$ by removing first all rules of form
(\ref{normalLP})
containing some $C_i\in I$ and then all $\neg C_i$ from the
remaining rules. Then
$I$ is a {\em stable model} or an {\em answer set} 
of $\Pi$ under the GL$_{nlp}$-semantics
if $I$ is the least model of $\Pi^I$. 

To date, the GL$_{nlp}$-semantics is the commonly agreed 
answer set semantics for the class of simple normal programs and
is the backbone for various extensions, 
such as disjunctive logic programs \cite{GL91} and
epistemic logic programs \cite{Gelfond91}.

In order to model incomplete information, 
\citeA{GL91} extended simple normal programs to
a class of logic programs, called
{\em simple disjunctive programs},  
which consist of rules of the form
\begin{equation}
\label{disjunctLP}
A_1 \mid \cdots \mid A_m \leftarrow B_1 \wedge \cdots \wedge B_n \wedge \neg C_1 \wedge \cdots \wedge \neg C_k
\end{equation}
where $\mid$ is a disjunctive rule head operator, 
$m\geq 1$, and the $A_i$'s, $B_j$'s and $C_l$'s 
are atoms except that when $m=1$,
$A_m$ can also be $\bot$.
Rule (\ref{disjunctLP}) says that 
{\em if} the body condition $B_1 \wedge \cdots
\wedge B_n \wedge \neg C_1 \wedge \cdots \wedge \neg C_k$ 
has been inferred to be true by any other rules,
{\em then} some $A_i$ in the head $A_1 \mid \cdots \mid A_m$ can
be inferred. 
Note that when $m=1$ for all rules,
a simple disjunctive program 
reduces to a simple normal program.

\citeA{GL91} then upgraded the GL$_{nlp}$-semantics
to simple disjunctive programs, leading to
an answer set semantics,
called {\em GL-semantics}. 
For a simple disjunctive program $\Pi$
and an interpretation $I$, the GL-semantics defines
$I$ to be an answer set of $\Pi$ 
if $I$ is a minimal model of the GL-reduct $\Pi^I$. 
When $\Pi$ is a simple normal program,
the GL-semantics coincides with the GL$_{nlp}$-semantics.

In order to enable representing
incomplete knowledge 
in a logic program that has more than one answer set,
\citeA{Gelfond91} extended simple disjunctive programs to
{\em epistemic specifications}
with epistemic modal operators ${\bf K}$ and ${\bf M}$,
which consist of rules of the form
\begin{equation}
\label{equa-1}
L_1\, \mid \cdots \mid \, L_m \leftarrow G_1 \wedge \cdots \wedge G_n
\end{equation}
where $m\geq 1$, each $L_i$ is an atom
except that when $m=1$,
$L_m$ can also be $\bot$,
and each $G_j$ is an atom $A$ or its 
negation $\neg A$ or a modal literal  
${\bf K} A$, $\neg {\bf K} A$,  ${\bf M} A$ or  
$\neg {\bf M} A$.
Rule (\ref{equa-1}) is a 
{\em subjective constraint} if the head is $\bot$ and
each $G_j$ is a modal literal.

\citeA{ShenEiter16} further extended 
epistemic specifications to general {\em epistemic programs}
with a modal operator $\mathbf{not}$ for epistemic negation,
which consist of rules of the form 
\begin{equation}
\label{equa-ep}
H_1\mid \cdots \mid H_m \leftarrow B
\end{equation}
where $m\geq 1$, $B$ and each $H_i$ 
are first-order formulas extended with
epistemic negations of the form $\naf F$,
and $F$ is a first-order formula. 
Intuitively, $\naf F$ expresses that 
there is no evidence proving that $F$ is true.
In semantics, for a collection $\cal A$ of interpretations,
$\naf F$ is true in $\cal A$ if
$F$ is false in some $I\in {\cal A}$.
The modal operators {\bf K} and {\bf M} can be viewed as
shorthands for $\neg \mathbf{not}$ and 
$\naf \neg$, respectively.
By an {\em epistemic-free program}, we refer to  
an epistemic program free of modal operators.

The semantics of an epistemic program is defined by 
some intended epistemic models, called world views.
\citeA{Gelfond91} defined the first world view semantics,
called {\em G91-semantics}, for 
an epistemic specification $\Pi$ by means of a program transformation.
Given a collection $\cal A$ of interpretations, $\Pi$ is transformed
into a {\em modal reduct} $\Pi^{\cal A}$ 
w.r.t.\ $\cal A$ by 
first removing all rules of form 
(\ref{equa-1}) containing a modal literal $G$ 
that is not true in $\cal A$, 
then removing the remaining modal literals. 
$\cal A$ is a {\em world view} of $\Pi$ under the G91-semantics
if $\cal A$ coincides with the collection of all answer sets of $\Pi^{\cal A}$ 
under the GL-semantics. 

Both the GL-semantics 
for simple disjunctive programs and the G91-semantics
for epistemic specifications have limitations, so
some other answer set semantics and world view semantics 
have been proposed, respectively.
For example,
\citeA{ShenE19} observed that the requirement of the GL-semantics, 
i.e., an answer set $I$ of a simple disjunctive program $\Pi$
should be a minimal model of rules of the GL-reduct $\Pi^I$, 
may sometimes be too strong because
some expected answer sets
may not be minimal models of the GL-reduct.
To address this issue, they defined 
an alternative answer set semantics, called {\em determining
inference semantics} (DI-semantics for short).
The DI-semantics applies to all epistemic-free programs. In particular,
for simple disjunctive programs, it is a relaxation of the GL-semantics
such that an answer set under the GL-semantics
is also an answer set under the DI-semantics,
but the converse does not hold in general. 

For epistemic specifications,
it was observed that the G91-semantic has both 
{\em the problem of unintended world views 
with recursion through} {\bf K}
and {\em the problem due to recursion through} {\bf M} \cite{Gelfond2011,Kahl14}.
As an example of the first problem, 
the epistemic specification $\Pi = \{p\leftarrow {\bf K} p\}$
has two world views ${\cal A}_1 = \{\emptyset\}$
and ${\cal A}_2 = \{\{p\}\}$ under the G91-semantic; 
however, as commented in \cite{Kahl14},  
${\cal A}_2$ is undesired.
For the second problem,  
$\Pi = \{p\leftarrow {\bf M} p\}$
has two world views ${\cal A}_1 = \{\emptyset\}$
and ${\cal A}_2 = \{\{p\}\}$ under the G91-semantics, where
${\cal A}_1$ is undesired.
Several alternative world view semantics
were proposed to address the two problems \cite{KahlGelfond2015,Fa2015,ShenEiter16,SuCH20}.
For example, \citeA{ShenEiter16} proposed to apply 
epistemic negation to minimize the knowledge in
world views and presented a definition 
of world views, which applies to all epistemic programs
and avoids both the problem with recursion through {\bf K} and 
the problem through {\bf M}. 

Recall that the aim of answer set programming
is to encode problems in AI
as logic programs whose 
answer sets or world views
correspond to the solutions of a problem.
Different ASP semantics, i.e., answer set semantics for 
epistemic-free programs and world view semantics
for epistemic programs, may be defined to
approach this goal.
In general, for a proposed ASP semantics,
it does not seem possible to formally prove 
whether the semantics defines answer sets/world views
that correspond exactly to the solutions
of any problem represented by a logic program.
Therefore, it is necessary 
to develop some general principles 
and use them as a baseline  to intuitively compare and assess
different ASP semantics, in particular beyond
simple normal programs.

\begin{remark}
As we know,  there are a few people in the ASP community who have gotten used to 
the GL-semantics such that they equate answer 
set semantics for simple disjunctive programs with the GL-semantics.
In their opinion,
the GL-semantics should be taken as the only answer set semantics
 for simple disjunctive programs, 
and any other semantics that may yield different answer sets 
than the GL-semantics
should not be considered as an answer set semantics.
This is arguably a misunderstanding. On the one hand,
as we can determine, there is no existing literature, including \cite{GL91},
in which the GL-semantics is formally proved to be 
an answer set semantics whose answer sets
correspond exactly to the solutions
of any problem represented by a simple disjunctive program.
On the other hand, although many examples have 
been presented in the literature \cite{ErdemGL16}, 
showing successful uses of the GL-semantics
for some real life applications, some counter-examples have been
uncovered \cite{ShenE19}, showing that the requirement of the GL-semantics, 
i.e., answer sets must be minimal models of GL-reducts, 
is too strong and may exclude expected answer sets 
for some simple disjunctive programs.
\end{remark}

Early works towards a baseline focused on properties
of the GL-semantics for simple disjunctive programs, 
mainly including:
\begin{enumerate}
\itemsep=1pt 
\item
{\bf The minimal model property (MM):} Every answer set of 
$\Pi$ is a minimal model of $\Pi$ \cite{GL91}.
\item
{\bf Constraint monotonicity (CM):}
For every constraint $C$,
answer sets of $\Pi\cup \{C\}$ 
are answer sets of $\Pi$ satisfying $C$ \cite{lifs-turn-94,LTT99}. 
\item
{\bf Foundedness (FN):}  Answer sets of $\Pi$ 
contain no unfounded set \cite{LeoneRS97},
where a set $X$ of atoms in an interpretation $I$ is an {\em 
unfounded set} of $I$ if $\Pi$ has no rule of form (\ref{disjunctLP})  
with $\{A_1,\cdots, A_m\} \cap X \,{\neq}\, \emptyset$ satisfying (f1)-(f3):
\begin{compactitem}
    \item[](f1) $I$ satisfies the rule body condition; 
    \vspace{1pt} 
    
    \item[](f2) $\{B_1,\cdots, B_n\}\cap X = \emptyset$; and 
      \vspace{1pt} 
      
    \item[](f3) $(\{A_1,\cdots, A_m\}\setminus X) \cap I = \emptyset$.
\end{compactitem}
\end{enumerate}

Recent works \cite{iclp-KahlL18,CabalarFC20aij}
extended the notions of constraint monotonicity
and foundedness from simple disjunctive programs
to epistemic specifications and 
used them as criteria/intuitions to compare 
different world view semantics  
on how they comply with these properties.
\citeA{lpnmr-CabalarFC19} also defined
a stronger property than constraint monotonicity, 
called {\em epistemic splitting}, intending
to establish a ``minimal requirement" 
that a world view semantics must satisfy. 

However, it turns out that for epistemic-free programs, the DI-semantics
does not satisfy the above three properties \cite{ShenEiter22},
and for epistemic programs, most of the existing
world view semantics, such as those defined in 
\cite{KahlGelfond2015,Fa2015,ShenEiter16,SuCH20}, 
violate the properties of constraint monotonicity
and foundedness; see
\cite{abs-2108-07669} for a survey.

This leads us to the following important question:  
\begin{quote}
{\bf Q1:} Should the minimal model property,  
constraint monotonicity 
and foundedness be mandatory conditions 
for an ASP semantics in general?
\end{quote}

This question is challenging and has  
incurred a debate in recent ASP research
\cite{shen-eiterCns20,Costantini21,lpnmrCostantiniF22,SuICLP2021}.

In this paper, we address the above question 
by demonstrating with examples that requiring
minimal models, constraint monotonicity
and foundedness as mandatory conditions 
may exclude expected answer sets 
for some simple disjunctive programs and
world views for some epistemic specifications. 

\medskip

Given that minimal models, constraint monotonicity 
and foundedness may not be mandatory conditions
for an ASP semantics in general, 
one would naturally ask another important question:
\begin{quote}
{\bf Q2:} What other properties could be considered 
as alternative principles for an ASP semantics in general?
\end{quote} 

As far as we can determine,
no solution to this question has been reported in the
literature except for the two principles by \citeA{Gelfond2008} 
as a guide for the construction of an answer set.
Recall that an answer set semantics 
assigns to a logic program $\Pi$ 
a collection of interpretations
called answer sets, which 
can be built on the basis of rules of $\Pi$. 
Specifically,
\citeA{Gelfond2008} suggests that the 
construction of an answer set
$I$ should be guided by the following two principles,
which we refer to as the {\em Gelfond answer set principles}
or shortly {\em GAS principles}:
\begin{description}
 \item[(P1)] $I$ must satisfy all rules of $\Pi$; and
 \item[(P2)] it should adhere to the {\em rationality principle}, 
i.e., one shall not believe anything one is not forced to believe.
\end{description}

Principle (P1) requires that every answer set 
must be a model of $\Pi$, meaning that the construction of 
an answer set is a generalized boolean satisfiability problem (SAT)
from a propositional theory to a logic program,
while (P2) says that the model selection follows 
the rationality principle.
Therefore, the GAS principles can be 
shortly formulated as follows: 
\begin{quote}
{\bf Gelfond answer set principles} 
 $=$ {\bf SAT} $+$ {\bf Rationality principle}.
\end{quote}

Note that the Gelfond's rationality principle
is just an abstract notion, which generally requires knowledge
minimization (i.e., only believe those things 
one has to believe) in nonmonotonic reasoning. 
One may interpret it in different ways,
thus leading to different answer set semantics.
For instance, in the GL-semantics the rationality principle
may be interpreted with the minimal model 
and foundedness properties, i.e., one shall not believe those 
ground atoms which are not in minimal models
and those atoms which are in some unfounded sets.

\leanparagraph{Our answer to Q2}
In this paper, we propose an answer to question Q2
by interpreting the Gelfond's rationality principle
with the following three refined principles: 
\begin{description}
\itemsep=1pt
 \item[(RP1)]
 {\bf Minimality of answer sets w.r.t.\ negation by default:} 
 Let $I$ be an answer set of $\Pi$ and $p$ be a ground atom.
 Then $\neg p$ is assumed by default 
 to be true in $I$ unless that incurs inconsistency.
 
 \item[(RP2)]
 {\bf Minimality of world views w.r.t.\ epistemic negation:} 
 Let ${\cal A}$ be a world view of $\Pi$ and $\naf F$ be an epistemic negation
 occurring in $\Pi$.
 Then $\naf F$ is assumed by default 
 to be true in ${\cal A}$ unless that incurs inconsistency.

  \item[(RP3)]
{\bf Well-supportedness of answer sets:} Every answer set $I$ of $\Pi$ has
a strict well-founded partial order on atoms in $I$ w.r.t. $\Pi$.
\end{description}

By minimality of answer sets w.r.t.\ negation by default,
we interpret the Gelfond's rationality principle as follows: 
for every ground atom $p$, we shall not
believe $p$ to be true in an answer set unless we are forced 
to believe it, i.e., unless assuming $\neg p$ to be true incurs inconsistency.
This will lead to a maximal set of ground atoms
being false in an answer set and thus no answer set
is a proper subset of another one.
This principle aims to minimize the knowledge 
of a logic program $\Pi$ at the level of answer sets
by making as many ground atoms
false in an answer set as possible.
As an example, consider 
$\Pi=\{p\leftarrow p\}$, which has two models
$I_1=\emptyset$ and $I_2=\{p\}$.
It is the minimality principle w.r.t.\ negation by default (RP1)
that rejects $I_2$ to be an answer set of $\Pi$. 
Note that this minimality principle looks like the {\em closed 
world assumption} (CWA) \cite{Reiter77}
and the minimal model semantics \cite{minker82},
but unlike them it does not require that answer sets
must be minimal models.

By minimality of world views w.r.t.\ epistemic negation
we further interpret the Gelfond's rationality principle as follows:
for every epistemic negation 
$\naf F$ occurring in $\Pi$,
we shall not believe $F$
to be true in a world view unless we are forced 
to believe it, i.e., unless assuming $\naf F$ to be true incurs inconsistency.
This will lead to a maximal set of epistemic negations 
in $\Pi$ being true in a world view.
This principle aims to minimize the epistemic knowledge 
of a logic program $\Pi$ at the level of world views
by making as many epistemic negations 
in $\Pi$ true in a world view as possible. 
By the {\em epistemic knowledge} 
of $\Pi$ we refer to all formulas $F$ with an
epistemic negation $\naf F$ occurring in $\Pi$.
As an example, consider 
$\Pi = \{p\leftarrow {\bf M} p\} = \{p\leftarrow \naf\neg p\}$, 
which has two world views ${\cal A} = \{\{p\}\}$ 
and ${\cal A}_1 = \{\emptyset\}$.
It is the minimality principle w.r.t.\ epistemic negation (RP2)
that rejects ${\cal A}_1$ to be a world view of $\Pi$. 

Finally, by well-supportedness of answer sets
we intuitively interpret, based on the 
{\em if-then} rules of a logic program $\Pi$,
the Gelfond's rationality principle as follows:
for every ground atom $p$
we shall not believe $p$
to be true in an answer set unless we are forced 
to believe it, i.e., unless there is a 
rule $p\leftarrow body(r)$ in $\Pi$,
where the condition
$body(r)$ has already been inferred to be true
from other rules in $\Pi$. 
Recall that in general an if-then rule 
$head(r) \leftarrow body(r)$ 
in a logic program essentially differs from the material implication
$body(r)\supset head(r)$
in classical logic because rules induce a level mapping on
each answer set such that answers at upper levels are derived from
answers at lower levels in the way that if
the body condition $body(r)$ of a rule $r$ is true 
in answers at lower levels,
then infer the rule head $head(r)$.  
For example, we can infer $p$ from a formula $\neg p\supset p$
in classical logic, as $p$ is a logical consequence of $\neg p\supset p$;
however, we cannot derive $p$ from the rule
$p\leftarrow \neg p$ in logic programs.
Another example is that $\neg p\supset q$
and $\neg q\supset p$ are equivalent formulas 
in classical logic, but $q\leftarrow \neg p$
and $p\leftarrow \neg q$ are different rules 
in logic programs.

The principle of well-supportedness guarantees
that every answer set is constructible from if-then  
rules obeying a level mapping
and is thus free of circular justification,
while the two minimality principles 
ensure that the formalism
minimizes knowledge both at the level
of answer sets and of world views. 
We say that an answer set is {\em free of 
circular justification} if it has a level mapping.
With our understanding, these principles are 
essential for an ASP semantics in general and
answer set and world view construction in particular. 
Therefore, we propose 
to consider (RP1)-(RP3) as three alternative principles 
for ASP semantics
and evolve the GAS principles to a refined form
as follows:
\begin{quote}
{\bf Refined Gelfond answer set principles} 
$=$ {\bf SAT}
 $+$ {\bf Well-supportedness} (of answer sets) $+$ {\bf Minimality} (of answer sets w.r.t.\ negation by default
and of world views w.r.t.\ epistemic negation).
 \end{quote}

The refined Gelfond answer set principles
(refined GAS principles for short) characterize
answer set and world view construction as a generalized SAT problem, 
where the model selection
follows the three general principles (RP1)-(RP3).
This provides an alternative baseline 
(in contrast to using the 
minimal model property, constraint monotonicity 
and foundedness as a baseline) to intuitively compare and assess
ASP semantics.

\smallskip

The notion of well-supportedness
was first introduced by \citeA{Fages:JMLCS:1994} 
as a key property to characterize
answer sets of simple normal programs
under the GL$_{nlp}$-semantics.
For a simple normal program $\Pi$,
\citeA{Fages:JMLCS:1994} defined that
an interpretation $I$ is {\em well-supported} in $\Pi$
if there exists a strict well-founded
partial order $\prec$ on atoms in $I$ such that for every $p\in I$ 
there is a rule $p\leftarrow body(r)$ in $\Pi$, where 
$I$ satisfies $body(r)$ and $q\prec p$
for every positive literal $q$ in $body(r)$.
Clearly, a well-supported interpretation $I$ induces 
a level mapping on atoms in $I$
via the partial order $\prec$ w.r.t.\ $\Pi$, i.e.,
for every $p\in I$, there is a rule $p\leftarrow body(r)$ in $\Pi$
such that $I$ satisfies $body(r)$ and for every positive
literal $q$ in $body(r)$, $q$ is at a lower level than $p$ (due to $q\prec p$).
As a result, the property of
well-supportedness guarantees
that every answer set is free of circular justification.

Fages's definition of well-supportedness is limited 
to simple normal programs.
In order to take well-supportedness as a general principle
for ASP semantics,
we have to extend the notion to 
general cases, i.e., to define well-supportedness
for answer sets of all epistemic-free programs and
world views of all epistemic programs.
Apparently, the extension is non-trivial.
For a simple normal program,
as every rule body is a conjunction of literals,
a strict well-founded partial order $\prec$ 
on atoms in an interpretation
can easily be defined based on 
positive literals in rule bodies.
However, for an epistemic program that
may contain disjunctive rule heads, 
propositional/first-order formulas and
epistemic negations both in rule bodies and heads,
it is a challenging task to define
a strict well-founded partial order $\prec$
on atoms in an interpretation on the basis of rules.

\leanparagraph{Goal and main contributions} The goal of this paper is to address the above two 
fundamental problems Q1 and Q2.
For simplicity of presentation, 
we restrict to propositional logic programs. 
The proposed approaches can easily be 
extended to logic programs with first-order formulas
by considering a first-order language, where
all free variables in a logic program are
grounded in the way as described in \cite{ShenEiter16}.
Our main contributions are summarized as follows:
\begin{myenumerate}
\item 
We use examples to demonstrate that requiring
minimal models, constraint monotonicity
and foundedness as mandatory conditions may  
exclude expected answer sets 
for some simple disjunctive programs
and world views
for some epistemic specifications.
This answers question Q1.

\item
We evolve the GAS principles 
by refining the Gelfond's rationality principle to 
the principle of well-supportedness, 
the principle of minimality of answer sets w.r.t.\ negation 
by default and the principle of minimality of world 
views w.r.t.\ epistemic negation, 
and propose to consider them as
 three alternative principles for an ASP 
semantics in general and for
answer set and world view construction in particular. 
This offers an answer to question Q2.

\item
To embody the refined GAS principles
for answer set and world view construction,
we extend the notion of well-supportedness  
substantially to answer sets of all epistemic-free programs
and world views of all epistemic programs.
Briefly, we first consider
an epistemic-free normal program $\Pi$ 
with atomic rule heads, which consists 
of rules of the form $H\leftarrow B$,
where $H$ is an atom and $B$ is a formula.
For an interpretation $I$,
let $\neg I^-$ denote the set of negative literals of $I$. 
We extend Fages's definition of well-supportedness
in a way that (i) we use the entailment relation $\models$
for logical consequences to handle  
rule bodies with first-order formulas, and (ii) we
define that a model $I$ is well-supported in $\Pi$
if there exists a strict well-founded
partial order $\prec$ on atoms in $I$, where 
for every $p\in I$
there is a rule $p\leftarrow body(r)$ in $\Pi$ and
some $S\subset I$ such that
$S\cup \neg I^- \models body(r)$
and $q\prec p$ for every $q\in S$.
As $\prec$ is a strict well-founded
partial order, $p$ is not in $S$.
This means that a well-supported model $I$ has 
a level mapping on its atoms
via the partial order $\prec$ w.r.t. $\Pi$
and thus is free of circular justification. 
Next, 
for an epistemic-free normal program $\Pi$
consisting of rules of the form $H\leftarrow B$,
where both $H$ and $B$ are formulas,
we first define a strict well-founded
partial order $\prec_h$ on rule heads of $\Pi$ 
w.r.t.\ a model $I$, and then 
define the well-supportedness of $I$ based on $\prec_h$.
Third, 
for an epistemic-free program $\Pi$ with disjunctive rule heads
we present a new disjunctive head selection function
$sel$ and use it to transform $\Pi$ to a set of epistemic-free 
normal programs, called 
disjunctive program reducts w.r.t.\ a model $I$.
The new disjunctive head selection function overcomes an issue
with the existing selection function
defined in \cite{ShenE19}. 
We then define $I$ to be well-supported in $\Pi$
if $I$ is well-supported in some disjunctive program reduct.
Finally, 
for an epistemic program $\Pi$,
we transform $\Pi$ to a set of epistemic-free programs, called 
epistemic reducts w.r.t.\ an epistemic model ${\cal A}$, 
and define ${\cal A}$ to be well-supported in $\Pi$
if every $I\in {\cal A}$ is well-supported in some epistemic reduct.

\item
We then propose to define ASP semantics  
in terms of the refined GAS principles. 
Specifically, by considering the principles of well-supportedness
and minimality w.r.t.\ negation by default,
we obtain a new definition of answer sets,
called {\em rational answer set semantics},
for epistemic-free programs. 
Informally, a model $I$ of an epistemic-free program
is an answer set under the rational answer set semantics
iff $I$ is a minimal well-supported model.
And by further considering minimality 
w.r.t.\ epistemic negation,
we obtain a new definition of world views, 
called {\em rational world view semantics},
for epistemic programs. 
Informally, an epistemic model ${\cal A}$ 
of an epistemic program is a world view 
under the rational world view semantics
iff at the answer set level, ${\cal A}$ is the collection of
all minimal well-supported models, and
at the world view level, ${\cal A}$ satisfies a 
maximal set of epistemic negations 
occurring in the program.
The rational answer set
and world view semantics embody 
the refined GAS principles,
but they satisfy neither
the minimal model property  
nor constraint monotonicity nor foundedness.

\item
We use the refined GAS principles
as an alternative baseline (in contrast to using the 
minimal model property, constraint monotonicity 
and foundedness as a baseline) 
to intuitively assess
the major existing ASP semantics.
We consider two cases in terms of the extent to
which an answer set/world 
view semantics ${\cal X}$ adheres to 
the refined GAS principles.
Informally, we say that ${\cal X}$ {\em satisfies} 
the refined GAS principles
if every answer set/world 
view under ${\cal X}$ is a model/epistemic model
satisfying the three general principles  (RP1)-(RP3);
and that ${\cal X}$ {\em embodies} the refined 
GAS principles if the converse also holds,
i.e., a model/epistemic model
is an answer set/world view under ${\cal X}$ 
iff it satisfies these three general principles.

\item 
Finally, we analyze the computational complexity of 
well-supportedness and the rational answer set and world view semantics.
To this end, we consider deciding
  well-supportedness of a model 
  and an epistemic model, respectively, 
 and whether a given program has a rational answer set
  or a rational world view, respectively; furthermore, we consider
  brave and cautious inferences from the rational answer sets
  respectively truth of a formula in some rational world view of
  a program. We characterize the complexity of these problems for different classes of
  propositional logic programs. More in detail, 
we show that the complexity of deciding well-supportedness covers a
 range of complexity classes in the Polynomial Hierarchy (PH) from
 \Pol{} (for simple normal epistemic-free programs) up to $\Pi^p_3$
 (for epistemic programs). Furthermore, we present
   algorithms for computing a well-founded partial order that
 witnesses the well-supportedness of a model, and we discuss
   the complexity of this problem. Regarding the rational answer 
   set and world view semantics, it appears that for varying program 
   classes complexity classes in PH from $\NP$ up to
   $\Sigma^p_5$ are covered. This suggests that the rational answer 
   set and world view semantics
   have rich capacity as a backend formalism for problem solving inside
   polynomial space,
   rarely seen for similar propositional knowledge representation formalisms.
\end{myenumerate}
\noindent{\bf Organisation.} The remainder of this paper is structured as follows.
In Section \ref{Preliminaries}, 
we present preliminaries related to
epistemic programs and epistemic models.
In Sections \ref{sec-min-cons-found} and 
\ref{ConstraintMonoFoundednessforEpistemicSpec},
we introduce examples to demonstrate that requiring
minimal models, constraint monotonicity
and foundedness as mandatory conditions may  
exclude expected answer sets 
for some simple disjunctive programs and world views
for some epistemic specifications, respectively.
In Section \ref{sec-Well-SupportednessASP},
we define well-supportedness  
for answer sets of epistemic-free programs
and world views of epistemic programs, respectively.
In Section \ref{sec-rational-semantics},
we define in terms of the refined GAS principles 
the rational answer set semantics
for epistemic-free programs
and the rational world view semantics
for epistemic programs. 
In Section \ref{sec:extension-choiceCons},
we formally define answer set semantics for
epistemic-free programs extended with choice constructs.
In Section \ref{sec-existing-semantics}, 
we use the refined GAS principles
as an alternative baseline to assess
the major existing ASP semantics.
In Section \ref{sec-Computation and Complexity}, 
we analyze the computational complexity of 
well-supportedness and the rational answer set and world view semantics.
In Section \ref{RelatedW}, we mention related work.
Finally, in Section~\ref{conclusion} we conclude the paper.

In order not to distract from the flow of reading, 
proofs of all theorems are collected in Appendix~\ref{app:proofs}.

\section{Preliminaries}
\label{Preliminaries}
For our concerns, it suffices to consider 
propositional (ground) logic programs and assume a fixed
propositional language ${\cal L}_\Sigma$ with a countable set 
$\Sigma$ of propositional atoms.
A (propositional) formula is constructed 
as usual from atoms using connectives
$\neg$, $\wedge$, $\vee$, $\supset$, $\equiv$, $\top$ and $\bot$, 
where $\top$ and $\bot$ are two 0-place 
logical connectives 
expressing $true$ and $false$, respectively. 
A theory is a set of formulas.
A literal is either an atom $a$ (positive literal) 
or its negation $\neg a$ (negative literal).
Any subset $I$ of $\Sigma$ is called an {\em interpretation};
$I$ {\em satisfies} an atom $a$
(or $a$ is true in $I$)
if $a\in I$, and $I$ satisfies $\neg a$ 
(or $a$ is false in $I$) if $a\not\in I$.
The notions of {\em satisfaction} and 
{\em models} of a formula in $I$
are defined as usual in classical logic. 
A theory $T$ {\em entails} a formula $F$, 
denoted $T\models F$,
if all models of $T$ are models of $F$. 
Two theories $T_1$ and $T_2$ are  
{\em logically equivalent}, 
denoted $T_1\equiv T_2$,
if $T_1\models T_2$ and $T_2\models T_1$.
For an interpretation $I$,
we let $I^- = \Sigma \setminus I$ 
and~$\neg I^- = \{\neg a\mid a\in I^-\}$. 

The above propositional language ${\cal L}_\Sigma$
is defined in classical logic, so for any formula $F$,
$F\vee \neg F$ is a tautology, i.e., 
$F\vee \neg F \equiv \top$, and
double negatives cancel each other out, i.e.,
$\neg\neg F\equiv F$.

We use propositional formulas in classical logic as building blocks to define logic programs,
and further extend  the 
language ${\cal L}_\Sigma$ to include 
the if-then rule operator $\leftarrow$, the disjunctive 
rule head operator $\mid$,
and the epistemic negation operator $\naf$.\footnote{In \cite{GL91}, {\em default negation} of the form $not\ A$,
where $A$ is an atom,
was used in rule bodies. In this paper, we use $\neg A$
to denote the negation of $A$ as in classical logic,
and achieve the effect of default negation 
by minimizing answer sets with negation by default.}
An {\em epistemic negation} is of the form
$\naf F$, where $F$ is a propositional formula.
The modal operators {\bf K} and {\bf M}
are viewed as shorthands for $\neg \naf$ and $\naf \neg$,
respectively. 
An {\em epistemic formula} is constructed 
from atoms and epistemic negations 
using the above connectives.\footnote{Note that 
like the approaches of \cite{Gelfond2011,KahlGelfond2015},
where modal operators {\bf K} and {\bf M} are not nested in logic programs,
we do not consider 
nested epistemic negations like $\naf(\naf F)$.
}

We then use epistemic formulas to construct rules and logic programs.

\begin{definition}
\label{def-episP}
An {\em epistemic program} 
consists of rules of form (\ref{equa-ep}),
where $m\geq 1$, and $B$ and each $H_i$ are epistemic formulas.
\end{definition}

An epistemic program 
is an {\em epistemic specification} 
if every rule
is of form (\ref{equa-1}).
It is an {\em epistemic-free program} 
if it contains no epistemic negation,
and is an {\em epistemic-free normal program} 
if additionally $m=1$ for every rule.
It is a {\em simple disjunctive program} 
if every rule
is of form (\ref{disjunctLP}), and
is a {\em simple normal program} 
if additionally $m=1$ for every rule.

For a rule $r$ of form (\ref{equa-ep}), 
we use $body(r)$ and $head(r)$ to 
refer to the body $B$ and the head 
$H_1\mid \cdots \mid H_m$, respectively. 
We also refer to each $H_i$ in $head(r)$ 
as a {\em head formula}.
When $B$ is void, we omit $\leftarrow$.
$r$ is a {\em constraint} if its head is $\bot$.
To uniquely identify rules in an epistemic program
$\Pi$, every rule in $\Pi$
is assigned an {\em ID number}.

For an epistemic-free program $\Pi$ consisting of rules
$r$ of form (\ref{equa-ep}),
an interpretation $I$ satisfies $r$ if
$I$ satisfies some $H_i$ in $head(r)$
once $I$ satisfies $body(r)$; $I$ is a (classical) model of $\Pi$
if $I$ satisfies all rules in $\Pi$, and 
$I$ is a minimal model if it is subset-minimal
among all models of $\Pi$.
$\Pi$ is consistent if it has a model.

The definition of satisfaction and models of 
epistemic programs
is based on a collection of interpretations.

\begin{definition}
\label{satisfaction-epis-program}
Let $\Pi$ be an epistemic program and 
$\cal A\neq \emptyset$ be a collection of interpretations.
Let $I\in \cal A$. Then
\begin{enumerate}[(1)]
\item
\label{sat-item1}
An epistemic negation $\naf F$ is true in $\cal A$ 
if $F$ is false in some $J\in \cal A$,
and false, otherwise.
$I$ satisfies $\naf F$ w.r.t.\ $\cal A$
if $\naf F$ is true in $\cal A$. 

\item
$I$ satisfies an epistemic formula $E$ w.r.t.\ $\cal A$
if $I$ satisfies 
$E^{\cal A}$ as in propositional logic,
where $E^{\cal A}$ is
$E$ with every epistemic negation $\naf F$ that is true in $\cal A$
replaced by $\top$, and the other epistemic negations
replaced by $\bot$.

\item
$I$ satisfies 
a rule $r$ of form (\ref{equa-ep}) w.r.t.\ $\cal A$ if
$I$ satisfies some $H_i$ in $head(r)$  w.r.t.\ $\cal A$
once $I$ satisfies $body(r)$  w.r.t.\ $\cal A$. 

\item
$\cal A$ satisfies a rule $r$ 
(or $r$ is true in $\cal A$)
if every $J\in \cal A$ satisfies $r$ w.r.t.\ $\cal A$.
$\cal A$ satisfies $\Pi$, or $\cal A$ is 
an {\em epistemic model} of $\Pi$,  
if $\cal A$ satisfies all rules in $\Pi$.
$\Pi$ is {\em consistent} if it has an epistemic model.
\end{enumerate}
\end{definition}

Thus, as {\bf K} stands for  $\neg \mathbf{not}$ and {\bf M}
for $\naf \neg$,  ${\bf K} F$ (resp.\ ${\bf M} F$) 
is true in $\cal A$ if $F$ is true in every 
(resp. some) $J\in \cal A$.

An answer set semantics
for an epistemic-free program $\Pi$
is defined in terms of some intended models of $\Pi$,
called answer sets, 
and a world view semantics for an epistemic program $\Pi$
is defined in terms of some intended epistemic
models of $\Pi$, called world views. 
By an ASP semantics
we refer to an answer set semantics
for epistemic-free programs or 
a world view semantics for epistemic programs.

\rev{A fundamental question is what intended
models and intended epistemic models of a logic program are.
We provide solutions to this question in
Section \ref{sec-rational-semantics}.
Briefly, based on the refined GAS principles
as formulated in the Introduction,
we define those models of an epistemic-free
program $\Pi$ as intended models/answer sets
of $\Pi$ if they satisfy the properties of well-supportedness and minimality
w.r.t.\ negation by default,
and those epistemic models of an epistemic
program $\Pi$ as intended epistemic models/world views
of $\Pi$ if they additionally satisfy the property of minimality
w.r.t.\ epistemic negation.}{}

\section{Examples of Answer Sets Violating Minimal Models, 
Constraint Monotonicity and Foundedness for Simple Disjunctive Programs}
\label{sec-min-cons-found}

In this section, we use 
examples to demonstrate that requiring
minimal models (MM), constraint monotonicity (CM)
and foundedness (FN) as mandatory conditions for an answer set semantics may  
exclude expected answer sets 
for some simple disjunctive programs.

\begin{example}[A counter-example to CM and FN]
\label{ex1-cmfn}
{\em
Given a simple disjunctive program like
\begin{tabbing} 
\hspace{.3in} \= $\Pi:\ $ \=  $a\mid b$\`$(1)$\\[2pt]
\>\> $a\mid c$ \` $(2)$ \\[2pt]
\>\> $\bot \leftarrow \neg a$  \` $(3)$\\[2pt]
\>\> $\bot \leftarrow \neg c$ \` $(4)$ 
\end{tabbing}
we may ask a question: what is the intuitive meaning of this program? 
In other words, what are the intended models/answer sets of $\Pi$
we expect to have, 
independently of any specific applications $\Pi$ may represent?
A direct way to answer this question is to intuitively
infer/construct all expected
answer sets from the rules of $\Pi$.

Intuitively, for construction of an intended model/answer set $I$,
rules (1) and (2) each present two alternatives.
Rule (1) states that $I$ either contains
$a$ or alternatively $b$, while rule (2) says that 
$I$ either contains $a$ or
alternatively $c$. Therefore, 
we may infer $I$ from rules (1) and (2)
with the following four possible choices:
\begin{tabbing}
\hspace{.3in} \= (i) $I = \{a\}$, where $I$ chooses $a$ from both rule (1) and rule (2);\\
\> (ii) $I = \{a, c\}$, where $I$ chooses $a$ from rule (1) and $c$ from rule (2);\\
\> (iii) $I = \{b, a\}$, where $I$ chooses $b$ from rule (1) and $a$ from rule (2);\\
\> (iv) $I = \{b, c\}$, where $I$ chooses $b$ from rule (1) and $c$ from rule (2).
\end{tabbing}

The last two rules in $\Pi$ are constraints, where rule (3) says that $I$ must contain $a$, and
rule (4) says that $I$ must contain $c$. 
This immediately eliminates the three choices (i), (iii)
and (iv). As a result, only $I = \{a, c\}$ satisfies 
the two constraints.
As $I = \{a, c\}$ is inferred/constructed
from all non-constraint rules (i.e., rules (1) and (2)) and it
satisfies all constraints (i.e., rules (3) and (4)) in $\Pi$, 
we call it a {\em candidate answer set} of $\Pi$.
Then we expect $I$ to be an intended model/answer set
provided that it is subset-minimal, i.e., 
no proper subset of $I$ is a candidate answer set of $\Pi$.
As no proper subset of $I$ is a model of $\Pi$,
$I$ is subset-minimal and thus is an  
expected answer set of $\Pi$.

Note the following features with this expected answer set $I$. 
First, $I$ violates CM. To show this,
let $\Pi'$ be $\Pi$ without rule (4); 
i.e., $\Pi=\Pi'\cup \{\bot \leftarrow \neg c\}$.
By CM, $I$ must be an answer set of $\Pi'$.
However, one can easily check that  $\Pi'$ has 
only one answer set $I'=\{a\}$, contradicting CM.

Second, $I$ also violates FN, i.e., some
subset $X$ of $I$
is an unfounded set (see the conditions
(f1)-(f3) of an unfounded set as defined in Section \ref{int}). 
Consider $X=\{c\}$. $\Pi$ has only rule (2) 
whose head intersects with $X$. 
For rule (2), as $(\{a,c\}\setminus X) \cap I =\{c\}$,
condition (f3) is violated.
As no rule in $\Pi$ 
whose head intersects
with $X$ satisfies conditions (f1)-(f3),
$X$ is an unfounded set of $I$ and thus
$I$ is not a founded answer set.

Finally, the GL-reduct $\Pi^I$ of  $\Pi$ w.r.t. $I$ 
consists of rules (1) and (2) in $\Pi$.
As $I' = \{a\}$ is a minimal model of $\Pi^I$, 
$I$ is not a minimal model of 
the GL-reduct and thus is not
an answer set of  $\Pi$ under the GL-semantics.

To sum up, for this simple disjunctive program $\Pi$, 
if we use the GL-semantics, or if we require CM or FN as mandatory conditions,
we would lose the expected answer set  $I = \{a, c\}$,
no matter what specific applications $\Pi$ may represent.
}
\end{example}

\begin{example}[A counter-example to MM, CM and FN borrowed from \cite{ShenE19}]
\label{ex-nonminimal}
{\em Consider another simple disjunctive program
\begin{tabbing} 
\hspace{.3in} \= $\Pi:\quad$ \= $a \mid b$ \` $(1)$\\[2pt]
\>\>          $b \leftarrow a$ \` $(2)$\\[2pt]
\> \>         $c \leftarrow a$  \` $(3)$\\[2pt]
\>\>          $\bot \leftarrow \neg c$  \` $(4)$
\end{tabbing}
Intuitively, rule $(1)$ presents two alternatives, either $a$ or $b$,
for construction of an answer set $I$,
and rules (2) and (3) infer $b$ and $c$,
respectively if $a$ is already in $I$.
We distinguish between the following two cases.

First, suppose that $I$ chooses $a$ 
from rule (1).
Then, given $a\in I$ we infer $b$ and $c$ from
rules (2) and (3), respectively.
That is, we infer $I=\{a,b,c\}$ from rules (1)-(3). 
As $I$ satisfies the constraint rule (4), 
it is a candidate answer set of $\Pi$.
Then we expect $I$ to be an answer set
provided that it is subset-minimal.

Alternatively, suppose that $I$ chooses $b$
from rule (1). 
Given $b\in I$, neither $a$ nor $c$ can be inferred
from rules (2) and (3),
so we can only infer $I=\{b\}$
from rules (1)-(3).
As $I$ does not satisfy the constraint rule (4),
it will not be an expected answer set of $\Pi$.

As there is no more alternative for $I$ to choose from rule (1),
we can obtain no candidate answer set 
that is a proper subset of $I=\{a,b,c\}$.
Consequently,  $I=\{a,b,c\}$
is a minimal candidate answer set and
thus we expect it to be an answer set of $\Pi$.

Note that this expected answer set $I$
 is not a minimal model of $\Pi$
($J = \{b,c\}$ is a smaller model), thus
violating MM. One can check that $I$ also
violates CM and FN and is not
an answer set under the GL-semantics.
}
\end{example}

There are more examples in which some expected answer sets 
violate CM, FN or MM. 
Next, we introduce two examples 
which encode in simple disjunctive programs a practical scenarios called
the Generalized 
Strategic Companies (GSC) problem
\cite{cado-etal-97,Leone2006,ShenE19}.\footnote{The formal definition
of the GSC problem is as follows \cite{ShenEiter22}.
Suppose a holding has a set $C = \{c_1,\ldots,c_m\}$ of companies 
and it produces a set $G = \{g_1,...,g_n\}$ of goods, 
where each company $c_i \in C$
produces some goods $G_i \subseteq G$. 
The holding wants to sell some of
its companies subject to the following conditions: all products should be still
in the portfolio and no company $c_i$ for which a strategy rationale,
expressed by justifications $\sigma^{i}_1,\ldots,\sigma^{i}_{k_i}$
with each
$\sigma^{i}_j= \ell^{i}_{1}\land\cdots\land\ell^{i}_{l_i}$ 
being a conjunction of literals on $C$, 
holds true is sold. 
The {\em GSC problem} is to find a set  $C' \subseteq C$ of companies,
called a {\em strategic set}, satisfying the following four
conditions:
(1) The companies in $C'$ 
produce all goods in $G$; 
(2) If a justification $\sigma^{i}_j$ is satisfied by $C'$,
then its associated company $c_i$ must be in $C'$;
(3) Every $c_i\in C'$ either produces some goods in $G$
or has some
justification $\sigma^{i}_j$
that is satisfied by $C'$;
 and (4) $C'$ is subset-minimal w.r.t.~conditions (1)-(3).
An instance of the GSC problem is obtained by instantiating
the variables $C$, $G$ and $\sigma$ with specific 
companies, goods and strategy conditions, respectively.
}

\begin{example}[Borrowed from \cite{ShenEiter22} with slight modification]
\label{ex:generalized-strat-comp}
{\em 
Consider the following simple disjunctive program, 
which encodes an instance of the GSC problem:\footnote{This instance of the GSC
problem is as follows. We have three companies $C = \{ c_1, c_2, c_3\}$,
two goods $G = \{ g_1, g_2 \}$ and the three strategy conditions: 
$\sigma^{1}_1 = c_2\land
c_3, \  \sigma^{2}_1 = c_3\  \mbox{      and     } \ \sigma^{3}_1 = c_1 \land \neg c_2.$
Suppose that $g_1$ can be produced either by 
$c_1$ or $c_2$, and $g_2$ produced 
either by $c_1$ or $c_3$. 
As explained in \cite{ShenEiter22,shenEiter2025arxiv},
the strategy condition
$\sigma^{3}_1$ can be encoded by a rule 
 $c_3 \leftarrow c_1\land \neg c_2$
or equivalently by a constraint $\bot \leftarrow c_1\land \neg c_2$,
as there will be no strategic set
$C' \subseteq C$ satisfying the condition
$c_1 \land \neg c_2$ (i.e., company $c_1$ is in $C'$ and $c_2$ is not).
Then, finding a strategic set $C'\subseteq C$ 
satisfying the four conditions for this GSC problem
corresponds to constructing an answer set 
$I$ from $\Pi$. That is, $C'\subseteq C$
is a strategic set satisfying the four conditions 
iff $I=C'\cup \{g_1,g_2\}$ is 
an answer set of $\Pi$.
This GSC instance has only one strategic set
$C'=\{c_1,c_2\}$.
}
\begin{tabbing} 
\hspace{.3in} \= $\Pi:\ $ \=  $g_1$\`$(1)$\\[2pt]
\>\> $g_2$ \` $(2)$ \\[2pt]
\>\> $c_1 \mid c_2 \leftarrow g_1$ 
$\quad$   \` $(3)$\\[2pt]
\>\> $c_1 \mid c_3 \leftarrow g_2$ \` $(4)$ \\[2pt]
 \>\> $c_1 \leftarrow c_2 \land c_3$  \` $(5)$\\[2pt]
\> \> $c_2 \leftarrow c_3$  \` $(6)$\\[2pt]
 \>\> $\bot \leftarrow c_1\land \neg c_2$  \` $(7)$
\end{tabbing}

Suppose that some model $I$ of $\Pi$ is
an expected answer set. 
As $g_1$ and $g_2$ are true in rules (1) and (2),
$I$ must contain both $g_1$ and $g_2$.
Given $\{g_1,g_2\}\subseteq I$, 
the body conditions of rules (3) and (4) are true in $I$.
Then, by rule (3) $I$ either contains
$c_1$ or alternatively $c_2$, and by rule (4)
$I$ either contains $c_1$ or
alternatively $c_3$. 
Suppose that $I$ chooses $c_2$ from rule (3) and $c_1$ from rule (4).
Then, we infer the set
$\{g_1,g_2, c_1, c_2\}\subseteq I$ from rules (1)-(4).
As no head of rules (5)-(7) contains $c_3$,
we cannot infer $c_3$ from the three rules.
As a result, we infer $I = \{g_1,g_2, c_1, c_2\}$
from rules (1)-(6).
As $I$ satisfies the constraint rule (7),
it is a candidate answer set of $\Pi$.
As no proper subset of $I$ 
is a model of $\Pi$,
$I$ is subset-minimal 
and thus we expect it to be an answer set of $\Pi$.

One can check that no other model of $\Pi$ 
can be constructed as $I$ in the above way.
Therefore,
$I = \{g_1,g_2, c_1, c_2\}$ is the only expected 
answer set of $\Pi$.

This expected answer set $I$ has the following features.
First, $I$ violates CM. 
Let $\Pi'$ be $\Pi$ without rule (7); i.e., 
$\Pi=\Pi'\cup \{\bot \leftarrow c_1\wedge  \neg c_2\}$.
By CM, $I$ must be an answer set of $\Pi'$.
However, one can easily check that $\Pi'$ has 
only one answer set $I' = \{g_1,g_2, c_1\}$,
contradicting CM.

Second, $I$ violates FN, i.e., some
subset $X$ of $I$ is an unfounded set. 
To show this,
consider $X=\{c_2\}$. $\Pi$ 
has rules (3) and (6), i.e.,
$c_1 \mid c_2 \leftarrow g_1$ and $c_2 \leftarrow c_3$,
whose heads intersect with $X$. 
For rule (3), as $(\{c_1,c_2\}\setminus X) \cap I =\{c_1\}$,
condition (f3) is violated; and for
rule (6), as $I$ does not satisfy $c_3$,
condition (f1) is violated.
Consequently, no rule in $\Pi$ 
whose head intersects
with $X$ satisfies all of the conditions (f1)-(f3);
thus $X$ is an unfounded set of $I$.

Finally,  as $I' = \{g_1,g_2, c_1\}$ is a minimal model of
the GL-reduct $\Pi^I$ of  $\Pi$ w.r.t. $I$, $I$ is not a minimal model of 
$\Pi^I$ and thus is not
an answer set of  $\Pi$ under the GL-semantics.

As a result, for this simple disjunctive program $\Pi$, 
if we use the GL-semantics, or if we require CM or FN as mandatory conditions,
we would lose the expected answer set $I = \{g_1,g_2, c_1, c_2\}$.
}
\end{example}

\begin{example} 
\label{eg-minimal-models}
{\em
Consider the following simple disjunctive program,
which encodes another instance of the GSC problem:\footnote{This instance of the GSC says that we have companies $C = \{ c_1, c_2, c_3\}$,
goods $G = \{ g_1, g_2 \}$ and strategy conditions
$\sigma^{1}_1 = \neg c_3,$ $\sigma^{2}_1 = c_1$
and $\sigma^{3}_1 = c_1.$
Suppose that both $g_1$ and $g_2$ 
can be produced either by 
$c_1$ or $c_2$. This GSC instance has only one strategic set
$C'=\{c_1,c_2,c_3\}$.
}
\begin{tabbing} 
\hspace{.3in} \= $\Pi:\ $ \=  $g_1$\`$(1)$\\[2pt]
\>\> $g_2$ \` $(2)$ \\[2pt]
\>\> $c_1 \mid c_2 \leftarrow g_1$  \` $(3)$\\[2pt]
\>\> $c_1 \mid c_2 \leftarrow g_2$  \` $(4)$ \\[2pt]
\> \> $c_1 \leftarrow \neg c_3$  \` $(5)$\\[2pt]
\> \> $c_2 \leftarrow c_1$ \` $(6)$\\[2pt]
\> \> $c_3 \leftarrow c_1$ \` $(7)$
\end{tabbing} 

Suppose that some model $I$ of $\Pi$ is
an expected answer set. 
As $g_1$ and $g_2$ are true in rules (1) and (2),
the body conditions of rules (3) and (4) hold true,
so by choosing $c_1$ from
the heads of rules (3) and (4), respectively
we infer 
$\{g_1,g_2, c_1\}\subseteq I$ from rules (1)-(4).
Given $\{g_1,g_2, c_1\}\subseteq I$,
 the body conditions
of rules (6) and (7)
are true in $I$, so we infer that $I$ also contains
$c_2$ and $c_3$. As a result, we infer a candidate answer set
$I=\{g_1,g_2, c_1,c_2,c_3\}$ from rules (1)-(7).
Then we expect $I$ to be an answer set
provided that it is subset-minimal.

Note that in addition to $I$, $\Pi$ has only one more model,
$J=\{g_1,g_2, c_2,c_3\}$. which is a proper subset of $I$.
Assume that $J$ is an expected answer set, i.e.,
$J$ can be constructed in the above way.
As $J$ does not contain $c_1$, we will not choose $c_1$ from
the heads of rules (3) and (4).
Let us choose $c_2$; then we infer 
$\{g_1,g_2, c_2\}\subseteq J$ from rules (1)-(4).
As no body condition
of rules (5)-(7) is true in $J$, 
we cannot infer $c_3$ from these rules
and thus cannot obtain $J$ in the above way,
contradicting our assumption.
Therefore, $I$ is subset-minimal 
and is an expected answer set of $\Pi$.

This expected answer set $I$ is not a minimal model of $\Pi$
i.e., it violates MM.
One can also check that $I$ violates FN and is not
an answer set under the GL-semantics.
}
\end{example}

\section{Examples of World Views Violating Constraint Monotonicity and Foundedness for Epistemic Specifications}
\label{ConstraintMonoFoundednessforEpistemicSpec}

In \cite{iclp-KahlL18},
the notion of constraint monotonicity 
for simple disjunctive programs
was extended to epistemic specifications.
In particular, they introduced a notion of
{\em subjective constraint monotonicity (SCM)},
which requires that for every epistemic specification $\Pi$
and subjective constraint $C$,
world views of $\Pi\cup \{C\}$ 
should be world views of $\Pi$ satisfying $C$.
Here a {\em subjective constraint} is a constraint of the form
$\bot \leftarrow G_1 \wedge \cdots \wedge G_n$,
where each $G_i$ is a modal literal of the form ${\bf K} A$,
$\neg {\bf K} A$,  ${\bf M} A$ or  
$\neg {\bf M} A$, and $A$ is an atom.
Recall that  {\bf K} stands for  $\neg \mathbf{not}$ and {\bf M}
for $\naf \neg$.

We use an example to
show that requiring SCM as a mandatory condition may 
exclude expected world views for 
some epistemic specifications.

\begin{example}[A counter-example to SCM] 
\label{ex1-modal:generalized-strat-comp}
{\em 
Consider the simple disjunctive program $\Pi$
in Example~\ref{ex-nonminimal},
where rule (4) $\bot \leftarrow \neg c$ is a constraint,
which requires that every answer set of $\Pi$ 
should contain $c$. 
Such a constraint on ``every answer set''
can also be encoded with a modal operator ${\bf K}$
as a subjective constraint 
of the form $\bot \leftarrow \neg {\bf K} c$, thus transforming $\Pi$
into an epistemic specification
\begin{tabbing} 
\hspace{.3in} \= $\Pi_1:\quad$ \= $a \mid b$ \` $(1)$\\[2pt]
\>\>          $b \leftarrow a$ \` $(2)$\\[2pt]
\> \>         $c \leftarrow a$  \` $(3)$\\[2pt]
\>\>          $\bot \leftarrow \neg {\bf K} c$  \` $(4')$
\end{tabbing}

As  no modal literals occur in rules (1)-(3),
and rule (4) in $\Pi$
and  rule ($4'$) in $\Pi_1$
play the same role, i.e.,
requiring that every answer set
should contain $c$,
$\Pi$ and $\Pi_1$ should admit the same answer sets.
Therefore, as $\Pi$ has only one
expected answer set $I = \{a,b,c\}$ (see Example~\ref{ex-nonminimal}),
we expect that $\Pi_1$ has the only world view 
${\cal W}=\{I\}$.

It turns out that this expected world view
${\cal W}$ for $\Pi_1$
violates SCM.
Let $\Pi_2$ consist of rules (1)-(3),
i.e., $\Pi_1=\Pi_2\cup \{\text{rule } (4')\}$.
One can easily check that $\Pi_2$
has only one answer set $J=\{b\}$
and thus has only one world view ${\cal V}=\{J\}$.
As $W\neq V$, SCM is violated, which requires that
world views of $\Pi_1$ 
should be world views of $\Pi_2$ satisfying 
the constraint rule ($4'$).
This shows that requiring SCM as a mandatory condition may 
exclude expected world views for 
some epistemic specifications.
}
\end{example}

Another stronger requirement on world view semantics,
called {\em Epistemic Splitting (ES)}, was introduced in 
\cite{lpnmr-CabalarFC19}.
It was lifted from the notion of 
{\em splitting sets} \cite{lifs-turn-94} for answer sets 
of simple disjunctive programs
to world views of epistemic specifications.
Informally, an epistemic specification 
can be split if its {\em top} part only refers to the atoms of 
the {\em bottom} part through modal literals. 
Then, a world view semantics is said to satisfy epistemic splitting 
if it is possible to get its world views by first obtaining the world views 
of the bottom and then using the modal literals in the top as {\em queries}
on the bottom part previously obtained. 
It was shown by \citeA{lpnmr-CabalarFC19} that
the ES property is even more restrictive than
SCM 
in  
that every world view semantics 
satisfying ES also satisfies SCM.
Due to this, the requirement of ES 
may exclude expected world views
for some epistemic specifications like
that in Example~\ref{ex1-modal:generalized-strat-comp}.

Furthermore, the notion of foundedness 
for simple disjunctive programs 
\cite{LeoneRS97} was extended in \cite{CabalarFC20aij}
to {\em epistemic specifications with double negation}
consisting of rules of the form
\[A_1\, \mid \cdots \mid \, A_m \leftarrow G_1 \wedge \cdots \wedge G_n\]
where each $A_i$ is an atom,
and each  $G_j$ is either an {\em objective literal} 
that is an atom $A$ or its negation $\neg A$ or its double negation
$\neg\neg A$, or a {\em subjective literal} of the form ${\bf K} L$,
$\neg {\bf K} L$ or $\neg\neg {\bf K} L$, where $L$ is an objective literal.
A {\em simple disjunctive program with double negation}
is an epistemic specification with double negation
without subjective literals.
For a rule $r$,  
let $Head(r)$ denote the set of atoms in $head(r)$,
and $Body_{ob}(r)$ (resp. $Body_{sub}(r)$)
denote those atoms occurring in objective (resp. subjective) literals in $body(r)$.
Moreover, $Body_{ob}^+(r)$ (resp. $Body_{sub}^+(r)$) denotes those atoms 
in $Body_{ob}(r)$ (resp. $Body_{sub}(r)$)  
occurring in positive literals (i.e., literals free of negation). 

\begin{definition}[Unfounded set \cite{CabalarFC20aij}]
\label{unfdedSet}
Let $\Pi$ be an epistemic specification with double negation 
and $\cal W$ a world view.
An {\em unfounded set} w.r.t.\ $\cal W$
is a set $S\neq \emptyset$ of pairs
$\langle X, I\rangle$, where $X$ and $I$ are sets of atoms,
such that 
for each $\langle X, I\rangle\in S$, 
no rule $r$ in $\Pi$ with 
$Head(r) \cap X \neq \emptyset$ satisfies 
(ef1)-(ef4):
\smallskip

\begin{compactitem}
\itemsep=1pt
    \item[](ef1) $I$ satisfies $body(r)$ w.r.t. $\cal W$; 
\item[](ef2)
$Body^+_{ob} (r)\cap X = \emptyset$; 
\item[](ef3) $(Head(r)\setminus X)\cap I = \emptyset$; and 
\item[](ef4) for all $\langle X', I'\rangle\in S$, $Body^+_{sub} (r)\cap X' = \emptyset$.
\end{compactitem}
\end{definition}

Compared to the conditions (f1)-(f3)
of an unfounded set for simple disjunctive programs
as defined in the introduction,
condition (ef2) generalizes (f2),
while (ef4) is new. Founded world views are then as follows.

\begin{definition}[World view foundedness \cite{CabalarFC20aij}, WFN]
\label{def-founded-aij20}
A world view $\cal W$ of an epistemic specification $\Pi$
with double negation
is {\em unfounded} if some
unfounded set $S$ exists such that every $\langle X, I\rangle\in S$
satisfies $X\cap I \neq \emptyset$ and $I\in \cal W$;
 otherwise, $\cal W$ is called {\em founded}.
\end{definition}

The following example 
shows that requiring WFN as a mandatory condition may 
exclude expected world views for 
some epistemic specifications.

\begin{example}[A counter-example to WFN] 
\label{ex3-epi:generalized-strat-comp}
{\em 
Consider the epistemic specification $\Pi_1$
in Example~\ref{ex1-modal:generalized-strat-comp},
where
${\cal W}=\{I\}$ with $\{a,b,c\}$
is the only expected world view.
We show that ${\cal W}$
violates WFN,
i.e., there is an 
unfounded set $S$ w.r.t.\ ${\cal W}$
such that every $\langle X, I\rangle\in S$
satisfies $I\in \cal W$ and $X\cap I \neq \emptyset$.
Let
$S=\{\langle X, I\rangle\}$, where 
$X=\{a,c\}$ and $I=\{a,b, c\}$. 
We show that for
$\langle X,I\rangle$
no rule $r\in \Pi_1$ with 
$Head(r) \cap X \neq \emptyset$ 
satisfies (ef1)-(ef4).

$\Pi_1$ has rules (1) and (3) whose heads 
intersect with $X$. For rule (1),
as $(\{a,b\}\setminus X) \cap I =\{b\}$,
condition (ef3) is violated; for rule (3),
condition (ef2) is violated.
Consequently, 
no rule in $\Pi_1$ 
whose head 
intersects with $X$ 
satisfies (ef1)-(ef4);
thus $S=\{\langle X, I\rangle\}$
is an unfounded set w.r.t.\ ${\cal W}$.
As $X\cap I \neq \emptyset$ and $I\in \cal W$,
$\cal W$ is hence unfounded.

As a result, requiring WFN as a mandatory condition may 
lose the expected world view ${\cal W}$ of $\Pi_1$.
}
\end{example}

\section{Defining Well-Supportedness for Answer Sets of Epistemic-Free Programs and World Views of Epistemic Programs}
\label{sec-Well-SupportednessASP}
\citeA{Fages:JMLCS:1994} introduced
the notion of well-supportedness
as a key property of
answer sets for simple normal programs.
In this section we extend the notion to answer sets of
epistemic-free programs and
world views of epistemic programs.
Briefly, we first build the well-supportedness for answer 
sets of epistemic-free normal programs. Then for an epistemic-free 
program $\Pi$ with disjunctive rule heads, the well-supportedness 
of an answer set $I$ is built by transforming $\Pi$ into an 
epistemic-free normal program $\Pi'$ w.r.t.\ $I$ and defining 
the well-supportedness of $I$ in $\Pi$ based on the well-supportedness 
of $I$ in $\Pi'$. Finally, for an epistemic 
program $\Pi$, the well-supportedness 
of a world view $\cal A$ is built by transforming $\Pi$ into an 
epistemic-free program $\Pi'$ w.r.t.\ $\cal A$ and defining 
the well-supportedness  
of $\cal A$ in $\Pi$ based on the well-supportedness 
of every $I\in {\cal A}$ in $\Pi'$.

We begin by recalling Fages's definition.
Recall that a binary relation $\leq$ on a set $S$ is {\em well-founded} if there is no
infinite decreasing chain $a_0\geq a_1\geq \cdots$ in $S$.

\begin{definition}[Well-supportedness for simple normal programs \cite{Fages:JMLCS:1994}]
\label{def-well-supported-simp-normal}
For a simple normal program $\Pi$,
an interpretation $I$ is {\em well-supported} in $\Pi$
if there exists a strict well-founded
partial order $\prec$ w.r.t.\ $\Pi$ on atoms in $I$ 
such that for every $p\in I$ 
there is a rule $p\leftarrow body(r)$ in $\Pi$, where
$I$ satisfies $body(r)$ and $q\prec p$
for every positive literal $q$ in $body(r)$.
\end{definition}

Note that $I=\emptyset$ is a well-supported interpretation of 
all simple normal programs, 
where as $I$ is an empty set,
the condition of Definition~\ref{def-well-supported-simp-normal} trivially holds.
Moreover, some interpretations of a simple normal program $\Pi$
may be well-supported even though they do not satisfy $\Pi$.
For example, consider $\Pi=\{q,\ p\leftarrow \neg p\}$ and $I=\{q\}$, where
$I$ does not satisfy $\Pi$, but it is
well-supported under Definition~\ref{def-well-supported-simp-normal}.
This does not seem to be a desirable behavior. 
To avoid this weakness, in this paper 
we choose to define the notion of well-supportedness
on models of logic programs.

\subsection{Well-supportedness of answer sets for epistemic-free normal programs with atomic rule heads}
  
We first extend the notion of well-supportedness to
epistemic-free normal programs with atomic rule heads
consisting of rules
of the form $H\leftarrow B$,
where $H$ is an atom or $\bot$, and $B$ is a propositional formula.
As rule bodies are arbitrary propositional formulas,
 Fages's definition of well-supportedness 
(Definition~\ref{def-well-supported-simp-normal})
is not applicable.

Observe that by the well-supportedness of
an interpretation $I$ we refer to the well-supportedness of
the {\em positive part} of $I$, 
where the {\em negative part} $\neg I^- $ of $I$
is assumed to be true by default (following the principle of
minimality of answer sets w.r.t.\ negation by default
defined in Section \ref{int}).
Therefore, our intuition behind the notion of 
well-supportedness is that 
when the negative part, i.e., every $\neg p\in \neg I^-$,
is assumed to be true by default, the positive part, i.e.,
every $p\in I$, can be inferred from the if-then rules in $\Pi$.
Formally we have the following definition.

\begin{definition}
\label{def-well-supported-normal-atomhead}
Let $\Pi$ be an epistemic-free normal program 
with atomic rule heads.
A model $I$ of $\Pi$ is {\em well-supported} in $\Pi$
if there exists a strict well-founded
partial order $\prec$  w.r.t.\ $\Pi$ on atoms in $I$ 
such that for every $p\in I$
there is a rule $p\leftarrow body(r)$ in $\Pi$ and
some $S\subset I$, where 
$S\cup \neg I^- \models body(r)$
and $q\prec p$
for every $q\in S$.
\end{definition}

Like Fages's definition,
a well-supported model $I$ defined in 
Definition~\ref{def-well-supported-normal-atomhead}
induces a level mapping on atoms in $I$ 
via the partial order $\prec$ w.r.t.\ $\Pi$, i.e.,
for every $p\in I$
there is a rule $p\leftarrow body(r)$ in $\Pi$ and some $S\subset I$,
where $body(r)$ is true in $S\cup \neg I^-$ (i.e., 
$S\cup \neg I^- \models body(r)$)  
and every $q\in S$ is at a lower level than $p$
(due to $q\prec p$).
Due to having a level mapping, 
well-supported models of 
epistemic-free normal programs 
with atomic rule heads are free of circular justification. 

\begin{example}
\label{eg-wellsupp-normal-pqs}
{\em
Consider an epistemic-free program
$\Pi = \{q\leftarrow q\vee \neg q, \ p\leftarrow q\wedge \neg s\}$ and a model $I = \{p,q\}$. 
Note that the propositional formula $q\vee \neg q$ 
in the body of the first rule 
represents a tautology in classical logic. 
Let $\neg I^-=\{\neg s\}$.
By Definition~\ref{def-well-supported-normal-atomhead},
$I$ is well-supported in $\Pi$ with a strict well-founded
partial order $q\prec p$.
For $q\in I$,
we have rule $q\leftarrow q\vee \neg q$ in $\Pi$ and 
$S=\emptyset$, where
$S\cup \neg I^- \models q\vee \neg q$.
For $p\in I$,
we have rule $p\leftarrow q\wedge \neg s$ in $\Pi$ and 
$S=\{q\}$ , where 
$S\cup \neg I^- \models q\wedge \neg s$
and $q\prec p$.
}
\end{example}

Our definition of well-supportedness
agrees with Fages's definition for simple normal programs.

\begin{restatable}{theorem}{RthSimpNormal}
\label{th-simp-normal}
For a simple normal program $\Pi$,
a model $I$ of $\Pi$ is well-supported in $\Pi$ by
Definition~\ref{def-well-supported-simp-normal} iff 
$I$ is well-supported by 
Definition~\ref{def-well-supported-normal-atomhead}.
\end{restatable}

In order to have an algorithm 
for automatically determining the well-supportedness of a model,
we present a characterization of well-supported models
based on the following one-step provability operator
introduced in \cite{ShenWEFRKD14}.

\begin{definition}
\label{def-one-step-operator}
Let $\Pi$ be an epistemic-free 
normal program and 
let $O$ and $N$
be two theories. Define
\[T_\Pi(O,N)= \{head(r) \mid r\in \Pi \textrm{ and } O\cup N \models body(r)\}\]  
\end{definition}

Intuitively, $T_\Pi(O,N)$ collects all heads of rules in $\Pi$
whose bodies are true in $O\cup N$.
When the parameter $N$ is fixed, 
the entailment relation $\models$ is monotone in $O$,
so $T_\Pi(O,N)$ is monotone
w.r.t.\ $O$, i.e., for all theories $O_1$ and $O_2$
with $O_1\subseteq O_2$, we have $T_\Pi(O_1,N) \subseteq
T_\Pi(O_2,N)$.  
As moreover $T_\Pi(O,N)$ is finitary,%
\footnote{I.e., whenever
$\alpha \in T_\Pi(O,N)$, then there is some finite $O'\subseteq
O$ such that $\alpha \in T_\Pi(O',N)$.}
it is immediate that
the inference sequence 
$\langle T_{\Pi}^i(\emptyset,N)
\rangle_{i=0}^\infty$, where $T_{\Pi}^0(\emptyset,N) = \emptyset$ and
for $i\geq 0$ $T_{\Pi}^{i+1}(\emptyset,N) =
T_{\Pi}(T_{\Pi}^i(\emptyset,N),N)$, will converge to a least fixpoint, denoted
$\mathit{lfp}(T_{\Pi}(\emptyset,N))$.

For an interpretation $I$, let 
$\mathit{lfp}(T_{\Pi}(\emptyset,\neg I^-))$
be the least fixpoint of the inference sequence 
$\langle T_{\Pi}^i(\emptyset,\neg I^-)
\rangle_{i=0}^\infty$ w.r.t.\ $I$.
It is easy to show that when $I$ is a model of $\Pi$,
for $i\geq 0$ $I$ satisfies $T_{\Pi}^i(\emptyset,\neg I^-)$
and thus satisfies $\mathit{lfp}(T_{\Pi}(\emptyset,\neg I^-))$.
A well-supported model $I$ 
can then be characterized in terms of
$\mathit{lfp}(T_{\Pi}(\emptyset,\neg I^-))$.

\begin{restatable}{theorem}{RthWellsupportLfp}
\label{th-wellsupport-lfp}
Let $\Pi$ be an epistemic-free normal program 
with atomic rule heads.
A model $I$ of $\Pi$ is well-supported in $\Pi$
by Definition~\ref{def-well-supported-normal-atomhead} iff
$\mathit{lfp}(T_{\Pi}(\emptyset,\neg I^-))=I$.
\end{restatable}

\subsection{Well-supportedness of answer sets for epistemic-free normal programs}

We now consider
epistemic-free normal programs
consisting of rules
of the form $H\leftarrow B$,
where both $H$ and $B$ are propositional formulas.
As rule heads are formulas,
Definition~\ref{def-well-supported-normal-atomhead}
is not applicable.
We define the well-supportedness of a model $I$
for such a program $\Pi$ by
first defining the well-supportedness of rule heads in $\Pi$ w.r.t.\ $I$.

\begin{definition}
\label{def-well-supportedHeads-normal}
Let $\Pi$ be an epistemic-free normal program
and $I$ be a model of $\Pi$.
Let ${\cal V}$ be a set of rule heads in $\Pi$, where
for every $H\in {\cal V}$
there is a rule $H\leftarrow body(r)$ in $\Pi$
such that $I$ satisfies $body(r)$.
Then
${\cal V}$ is {\em well-supported} in $\Pi$ w.r.t.\ $I$
if (1) for every rule $H\leftarrow body(r)$ in $\Pi$, 
if ${\cal V}\cup \neg I^- \models body(r)$
then $H$ is in ${\cal V}$,
and (2) there exists a strict well-founded
partial order $\prec_h$ on rule heads in ${\cal V}$ 
such that for every $H\in {\cal V}$
there is a rule $H\leftarrow body(r)$ in $\Pi$ and
some $S\subset \cal V$, where 
$S\cup \neg I^- \models body(r)$
and $F\prec_h H$
for every $F\in S$.
\end{definition}

Every well-supported set $\cal V$ of rule heads defined 
in Definition~\ref{def-well-supportedHeads-normal}
induces a level mapping on $\cal V$
via the partial order $\prec_h$ w.r.t.\ $\Pi$ and $I$, i.e.,
for every $H\in {\cal V}$
there is a rule $H\leftarrow body(r)$ in $\Pi$ and
some $S\subset \cal V$, where $body(r)$ is true
in $S\cup \neg I^-$ and 
every $F\in S$ is at a lower level than $H$
(due to $F\prec_h H$).

Well-supported rule heads
of an epistemic-free normal program $\Pi$ w.r.t.\ $I$
can also be characterized in terms of the least fixpoint
$\mathit{lfp}(T_{\Pi}(\emptyset,\neg I^-))$
of the inference sequence 
$\langle T_{\Pi}^i(\emptyset,\neg I^-)
\rangle_{i=0}^\infty$ w.r.t.\ $I$.

\begin{restatable}{theorem}{RthWellsupportRuleheadsLfp}
\label{th-wellsupport-ruleheads-lfp}
Let $\Pi$ be an epistemic-free normal program
and $I$ be a model of $\Pi$.
A set ${\cal V}$ of rule heads in $\Pi$ 
is well-supported in $\Pi$ w.r.t.\ $I$ iff
$\mathit{lfp}(T_{\Pi}(\emptyset,\neg I^-))={\cal V}$.
\end{restatable}

We then define the well-supportedness of a model $I$ for 
an epistemic-free normal program $\Pi$ based on
some well-supported rule heads $\cal V$
in $\Pi$ w.r.t.\ $I$.
Intuitively, $I$ is well-supported
if every $p\in I$ is well-supported
in the way that there exists 
one or more rules $head(r_i)\leftarrow body(r_i)$
$(1\leq i\leq k)$ in $\Pi$ and  
some $S\subset {\cal V}$, where 
for every $i$, $S\cup \neg I^- \models body(r_i)$ and 
$F\prec_h head(r_i)$ for every $F\in S$,
and $p$ is derivable from
$S\cup\{head(r_1), \cdots, head(r_k)\}\cup \neg I^-$
but it is not derivable from $S\cup \neg I^-$. 
Formally we have the following definition.

\begin{definition}[Well-supportedness for epistemic-free normal programs]
\label{def-well-supported-normal}
Let $\Pi$ be an epistemic-free normal program.
A model $I$ of $\Pi$ is {\em well-supported} in $\Pi$
if there exists (1) a well-supported set $\cal V$ of rule heads
in $\Pi$ w.r.t.\ $I$ with a strict well-founded
partial order $\prec_h$ on ${\cal V}$
and (2) a strict well-founded
partial order $\prec$ on $I$, such that for every $p\in I$
there are $k\geq 1$ rules 
$head(r_1)\leftarrow body(r_1), \cdots,
head(r_k)\leftarrow body(r_k)$ in $\Pi$
and  
some $S\subset {\cal V}$ with
$S\cup\{head(r_1), \cdots, head(r_k)\}\cup \neg I^- \models p$, 
where for $1\leq i\leq k$, 
$S\cup \neg I^- \models body(r_i)$, 
$F\prec_h head(r_i)$
for every $F\in S$, 
and $q\prec p$ for every $q\in I$ 
with $S\cup \neg I^- \models q$.
\end{definition}

Every well-supported model $I$ defined in 
Definition~\ref{def-well-supported-normal}
induces a level mapping on atoms in $I$ 
via the partial orders $\prec_h$ and $\prec$ w.r.t.\ $\Pi$, i.e., for every $p\in I$ 
there are $k$ rules $head(r_i)\leftarrow body(r_i)$
$(1\leq i\leq k)$ in $\Pi$
and some rule heads $S\subset {\cal V}$
such that $p$ is true in $S\cup \{head(r_1), \cdots, head(r_k)\}\cup \neg I^-$,
where every $body(r_i)$ is true in $S\cup \neg I^-$,  
every $F\in S$ is at a lower level 
than $head(r_i)$
(due to $F\prec_h head(r_i)$),
and every $q\in I$ that is true in $S\cup \neg I^-$
is at a lower level than $p$
(due to $q\prec p$).
Therefore, well-supported models of 
epistemic-free normal programs 
are free of circular justification. 

Most importantly, a well-supported model $I$ 
defined in Definition~\ref{def-well-supported-normal}
can concisely be characterized in terms of
the least fixpoint $\mathit{lfp}(T_{\Pi}(\emptyset,\neg I^-))$
of the inference sequence 
$\langle T_{\Pi}^i(\emptyset,\neg I^-)
\rangle_{i=0}^\infty$ w.r.t.\ $I$.

\begin{restatable}{theorem}{RthWellsupportLfpNormal}
\label{th-wellsupport-lfp-normal}
Let $\Pi$ be an epistemic-free normal program.
A model $I$ of $\Pi$ is well-supported in $\Pi$
by Definition~\ref{def-well-supported-normal} iff
$\mathit{lfp}(T_{\Pi}(\emptyset,\neg I^-))\cup \neg I^-\models p$
for every $p\in I$.
\end{restatable}

Observe that when $\mathit{lfp}(T_{\Pi}(\emptyset,\neg I^-))$
is a set of atoms and $\mathit{lfp}(T_{\Pi}(\emptyset,\neg I^-))\cup \neg I^-$
is consistent, we have that
$\mathit{lfp}(T_{\Pi}(\emptyset,\neg I^-))\cup \neg I^-\models p$
for every $p\in I$
iff $\mathit{lfp}(T_{\Pi}(\emptyset,\neg I^-))=I$.
The following corollary is now immediate from 
Theorems \ref{th-wellsupport-lfp-normal} and
\ref{th-wellsupport-lfp}.

\begin{corollary}
\label{cor-normal-atomhead-formula}
Let $\Pi$ be an epistemic-free normal program with atomic rule heads.
A model $I$ of $\Pi$ is well-supported in $\Pi$ by
Definition~\ref{def-well-supported-normal} iff 
$I$ is well-supported by 
Definition~\ref{def-well-supported-normal-atomhead}.
\end{corollary}

\subsection{Well-supportedness of answer sets for epistemic-free programs with disjunctive rule heads}
\label{ws-disrh}
Next we consider epistemic-free programs with disjunctive rule heads.
Note that the disjunctive rule head operator $\mid$
in a logic program essentially
differs from the disjunction connective $\vee$
in classical logic. For simple disjunctive programs,
\citeA{GL91} deliberately used $\mid$ to indicate
that the connection of alternatives in a disjunctive rule head 
$A_1 \mid \cdots \mid A_m$  is not simply
disjunction $\vee$ as in classical logic.\footnote{In \cite{Gelfond2008,gelfond-kahl2014},
$or$ is used in place of $\mid$ and
called {\em epistemic disjunction}.} 
Rather, one ``commits'' to some alternative $A_i$ in the rule head 
for the belief state that represents the model one 
is willing to accept. 

\citeA{ShenE19} presented a constructive view 
of $\mid$ as a nondeterministic inference operator.
Intuitively, a rule head $A_1 \mid \cdots \mid A_k$
models incomplete information concerning the
truth value of each alternative $A_i$, namely that the information
available is insufficient to establish whether $A_1$ is true 
or $A_2$ is true $\cdots$ or $A_k$ is true,
but nonetheless sufficient to establish that
at least one $A_i$ is true.
This means that
the disjunctive rule head offers different choices, 
i.e., it either infers $A_1$ 
or alternatively $A_2$ $\cdots$ or alternatively $A_k$.
Therefore, a rule $A_1 \mid \cdots \mid A_k\leftarrow Body$ 
infers some $A_i$ in the head if 
the body condition $Body$ holds.  
As the choice of such an alternative 
$A_i$ among $A_1, \cdots,  A_k$ 
is arbitrary, the inference with such rules 
for the construction of an answer 
set is nondeterministic.  
It is such nondeterministic inferences 
via disjunctive rule heads  
that allow us to generate different candidate answer sets. 

The above constructive view of 
$\mid$ as a nondeterministic inference operator
allows us to clearly differentiate 
a rule head $A_1 \mid \cdots \mid A_k$ from
a logical formula $A_1 \vee \cdots \vee A_k$ ($k>1$).
The former infers nondeterministically
one $A_i$ among the alternatives, while the latter
cannot establish the truth of any $A_i$ and thus can
infer neither $A_1$ nor $A_2$ $\cdots$ nor $A_k$.
For example, $\Pi_1=\{p\vee q,\ q\leftarrow p,\ p\leftarrow q\}$
and $\Pi_2=\{p\mid q,\ q\leftarrow p,\ p\leftarrow q\}$
are two different logic programs,
where $p\vee q$ can be inferred from $\Pi_1$, but 
$p$ and $q$ can be inferred from $\Pi_2$.
In $\Pi_2$, the rule $p\mid q$ presents two alternatives,
i.e., $p$ or $q$. If we choose $p$,
we would infer $q$ from the rule $q\leftarrow p$;
alternatively, if we choose $q$,
we would infer $p$ from the rule $p\leftarrow q$.

Note that a disjunctive rule head like $p\mid q$ 
offers only two alternatives for answer set construction, i.e., 
an answer set $I$ can contain either $p$ or alternatively $q$.
If we want $I$ to have a third alternative covering both $p$ and $q$,
we may need to use a rule head like $p\mid q\mid p\land q$.

\citeA{ShenE19} introduced a 
disjunctive rule head selection function $sel(head(r), I)$
to characterize the intuitive meaning of the operator $\mid$.
Let $\Pi$ be an epistemic-free program consisting of rules
of the form $H_1\mid \cdots\mid H_m\leftarrow B$,
where each $H_i$ and $B$ are propositional formulas.
Given an interpretation $I$,
$sel(head(r), I)$ selects one
head formula $H_i$ that is satisfied by $I$ from 
the head $head(r)=H_1\mid \cdots\mid H_m$ of every rule $r$ in $\Pi$
once the rule body condition $body(r)$ is satisfied by $I$.
This transforms $\Pi$ 
into an epistemic-free normal program $\Pi^I_{sel}$, 
called a {\em disjunctive
program reduct} of $\Pi$ w.r.t.\ $I$ and $sel$.

One critical requirement of the head selection function $sel$
defined in \cite{ShenE19} is that for two 
disjunctive rules $r_1$ and $r_2$
whose heads are identical or variant,
it must select the same head formula, 
i.e., $sel(head(r_1), I)\equiv sel(head(r_2), I)$.
This means that the selection of a head formula
from one rule may depend on the selection 
of head formulas from some other rules.
We note that this requirement 
deviates from the intuitive interpretation of the disjunctive 
rule head operator $\mid$ 
as a nondeterministic inference operator which
returns one alternative $H_i$
from every rule $H_1\mid\cdots\mid H_m\leftarrow  B$, 
independently of any other rules, 
once $B$ is satisfied by $I$.
As a result, when using this head selection function, 
some expected answer sets may be excluded and some
unexpected ones may be obtained;
see Section~\ref{subsec-DI-GL-semantics} for an example.

To overcome the above issue with the existing 
selection function
$sel(head(r), I)$, 
in this section we introduce
an upgraded version $sel((n, head(r)), I)$ by attaching
a unique ID number $n$ to every rule $r$.

\begin{definition}[Head selection function $sel((n, head(r)), I)$]
\label{head-sel-gdlp}
Let $\Pi$ be an epistemic-free program and $I$ 
be an interpretation. Then for every rule $r\in \Pi$ 
with an ID number $n$, define 
\[sel((n, head(r)), I) = \left \{\begin{array}{ll}
                                 H_i & \mbox{$head(r)$ has some head formula $H_i$ that is satisfied by $I$,} \\
                                  \bot & \mbox{otherwise.}
                                 \end{array} \right. \]
\end{definition}

Due to having a unique ID number $n$
for every rule $r$, $sel((n, head(r)), I)$ 
can select any head formula $H_i$ satisfied by $I$
from the head $head(r)$ of $r$,
independently of 
any other rule $r'$, even if $r$ and $r'$ have
the same rule head. 
In a special case, 
it allows the user to  
use $k$ ($1\leq k\leq m$)
duplicate disjunctive rule heads 
\begin{tabbing} 
\hspace{.3in} \=  $H_1\mid H_2\mid\cdots\mid H_m$\`$(1)$\\[2pt]
\> $H_1\mid H_2\mid$ \= $\cdots\mid H_m$ \` $(2)$ \\[2pt]
\> \>  $\vdots$  \\
\> $H_1\mid H_2\mid\cdots\mid H_m$  \` $(k)$ 
\end{tabbing}
to produce candidate answer 
sets covering one to $k$ $H_i$'s, respectively. 
This significantly differs from the selection function $sel(head(r), I)$,
which selects for every candidate answer set only one $H_i$ 
satisfied by $I$ from the $k$ identical heads. 

\begin{definition}[Disjunctive program reducts]
\label{def:sel-reduct}
Let $\Pi$ be an epistemic-free program, $I$ be an interpretation, 
and $sel$ be a head selection function in Definition~\ref{head-sel-gdlp}.
The {\em disjunctive program reduct}  
of $\Pi$ w.r.t.\ $I$ and $sel$ is
\[\Pi^I_{sel} = \{sel((n,head(r)), I) \leftarrow body(r) \mid r \in \Pi \textrm{ with an ID number $n$ and $I$ satisfies $body(r)$}\}\]
\end{definition}

A disjunctive program reduct $\Pi^I_{sel}$  
can be viewed as
a {\em projection} of $\Pi$  w.r.t.\ a specific selection $sel$
of rule head formulas that are satisfied by $I$.
Therefore, the well-supportedness of $I$ in $\Pi$
can be defined
in terms of the well-supportedness of $I$
in projections of $\Pi$. As a disjunctive program
$\Pi$ may have more than one projection w.r.t.\ 
different head selections, 
we take a credulous view.
That is, we define a model $I$ to be well-supported
in $\Pi$ if it is well-supported in some projection of $\Pi$.

\begin{definition}[Well-supportedness for epistemic-free programs]
\label{def-well-supported-disj}
A model $I$ of an epistemic-free program $\Pi$ is {\em well-supported} 
in $\Pi$ if for some head selection function $sel$,
$I$ is well-supported in $\Pi^I_{sel}$ by Definition~\ref{def-well-supported-normal}. 
\end{definition}

\begin{example}
\label{ex:generalized-strat-comp-wellsupp}
{\em 
For the simple disjunctive program $\Pi$  
in Example~\ref{ex:generalized-strat-comp},
we show that $I=\{g_1, g_2, c_1, c_2\}$
 is well-supported in $\Pi$.
Consider a head selection function $sel$ which
selects $c_2$ from the head of rule (3) 
and $c_1$ from the head of rule (4). 
That is, let $sel((3, c_1\mid c_2), I) = c_2$ 
and $sel((4, c_1\mid c_3), I) = c_1$.
The disjunctive program reduct of
$\Pi$ w.r.t.\ $I$ and $sel$ is 
\begin{tabbing} 
\hspace{.3in} $\Pi^{I}_{sel}:\ $ \=  $g_1$\`$(1)$\\[2pt]
\> $g_2$ \` $(2)$ \\[2pt]
\> $c_2 \leftarrow g_1$  \` $(3)$\\[2pt]
\> $c_1 \leftarrow g_2$ \` $(4)$ 
\end{tabbing}
Note that $\Pi^{I}_{sel}$ is an epistemic-free normal program
with atomic rule heads.
We have $\mathit{lfp}(T_{\Pi^{I}_{sel}}(\emptyset,\neg I^-))=\{g_1, g_2, c_1, c_2\}=I$,
so by Theorem~\ref{th-wellsupport-lfp},
$I$ is well-supported in $\Pi^{I}_{sel}$
with a strict well-founded partial order
$g_1\prec c_2$ and $g_2\prec c_1$.
By Definition~\ref{def-well-supported-disj},
$I$ is well-supported in $\Pi$.
}
\end{example}

\begin{example}
\label{ex-min:generalized-strat-comp-wellsupp}
{\em 
For the simple disjunctive program $\Pi$
in Example~\ref{eg-minimal-models},
we check that $I=\{g_1, g_2, c_1, c_2, c_3\}$
is well-supported in $\Pi$.
Consider a head selection function $sel$ which
selects $c_1$ both from the head of rule (3) and 
rule (4). 
This yields the following disjunctive program reduct of
$\Pi$ w.r.t.\ $I$ and $sel$:
\begin{tabbing} 
\hspace{.3in} \= $\Pi^{I}_{sel}:\ $ \=  $g_1$\`$(1)$\\[2pt]
\>\> $g_2$ \` $(2)$ \\[2pt]
\>\> $c_1 \leftarrow g_1$   \` $(3)$\\[2pt]
\>\> $c_1 \leftarrow g_2$ \` $(4)$ \\[2pt]
\> \> $c_2 \leftarrow c_1$ \` $(5)$\\[2pt]
\> \> $c_3 \leftarrow c_1$  \` $(6)$
\end{tabbing} 
We have $\mathit{lfp}(T_{\Pi^{I}_{sel}}(\emptyset,\neg I^-))=\{g_1, g_2, c_1, c_2,c_3\}=I$,
so by Theorem~\ref{th-wellsupport-lfp},
$I$ is well-supported in $\Pi^{I}_{sel}$
with a strict well-founded partial order
$g_1\prec c_1$, $g_2\prec c_1$, $c_1\prec c_2$ and $c_1\prec c_3$.
By Definition~\ref{def-well-supported-disj},
$I$ is well-supported in $\Pi$.
}
\end{example}

When $\Pi$ is an epistemic-free normal program,
Definition~\ref{def-well-supported-disj} reduces to 
Definition~\ref{def-well-supported-normal}.
For every model $I$, there is only one 
head selection function $sel$ and
the disjunctive program reduct $\Pi^{I}_{sel}$ of
$\Pi$ w.r.t.\ $I$ and $sel$ consists of
all rules of $\Pi$ whose rule bodies are
satisfied by $I$. As heads of rules in $\Pi$
whose bodies are not satisfied
by $I$ will not be included in 
the least fixpoint $\mathit{lfp}(T_{\Pi}(\emptyset,\neg I^-))$,
we have $\mathit{lfp}(T_{\Pi}(\emptyset,\neg I^-))=\mathit{lfp}(T_{\Pi^{I}_{sel}}(\emptyset,\neg I^-))$.
Then the following corollary is immediate
from Theorem~\ref{th-wellsupport-lfp-normal}
and Definition~\ref{def-well-supported-disj}.

\begin{corollary}
\label{cor-disj-normal-wellsupp}
A model $I$ of an epistemic-free normal program
$\Pi$ is well-supported in $\Pi$ by
Definition~\ref{def-well-supported-disj} iff 
$I$ is well-supported by 
Definition~\ref{def-well-supported-normal}.
\end{corollary}

\subsection{Well-supportedness of world views for epistemic programs}

We now extend the notion of well-supportedness to world views of
an epistemic program $\Pi$ consisting of rules
of the form $H_1\mid \cdots\mid H_m\leftarrow B$,
where each $H_i$ and $B$ are epistemic formulas.
As $\Pi$ may contain epistemic negations,
Definition~\ref{def-well-supported-disj} is not applicable.

In our definition of a well-supported
model $I$ of an epistemic-free program $\Pi$,
as $I$ is at the level of answer sets,
by the minimality principle 
of answer sets w.r.t.\ negation by default
we want to minimize the knowledge in $I$
by maximizing its negative part $\neg I^-$.
That is,
when its negative part, i.e., every $\neg p\in \neg I^-$, is 
assumed to be true by default,
its positive part, i.e., every $p\in I$, 
is expected to be derived from rules in $\Pi$.
This is achieved via an if-then inference sequence 
$\langle T_{\Pi}^i(\emptyset,\neg I^-)
\rangle_{i=0}^\infty$ with the least fixpoint
$\mathit{lfp}(T_{\Pi}(\emptyset,\neg I^-))$
such that $\mathit{lfp}(T_{\Pi}(\emptyset,\neg I^-))\cup \neg I^-$
entails every $p\in I$ (see Theorem~\ref{th-wellsupport-lfp-normal}).
This inference sequence $\langle T_{\Pi}^i(\emptyset,\neg I^-)
\rangle_{i=0}^\infty$ naturally induces a strict well-founded partial order
$\prec$ (as well as a level mapping) on atoms in $I$, so $I$ is well-supported.

For an epistemic model $\cal A$ of an epistemic program $\Pi$, 
as $\cal A$ is at the level of world views, 
by the minimality principle 
of world views w.r.t.\ epistemic negation
we want to additionally minimize the epistemic 
knowledge of $\Pi$ in $\cal A$.
Let $Ep(\Pi)$ consist of all epistemic negations in $\Pi$.
Then the epistemic knowledge of $\Pi$
consists of all formulas $F$ with $\naf F\in Ep(\Pi)$.
Let $\Phi \subseteq Ep(\Pi)$ be those epistemic negations
that are true in $\cal A$.
We view $\Phi$ as the {\em negative part}
of $\cal A$ w.r.t.\ the epistemic knowledge,
and all formulas $F$ with $\naf F\in Ep(\Pi)\setminus \Phi$
as the {\em positive part}.
For $\cal A$ to be a world view,
we want to minimize the epistemic knowledge in $\cal A$
by maximizing the negative part $\Phi$.
This amounts to requiring that 
when the negative part, i.e., every $\naf F\in \Phi$,
is assumed to be true by default,
the positive part, i.e., every $F$ with 
$\naf F\in Ep(\Pi)\setminus \Phi$, should be derived
from the rules in $\Pi$.
To achieve this, we present a transformation of $\Pi$
into an epistemic-free program $\Pi^{\cal A}$ w.r.t.\ $\cal A$,
where: (1) as all $\naf F\in \Phi$ are assumed
to be true by default, we replace them in $\Pi$
with $\top$, this transforming $\Pi$ to $\Pi'$;
and (2) we expect that every $F$ with 
$\naf F\in Ep(\Pi)\setminus \Phi$ is true in $\Pi'$,
which by Theorem 1 
in \cite{ShenEiter16},\footnote{The theorem is as follows:
Let $\Pi$ be an epistemic program, where 
for every $\naf F$
in $\Pi$, $F$ is true in every epistemic model of $\Pi$.
Then $\Pi$ has the same epistemic models as $\Pi^\neg$, 
which is $\Pi$ with every epistemic negation
$\naf F$ replaced by $\neg F$.} 
amounts to 
that every $F$ with $\naf F\in Ep(\Pi)\setminus \Phi$ 
is true in $\Pi^{\cal A}$ that is $\Pi'$ with
every epistemic negation $\naf F$
replaced by $\neg F$. Formally, we have the 
following definition of the transformation.

\begin{definition}
\label{reduct}
Let $\Pi$ be an epistemic program and  
$\cal A$ be an epistemic model of $\Pi$.
Let $\Phi \subseteq Ep(\Pi)$ be the set of
all epistemic negations occurring in $\Pi$
that are true in $\cal A$.
The {\em epistemic reduct} $\Pi^{\cal A}$ 
of $\Pi$ w.r.t.\ $\cal A$ is $\Pi$ 
with every $\naf F\in \Phi$ replaced by $\top$
and every $\naf F\in Ep(\Pi)\setminus \Phi$
replaced by $\neg F$. 
\end{definition}

As $\cal A$ is an epistemic model of $\Pi$,
by Definition~\ref{satisfaction-epis-program}
every $I\in \cal A$ satisfies all rules in $\Pi$ w.r.t.\ $\cal A$,
and by Theorem 1 in \cite{ShenEiter16}
every $I\in \cal A$ satisfies all rules in $\Pi^{\cal A}$.
That is, every $I\in \cal A$ is a model of $\Pi^{\cal A}$.

$\Pi^{\cal A}$ can be viewed as a {\em mirror}
of $\Pi$ w.r.t.\ $\cal A$ from the level of world views 
down to the level of answer sets under the minimality principle 
of world views w.r.t.\ epistemic negation.
Therefore, we may well define the well-supportedness of 
$\cal A$ in $\Pi$ at the level of world views
in terms of the well-supportedness 
of $A$ in the mirror $\Pi^{\cal A}$ of $\Pi$
at the level of answer sets; i.e.,
$\cal A$ is well-supported in $\Pi$ if every $I\in \cal A$
is well-supported in $\Pi^{\cal A}$.

\begin{definition}[Well-supportedness for epistemic programs]
\label{def-well-supported-epis}
Let $\Pi$ be an epistemic program.  
An epistemic model $\cal A$ of $\Pi$
is {\em well-supported} in $\Pi$ if every $I\in \cal A$
is well-supported by Definition~\ref{def-well-supported-disj}
in the epistemic reduct $\Pi^{\cal A}$
of $\Pi$ w.r.t.\ $\cal A$.
\end{definition}

\begin{example}
\label{ex3-epigeneralized-strat-comp-wellsupp}
{\em 
Consider the epistemic program $\Pi_1$ 
in Example~\ref{ex1-modal:generalized-strat-comp}, where
${\cal W}=\{I\}$ with $I=\{a,b,c\}$
is an expected world view of $\Pi_1$. We next show 
that ${\cal W}$ is well-supported in $\Pi_1$
by Definition~\ref{def-well-supported-epis}.

As ${\bf K}$ stands for $\neg \mathbf{not}$,
we first replace rule ($4'$) by the rule
\begin{tabbing} 
\hspace{.3in} $\bot \leftarrow \neg \neg\naf c$  \` $(4'')$
\end{tabbing}
As $\naf c$ is false in ${\cal W}$,
the interpretation $I$ in ${\cal W}$
satisfies all rules in $\Pi_1$ w.r.t.\ ${\cal W}$. So
${\cal W}$ is an epistemic model of $\Pi_1$.

Let $\Phi=\emptyset$, which is the set of
all epistemic negations occurring in $\Pi_1$
that are true in $\cal W$.
Then $Ep(\Pi_1)\setminus \Phi=\{\naf c\}$. 
The epistemic reduct $\Pi_1^{\cal W}$ 
w.r.t.\ $\cal W$ is obtained from $\Pi_1$
by replacing $\naf c$ in rule ($4''$) with $\neg c$, i.e., we obtain
\begin{tabbing} 
\hspace{.3in} $\Pi_1^{\cal W}:\ $ \= $a \mid b$ \` $(1)$\\[2pt]
\>          $b \leftarrow a$ \` $(2)$\\[2pt]
\>         $c \leftarrow a$  \` $(3)$\\[2pt]
\>        $\bot \leftarrow\neg \neg\neg c$  \` $(4)$
\end{tabbing}

Note that as our definition of an epistemic program (Definition~\ref{def-episP})
is based on a propositional language 
${\cal L}_\Sigma$ (see Section \ref{Preliminaries}),
 $\neg\neg\neg c$ in rule (4) is a formula 
 in classical logic, where
double negatives will cancel each other out, i.e.,
$\neg \neg\neg c\equiv \neg c$.

Consider $I=\{a,b,c\}$ in ${\cal W}$
and a head selection function $sel$ which
selects $a$ from the head of rule (1). 
The disjunctive program reduct of
$\Pi_1^{\cal W}$ w.r.t.\ $I$ and $sel$ is 
\begin{tabbing} 
\hspace{.3in} $(\Pi_1^{\cal W})^{I}_{sel}:\ $ \=  $a$ \` $(1)$\\[2pt]
\>          $b \leftarrow a$ \` $(2)$\\[2pt]
\>         $c \leftarrow a$  \` $(3)$
\end{tabbing}

We have $\mathit{lfp}(T_{(\Pi_1^{\cal W})^{I}_{sel}}(\emptyset,\neg I^-))=\{a,b,c\}=I$,
so by Theorem~\ref{th-wellsupport-lfp},
$I$ is well-supported in $(\Pi_1^{\cal W})^{I}_{sel}$.
Then by Definition~\ref{def-well-supported-disj},
$I$ is well-supported in $\Pi_1^{\cal W}$.
Finally by Definition~\ref{def-well-supported-epis},
${\cal W}$ is well-supported in $\Pi_1$.
}
\end{example}

\begin{example}
\label{eg-well-suppMp-epi}
{\em Consider the epistemic program
$\Pi=\{p\leftarrow {\bf M} p\}.$
As ${\bf M}$ is shorthand for $\naf \neg$,
this program can be rewritten as $\Pi=\{p\leftarrow \naf \neg p\}$.
The rule in $\Pi$ states that if there is no evidence proving $\neg p$, then 
we jump to conclude $p$.
This program has two epistemic models: ${\cal A}=\{\{p\}\}$ 
and ${\cal A}_1=\{\emptyset\}$.

Consider ${\cal A}=\{\{p\}\}$. 
Let $\Phi=\{\naf \neg p\}$ be the set of
all epistemic negations occurring in $\Pi$
that are true in $\cal A$.
The epistemic reduct  
of $\Pi$ w.r.t.\ $\cal A$ is $\Pi^{\cal A}= \{p\leftarrow \top\}$.
Consider $I=\{p\}$ in ${\cal A}$.
As $\Pi^{\cal A}$ contains no disjunctive rule head,
we directly apply Definition~\ref{def-well-supported-normal-atomhead}.
We have $\mathit{lfp}(T_{\Pi^{\cal A}}(\emptyset,\neg I^-))=\{p\}=I$,
so by Theorem~\ref{th-wellsupport-lfp},
$I$ is well-supported in $\Pi^{\cal A}$.
Then by Definition~\ref{def-well-supported-epis},
${\cal A}$ is well-supported in $\Pi$.

Consider ${\cal A}_1=\{\emptyset\}$. 
Let $\Phi=\emptyset$ be the set of
all epistemic negations occurring in $\Pi$
that are true in ${\cal A}_1$.
The epistemic reduct  
of $\Pi$ w.r.t.\ ${\cal A}_1$ is $\Pi^{{\cal A}_1}= \{p\leftarrow \neg\neg p\}$,
where $\naf \neg p$ in $\Pi$ is replaced by $\neg\neg p$.
Consider $I=\emptyset$ in ${\cal A}_1$.
We have $\mathit{lfp}(T_{\Pi^{{\cal A}_1}}(\emptyset,\neg I^-))=\emptyset=I$,
so by Theorem~\ref{th-wellsupport-lfp},
$I$ is well-supported in $\Pi^{{\cal A}_1}$.
Then by Definition~\ref{def-well-supported-epis},
${\cal A}_1$ is well-supported in $\Pi$.

As a result, both ${\cal A}=\{\{p\}\}$ 
and ${\cal A}_1=\{\emptyset\}$
are well-supported epistemic models of $\Pi$.
However,  ${\cal A}_1$ would not be a
world view because it violates the minimality
property w.r.t.\ epistemic negation. We will discuss this
later in Example~\ref{ex-Mp-Kp}.
}
\end{example}

\begin{example}
\label{eg-well-suppKp-epi}
{\em 
Consider the epistemic program $\Pi=\{p\leftarrow {\bf K} p\}.$
As ${\bf K}$ stands for $\neg\naf$, 
this program can be rewritten as $\Pi=\{p\leftarrow \neg\naf p\}$.
$\Pi$ has two epistemic models: ${\cal A}=\{\emptyset\}$ 
and ${\cal A}_1=\{\{p\}\}$.

Consider ${\cal A}=\{\emptyset\}$. 
Let $\Phi=\{\naf p\}$ be the set of
all epistemic negations occurring in $\Pi$
that are true in $\cal A$.
The epistemic reduct  
of $\Pi$ w.r.t.\ $\cal A$ is $\Pi^{\cal A}= \{p\leftarrow \neg\top\}$.
Consider $I=\emptyset$ in ${\cal A}$.
We have $\mathit{lfp}(T_{\Pi^{\cal A}}(\emptyset,\neg I^-))=\emptyset=I$,
so by Theorem~\ref{th-wellsupport-lfp},
$I$ is well-supported in $\Pi^{\cal A}$.
Then by Definition~\ref{def-well-supported-epis},
${\cal A}$ is well-supported in $\Pi$.

Consider ${\cal A}_1=\{\{p\}\}$. 
Let $\Phi=\emptyset$ be the set of
all epistemic negations occurring in $\Pi$
that are true in ${\cal A}_1$.
The epistemic reduct  
of $\Pi$ w.r.t.\ ${\cal A}_1$ is $\Pi^{{\cal A}_1}= \{p\leftarrow \neg\neg p\}$,
where $\naf p$ in $\Pi$ is replaced by $\neg p$.
Consider $I=\{p\}$ in ${\cal A}_1$.
We have $\mathit{lfp}(T_{\Pi^{{\cal A}_1}}(\emptyset,\neg I^-))=\emptyset \neq I$,
so by Theorem~\ref{th-wellsupport-lfp},
$I$ is not well-supported in $\Pi^{{\cal A}_1}$.
Thus by Definition~\ref{def-well-supported-epis},
${\cal A}_1$ is not well-supported in $\Pi$.

As a result, only ${\cal A}=\{\emptyset\}$
is a well-supported epistemic model of $\Pi$.
As will be discussed in Example~\ref{ex-Mp-Kp},
 ${\cal A}$ also satisfies the minimality properties
w.r.t.\ negation by default and epistemic negation.
}
\end{example}

Finally, when $\Pi$ is an epistemic-free program,
we have $Ep(\Pi)=\emptyset$ 
and thus $\Pi^{\cal A}=\Pi$ for every epistemic model ${\cal A}$. 
Then, ${\cal A}$ is well-supported in $\Pi$ iff every $I\in \cal A$
is well-supported in $\Pi$ by Definition~\ref{def-well-supported-disj}.

\section{Defining ASP Semantics in Terms of the Refined GAS Principles}
\label{sec-rational-semantics}

The refined GAS principles presented 
in the Introduction characterize answer set 
and world view construction as a 
SAT problem for a
propositional theory customized for a logic program,  
where the model selection
follows the three general principles (RP1)-(RP3).
In this section, we define ASP semantics  
guided by these principles. 
In particular, we define an answer set
semantics that fulfills the principles of
well-supportedness and
minimality of answer sets w.r.t.\ negation by default,
and a world view semantics that additionally
fulfills the principle of minimality of world views w.r.t.\ 
epistemic negation.

\subsection{Rational answer set semantics}
\label{rational-answerset}
At the level of answer sets,
by considering the principles of 
well-supportedness and minimality
w.r.t.\ negation by default
we obtain a new definition of answer sets,
called {\em rational answer set semantics},
for epistemic-free programs.
Here, we borrow the word ``rational" from
the Gelfond's {\em rationality} principle
because we refine the rationality principle
to the well-supportedness and
minimality principles.

\begin{definition}[Rational answer set semantics]
\label{def-ideal-ans-semantics}
A model $I$ of an epistemic-free program $\Pi$
is an answer set of $\Pi$
if $I$ is well-supported in $\Pi$ 
by Definition~\ref{def-well-supported-disj}
and no other model $J$ of $\Pi$ that is
a proper subset of $I$ is well-supported in $\Pi$.
\end{definition}

Definition~\ref{def-ideal-ans-semantics} says that
a model $I$ is an answer set under the rational 
answer set semantics
iff $I$ is a minimal well-supported model of $\Pi$.
This fulfills the two general principles for answer sets
(well-supportedness and minimality
w.r.t.\ negation by default)
and thus the rational answer set semantics
embodies the refined GAS principles.
Note that a minimal well-supported model of $\Pi$
is not necessarily a minimal model of $\Pi$.

The rational answer set semantics does not satisfy
the minimal model property  
nor constraint monotonicity nor foundedness.

\begin{example}
\label{ex:rationalanswer-minwellsupp}
{\em 
Consider the simple disjunctive program $\Pi$  
in Example~\ref{ex:generalized-strat-comp}.
We showed in Example~\ref{ex:generalized-strat-comp-wellsupp} 
that the expected answer set
$I=\{g_1, g_2, c_1, c_2\}$ is a well-supported model of $\Pi$.
It is easy to check that $I$ is also a minimal model of $\Pi$.
Therefore, $I$ is a minimal well-supported model of $\Pi$,
and by Definition~\ref{def-ideal-ans-semantics} it
is an answer set of $\Pi$ under the rational answer set semantics.
As shown in Example~\ref{ex:generalized-strat-comp},
this answer set violates constraint monotonicity and foundedness.

Similarly, the expected answer set
$I=\{g_1, g_2, c_1, c_2, c_3\}$
in Example~\ref{eg-minimal-models} 
is a minimal well-supported model
and thus is an answer set of $\Pi$ 
under the rational answer set semantics.
This answer set violates the minimal model property.

One can also check that the expected answer sets
in Examples \ref{ex1-cmfn} and \ref{ex-nonminimal} are
answer sets under the rational answer set semantics.
}
\end{example}

Note that for an epistemic-free normal program $\Pi$,
when a model $I$ is well-supported, i.e.,
$\mathit{lfp}(T_{\Pi}(\emptyset,\neg I^-))\cup \neg I^-\models p$
for every $p\in I$, $I$ must be a minimal model of $\Pi$,
as the provability operator $T_\Pi(O,N)$ is monotone
w.r.t.\ the parameter $O$ (see Definition~\ref{def-one-step-operator}).
This means that every answer set 
of an epistemic-free normal program
under the rational answer set semantics
is a minimal model. Therefore,
the rational answer set semantics satisfies
the minimal model property  
for epistemic-free normal programs.

\leanparagraph{A relaxed constraint monotonicity}
Constraint monotonicity does not hold
at the level of answer sets (i.e., minimal well-supported models)
under the rational answer set semantics.
However, at the level of well-supported models,
one can easily show that
for every epistemic-free program $\Pi$
and constraint $C$,
well-supported models of $\Pi\cup \{C\}$ 
are well-supported models of $\Pi$ satisfying $C$.
This provides a relaxed constraint monotonicity.

For an epistemic-free normal program $\Pi$,
as every well-supported model is a minimal model,
an interpretation $I$ is an answer set 
under the rational answer set semantics
iff $I$ is a well-supported model.
Then, for every constraint $C$,
answer sets of $\Pi\cup \{C\}$ 
are answer sets of $\Pi$ satisfying $C$.
That is, the rational answer set semantics satisfies
constraint monotonicity
for epistemic-free normal programs.

\medskip

We summarize the following {\em guess and check} steps
as a guide for answer set construction under the rational 
answer set semantics.

\leanparagraph{Guess and check for answer sets}
Whether a model $I$ of an epistemic-free program $\Pi$
is an answer set under the rational 
answer set semantics can be determined
via a guess and check process
in the following three steps.
\begin{enumerate}
\item
{\bf Guess a head selection function:}
Guess a function $sel$ that selects one $H_i$ 
satisfied by $I$ from every
rule head $H_1\mid \cdots\mid H_m$ in $\Pi$
and compute the disjunctive program reduct $\Pi^I_{sel}$
of $\Pi$ w.r.t.\ $I$ and $sel$ (see Definition~\ref{def:sel-reduct}).

\item
{\bf Check well-supportedness:}
Compute the least fixpoint 
$\mathit{lfp}(T_{\Pi^I_{sel}}(\emptyset,\neg I^-))$
via the inference sequence 
$\langle T_{\Pi^I_{sel}}^i(\emptyset,\neg I^-)
\rangle_{i=0}^\infty$
and determine that $I$ is well-supported by checking
$\mathit{lfp}(T_{\Pi^I_{sel}}(\emptyset,\neg I^-))\cup \neg I^-\models p$
for every $p\in I$ (see Theorem~\ref{th-wellsupport-lfp-normal}).

\item
{\bf Check minimality:}
Check that there is no other well-supported model $J$ that is
a proper subset of $I$ (see Definition~\ref{def-ideal-ans-semantics}).
\end{enumerate}

Note that when $\Pi$ is an epistemic-free normal program,
Steps 1 and 3 can be ignored. In this case,
when $I$ is well-supported, i.e.,
$\mathit{lfp}(T_{\Pi}(\emptyset,\neg I^-))\cup \neg I^-\models p$
for every $p\in I$, $I$ must be a minimal model of $\Pi$,
as the provability operator $T_\Pi(O,N)$ is monotone
w.r.t.\ the parameter $O$ (see Definition~\ref{def-one-step-operator}).

As another typical example, we next demonstrate that the well-known
{\em Hamiltonian cycle problem} can be encoded
as an epistemic-free program whose answer sets
under the rational answer set semantics
correspond to solutions of the problem.
The Hamiltonian cycle problem says that
given a directed graph $G=(V, E)$ with nodes
$V$ and edges $E\subseteq V\times V$,
find paths that visit every node in $V$ exactly once. 
Such a path forms a cycle, called
a {\em Hamiltonian cycle}.

\begin{example}[Hamiltonian cycles]
\label{ex-Hamiltonian-cycles}
{\em   
The following epistemic-free program formalizes
an instance of the Hamiltonian cycle problem, where 
for illustration, we simply consider a graph
$G$ with three nodes $V=\{a, b, c\}$ 
and four directed edges $E=\{(a,c), (b,a), (b,c), (c,b) \}$.
It is easy to check that $G$ has only one Hamiltonian cycle
$HC=\{(a,c), (c,b),(b,a)\}$.
\begin{tabbing} 
\hspace{.2in} $\Pi:\quad$ \= $node(a)$ \`$(1)$\\
 \>         $node(b)$ \`$(2)$\\
  \>         $node(c)$ \`$(3)$\\
\>          $edge(a,c)$ \`$(4)$\\
\>          $edge(b,a)$ \`$(5)$\\
\>          $edge(b,c)$ \`$(6)$\\
\>          $edge(c,b)$ \`$(7)$\\
\>   $in(X,Y)\mid \neg in(X,Y)\leftarrow edge(X,Y)$ \`$(8)$\\
\>   $path(X,Y) \leftarrow in(X,Y)$  \`$(9)$\\
\>   $path(X,Y) \leftarrow in(X,Z)\wedge path(Z,Y)$  \`$(10)$\\
\>   $\bot \leftarrow in(X,Y)\wedge in(X,Z)\wedge Y\neq Z$ \`$(11)$\\
\>   $\bot \leftarrow in(X,Y)\wedge in(Z,Y)\wedge X\neq Z$ \`$(12)$\\
\>   $\bot \leftarrow node(X)\wedge node(Y) \wedge \neg path(X,Y)$ \`$(13)$
\end{tabbing}
In the program, $X$, $Y$ and $Z$ are free variables.\footnote{In 
ASP, a logic program $\Pi$ with free variables is viewed as 
shorthand for a propositional program $ground(\Pi)$, 
which is obtained from $\Pi$ 
by replacing every free variable with all constants in $\Pi$.
The ASP semantics of $\Pi$ is then 
defined by the ASP semantics of $ground(\Pi)$.} 
We use a predicate $node(X)$ 
to represent a node $X$, and
$edge(X,Y)$ to represent a directed edge 
from node $X$ to node $Y$.
For a Hamiltonian cycle $HC$,
we use $in(X,Y)$ to represent 
that an edge $(X,Y)$ is on $HC$,
and use $path(X,Y)$ to represent 
that there is a directed path on $HC$
from node $X$ to node $Y$.

Rule (8) presents two alternatives
for every edge in $E$, i.e., it could be on or not on $HC$.\footnote{
In \cite{gelfond-kahl2014}, a rule like
$in(X,Y)\mid$ $\sim$$in(X,Y)\leftarrow edge(X,Y)$ 
was used to represent the choices,
where $\sim$$in(X,Y)$ is a {\em strong negation}
used in the GL-semantics \cite{GL91}.
In this paper, we directly use $\neg$ to represent negation as in classical logic.
As described in \cite{GL91},
for any atom $p(X_1,\cdots, X_n)$, the strong negation $\sim$$p(X_1,\cdots, X_n)$
can be compiled away by replacing it with a fresh predicate 
$p'(X_1,\cdots, X_n)$ along with a constraint rule 
$\bot\leftarrow p(X_1,\cdots, X_n)\wedge p'(X_1,\cdots, X_n)$.}
Rules (9) and (10) inductively define a directed path on $HC$.
Rules (11)-(13) are constraints saying that 
$HC$ must visit every node exactly once; i.e.,
for every node $X$ in $V$, $HC$ has no more than one edge
of the form $(X, \_)$ (rule (11)) and one edge
of the form $(\_, X)$ (rule (12)), and 
for any two nodes $X$ and $Y$ in $V$, $HC$ has
a path from $X$ to $Y$ (rule (13)). 

Then, $G$ has a Hamiltonian cycle $HC$
iff $\Pi$ has an answer set $I$ under the 
rational answer set semantics,
where for any two nodes $X$ and $Y$ in $V$,
an edge $(X,Y)$ in $E$ is on $HC$ iff $in(X,Y)$ is in $I$.

As an exercise, the interesting reader may follow the 
above guess and check steps,
specifically by guessing a function $sel$ that selects 
either $in(X,Y)$ or $\neg in(X,Y)$
for every $edge(X,Y)$ (rule (8)) without violating
the constraints (rules (11)-(13)),
to derive from $\Pi$ the following answer set 
(minimal well-supported model) under the 
rational answer set semantics:
\begin{tabbing} 
\hspace{.2in} 
$I=\{$\=$node(a),node(b),node(c),edge(a,c),edge(b,a),edge(b,c),edge(c,b),$\\
\>       $in(a,c),in(c,b),in(b,a),path(a,c),path(c,b),path(b,a),$\\
\>        $path(a,b),path(c,a),path(b,c),path(a,a),path(b,b),path(c,c)\}.$
\end{tabbing}
Then, the collection of all instances of $in(X,Y)$ in $I$, i.e.,
$\{in(a,c),in(c,b),in(b,a)\}$,
constitutes a Hamiltonian cycle 
$HC=\{(a,c), (c,b),(b,a)\}$ in $G$.
}
\end{example}

\subsection{Rational world view semantics}
\label{rational-worldview}
At the level of world views,
by further considering the principle of minimality 
of world views
w.r.t.\ epistemic negation
we obtain a new definition of world views, 
called {\em rational world view semantics},
for epistemic programs. 

\begin{definition}[Rational world view semantics]
\label{def-ideal-wv-semantics}
An epistemic model $\cal A$ of an epistemic program $\Pi$
is a world view of $\Pi$
if (1) $\cal A$ coincides with the collection of 
all answer sets of the epistemic reduct $\Pi^{\cal A}$ 
of $\Pi$ w.r.t.\ ${\cal A}$ under the 
rational answer set semantics, 
and (2) there exists no other epistemic model ${\cal W}$ of $\Pi$ 
satisfying condition (1) with $\Phi_{\cal A}\subset \Phi_{\cal W}$,
where $\Phi_{\cal A}$ and $\Phi_{\cal W}$ 
are the sets of
all epistemic negations occurring in $\Pi$
that are true in ${\cal A}$ and $\cal W$, respectively.
\end{definition}

Note that for an epistemic-free program $\Pi$,
the rational world view semantics
reduces to the rational answer set semantics,
i.e., an epistemic model $\cal A$ is a world view of $\Pi$
under the rational world view semantics
iff $\cal A$ consists of all answer sets of $\Pi$  
under the rational answer set semantics.

Condition (1) of a world view $\cal A$
in Definition~\ref{def-ideal-wv-semantics}
says that after an epistemic program $\Pi$ is transformed
to an epistemic-free program $\Pi^{\cal A}$
from the world view level down to
the answer set level, $\cal A$
must coincide with the collection of
all answer sets of $\Pi^{\cal A}$ under the
rational answer set semantics because
at the answer set level,
the rational world view semantics
reduces to the rational answer set semantics.
This means that at the answer set level,
the rational world view semantics
fulfills the two general principles for answer sets,
viz. well-supportedness and minimality
w.r.t.\ negation by default.
 
At the world view level, 
condition (2) of a world view in 
Definition~\ref{def-ideal-wv-semantics}
says that the rational world view semantics
fulfills minimality w.r.t.\ epistemic negation, i.e., 
every world view $\cal A$ of $\Pi$
satisfies a maximal set of epistemic negations in $\Pi$.
Therefore, the rational world view semantics
embodies the refined GAS principles.

\begin{example}
\label{ex-rationalWVexamples}
{\em 
We showed in Example~\ref{ex3-epigeneralized-strat-comp-wellsupp}
that the expected world view ${\cal W}=\{I\}$ with $I=\{a,b,c\}$
of the epistemic program $\Pi_1$
in Example~\ref{ex1-modal:generalized-strat-comp}
is well-supported in $\Pi_1$.
Next we show that ${\cal W}$
is a world view of $\Pi_1$ under the rational 
world view semantics. 

In Example~\ref{ex3-epigeneralized-strat-comp-wellsupp},
we showed that $I=\{a,b,c\}$ 
is well-supported in the epistemic reduct $\Pi_1^{\cal W}$ 
of $\Pi_1$ w.r.t.\ $\cal W$. 
In addition to $I$, $\Pi_1^{\cal W}$ has only one more model 
 $J=\{b,c\}$. It is easy to check that 
$J$ is not well-supported in $\Pi_1^{\cal W}$. 
Therefore, $I$ is a minimal well-supported model of $\Pi_1^{\cal W}$
and by Definition~\ref{def-ideal-ans-semantics} it
is an answer set of $\Pi_1^{\cal W}$ under 
the rational answer set semantics.
As $\Pi_1^{\cal W}$ has only two models $I$ and $J$,
and $J$ is not well-supported in $\Pi_1^{\cal W}$, 
$I$ is the only minimal well-supported model
and thus the only answer set of $\Pi_1^{\cal W}$. 
As $\cal W$ coincides with the collection of 
all answer sets of $\Pi_1^{\cal W}$ under 
the rational answer set semantics,
condition (1) in Definition~\ref{def-ideal-wv-semantics}
for $\cal W$ to be a world view of $\Pi_1$
is satisfied.

Note that $\Pi_1$ contains only one epistemic negation, i.e.,
$\naf c$ in rule ($4'$). As rule ($4'$) is a constraint saying
that $\naf c$ will not be true in any epistemic model
of $\Pi_1$, there exists no
epistemic model satisfying more epistemic negations 
occurring in $\Pi_1$ than $\cal W$.
So condition (2) of Definition~\ref{def-ideal-wv-semantics}
for $\cal W$ to be a world view of $\Pi_1$
is satisfied. 
Therefore, ${\cal W}$
is a world view of $\Pi_1$ under the rational 
world view semantics. 

As shown in Examples~\ref{ex1-modal:generalized-strat-comp}
and \ref{ex3-epi:generalized-strat-comp},
this world view ${\cal W}$ violates both constraint monotonicity 
and foundedness.
Therefore, the rational world view semantics
does not satisfy constraint monotonicity
nor foundedness.
}
\end{example}

\begin{example}
\label{ex-Mp-Kp}
{\em
Consider the epistemic program
$\Pi=\{p\leftarrow {\bf M} p\}=\{p\leftarrow \naf \neg p\}$ 
in Example~\ref{eg-well-suppMp-epi}.
It is easy to check that ${\cal A}=\{\{p\}\}$
and ${\cal A}_1=\{\emptyset\}$
are the collections of 
all answer sets of the epistemic reducts 
$\Pi^{\cal A}$  and $\Pi^{{\cal A}_1}$
of $\Pi$ w.r.t.\ ${\cal A}$ and ${\cal A}_1$,
respectively under the 
rational answer set semantics.
${\cal A}$ also satisfies 
the minimality property w.r.t.\ epistemic negation,
so it is a world view of $\Pi$ 
under the rational world view semantics.
However, ${\cal A}_1$ is not a world view.
Although at the answer set level,
${\cal A}_1$ satisfies the properties of well-supportedness 
and minimality w.r.t.\ negation by default,
at the world view level it violates 
the minimality property w.r.t.\ epistemic negation
(it does not satisfy a maximal set of epistemic negations 
occurring in $\Pi$).

In contrast, for the epistemic program
$\Pi=\{p\leftarrow {\bf K} p\}=\{p\leftarrow \neg\naf p\}$ 
in Example~\ref{eg-well-suppKp-epi},
it is easy to check
that ${\cal A}=\{\emptyset\}$ is a world view
under the rational world view semantics,
but ${\cal A}_1=\{\{p\}\}$ is not, because 
at the answer set level,
$\{p\}\in {\cal A}_1$ is not a well-supported model of $\Pi^{{\cal A}_1}$
and thus ${\cal A}_1$ does not satisfy the property 
of well-supportedness.
}
\end{example}

The intuitive behavior of the epistemic program
$\Pi = \{p\leftarrow {\bf M} p\}$ 
has been an open debate.
$\Pi$ has different world views
under the existing world view 
semantics in the literature.
For instance, 
in the world view semantics of
\cite{Gelfond2011,CabalarFC20aij},
$\Pi$ has two world views ${\cal A} = \{\{p\}\}$ 
and ${\cal A}_1 = \{\emptyset\}$; 
thus we can either conclude that $p$ is true in $\Pi$ 
or that $p$ is false.
However, 
$\Pi$ has only one world view ${\cal A}$ in the semantics of
\cite{KahlGelfond2015,Fa2015,ShenEiter16,SuCH20},
and only one world view ${\cal A}_1$
in the semantics of \cite{ZhangZ17a}. 
The question then is what are intuitive world views
of this epistemic program.

In this paper, we justify
in view of the refined GAS principles
that ${\cal A} = \{\{p\}\}$
is the only intuitive world view  
of $\Pi = \{p\leftarrow {\bf M} p\}$,
which satisfies the well-supportedness and 
minimality principles both at the level of
answer sets and of world views.

A rule of the form $p\leftarrow \naf \neg p$
(or $p\leftarrow {\bf M} p$) expresses that if there is no evidence to 
prove $\neg p$ to be true, we conclude $p$; put another way, 
we admit $p$ unless we can prove its contrary.
Such epistemic rules are commonly used in our daily life. 
Consider the following close example
introduced in \cite{ShenEiter16}.

\begin{example}
\label{group-meeting}
{\em
Our research group has a regular weekly meeting and the secretary 
reminds us every week by a {\em no-reply email} with the assumption that 
every member in the group will be present 
unless some notice of absence is available.
This epistemic assumption can be suitably formalized by a rule of the form
\[\mathit{present(X)\leftarrow member(X)\wedge \naf \neg present(X),}\] 
i.e., for every  group member $X$, 
if there is no evidence to prove $X$ will not be present,
then $X$ will be present.
Then for an individual member $\mathit{John}$ 
with no notice to the secretary
saying that he will not be present, we would derive from the 
above rule an epistemic program
\[\mathit{\Pi = \{member(John),\ present(John)\leftarrow member(John)\wedge \naf \neg present(John)\},}\] 
which has a single world view ${\cal A}=\mathit{\{\{member(John),\ present(John)\}\}}$
under the rational world view semantics, 
meaning that $\mathit{John}$ will
be present in the meeting.
}
\end{example}

\section{Extension to Handling Choice Constructs}
\label{sec:extension-choiceCons}
To embrace different scenarios and 
satisfy more application requirements,
some built-in functions as extensions to ASP
have been introduced to logic programs, including
aggregates or abstract constraint atoms \cite{FLP04,FaberPL11},
choice constructs \cite{SNS02,FerrarisL05,CalimeriFGIKKLM20},
set introduction rules \cite{GelfondZ19}, etc.
As choice constructs are used as an essential extension
in ASP programming languages 
such as ASP-Core-2.0 \cite{CalimeriFGIKKLM20},
in this section we formally define answer set semantics for
epistemic-free programs extended with choice constructs.

For a set $S$ of atoms, let $|S|$ denote
the number of atoms in $S$.
In general, a {\em choice construct} is
of the form $u_1 \{a_1,\cdots,a_m\} u_2$, 
where $m>0$, $0\leq u_1\leq u_2\leq m$,
and the $a_i$'s are ground atoms.
For answer set construction, a choice construct
 $u_1 \{a_1,\cdots,a_m\} u_2$ presents, for every subset 
$S$ of $ \{a_1,\cdots,a_m\}$
with $u_1\leq |S|\leq u_2$,
a {\em choice} that all atoms in $S$ 
are true and the other atoms in 
$\{a_1,\cdots,a_m\}$ are false. 
Such a choice is represented by a {\em choice formula} of the form
\[\bigwedge_{a\in S} a\ \wedge\ \bigwedge_{a\in \{a_1,\cdots,a_m\}\setminus S} \neg a\]

An epistemic-free program whose rules are extended with 
choice constructs can be viewed as a compact representation of
a collection of epistemic-free programs in which
every choice construct is replaced by a choice formula.

\begin{definition}
Let $\Pi$ be an epistemic-free program extended with 
choice constructs. A {\em choice program} of $\Pi$ 
w.r.t.\ the choice constructs
is $\Pi$ with every choice construct 
replaced by one of its choice formulas.
\end{definition}

Then the answer set semantics of an epistemic-free program extended 
with choice constructs
is defined in terms of its choice programs.

\begin{definition}
\label{sem-choiceprog}
Let $\Pi$ be an epistemic-free program extended with 
choice constructs and ${\cal X}$ be an answer set
semantics for epistemic-free programs.
Then the set of answer sets of $\Pi$ under ${\cal X}$ is 
the set of answer sets of all choice programs of $\Pi$ under ${\cal X}$.
\end{definition}

By Definition~\ref{sem-choiceprog},
any answer set semantics ${\cal X}$ for epistemic-free programs
can be extended to handling programs with choice constructs.
In the following examples, we let ${\cal X}$ be the rational answer set semantics.

To see the necessity and usefulness of choice constructs,
we consider the following pin-touching-colored board problem.

\begin{example}
\label{pin-dropping}
{\em
Suppose that we have a pin and a board painted black and white. 
The {\em pin-touching-colored board problem} is to 
formulate in ASP the scenario that the pin will touch either a black area 
or a white area, or both, if it is dropped on the board. 
Specifically,
let $drop$, $black$ and $white$ be three predicates, where
$drop$ represents that the pin is dropped on the board,
and $black$ and $white$ represent that the pin touches a black and
a white area, respectively.
Then, we aim to use these three predicates
to write a logic program $\Pi$
whose intended models/answer sets are
$I_1=\{drop, black\}$, $I_2=\{drop, white\}$ and
$I_3=\{drop, black, white\}$, which correspond to
the three solutions to the pin-touching-colored board problem.

The difficulty with this pin-touching-colored board problem
is that its solutions are not subset-minimal w.r.t.\ the two colors
$\{black, white\}$ that the pin may touch.
Therefore, we cannot formulate it using an epistemic-free program
whose answer sets must be subset-minimal (due to the minimality principle w.r.t.
negation by default).\footnote{For example, the pin-touching-colored board problem
cannot be formulated using either of the following epistemic-free programs:
 \begin{tabbing}
\hspace{.3in} \= $\Pi_1:\ $ \= $drop$ \` $(1)$ \\
\>\> $black\mid white \leftarrow drop$ \` $(2)$ \\[.05in]
\>      $\Pi_2:\ $ \> $drop$ \` $(1)$ \\
\>\> $black\wedge \neg white \mid\neg black\wedge white \mid black\wedge white \leftarrow drop$ \` $(2)$ 
\end{tabbing}}

However, we can formulate
the pin-touching-colored board problem using 
the following epistemic-free program extended 
with a choice construct:
 \begin{tabbing}
\hspace{.3in} \= $\Pi:\ $ \= $drop$ \` $(1)$ \\
\>\> $1 \{black,white\} 2 \leftarrow drop$ \` $(2)$ 
\end{tabbing}
Rule (1) says that a pin is dropped on the board, and rule (2) says 
that if a pin is dropped on the board, it will touch either a black 
area or a white area, or both.

For answer set construction, the choice construct 
$1 \{black,white\} 2$ in $\Pi$ presents three choices represented
by the three choice formulas: $black\wedge \neg white$, $\neg black\wedge white$,
and $black\wedge white$.
By replacing $1 \{black,white\} 2$ with the three choice formulas, respectively,
we obtain the following three choice programs:
 \begin{tabbing}
\hspace{.3in} \= $\Pi_1:\ $ \= $drop$ \` $(1)$ \\
\>\> $black\wedge \neg white \leftarrow drop$ \` $(2)$ \\[.05in]
\>      $\Pi_2:\ $ \> $drop$ \` $(1)$ \\
\>\> $\neg black\wedge white \leftarrow drop$ \` $(2)$ \\[.05in]
\>      $\Pi_3:\ $ \> $drop$ \` $(1)$ \\
\>\> $black\wedge white \leftarrow drop$ \` $(2)$ 
\end{tabbing}

Under the rational answer set semantics,
$\Pi_1$ has one answer set $I_1=\{drop, black\}$,
$\Pi_2$ has one answer set $I_2=\{drop, white\}$,
and $\Pi_3$ has one answer set $I_3=\{drop, black, white\}$.
Therefore, $\Pi$ has in total three answer sets:
$I_1$, $I_2$ and $I_3$, which correspond to
the three solutions to the pin-touching-colored board problem.
}
\end{example}

As said, an epistemic-free program $\Pi$ extended with 
choice constructs represents
a collection of choice programs each of which is an epistemic-free program. 
Due to this, 
the collection of answer sets of $\Pi$,
which is the union of answer sets of all choice programs of $\Pi$,
may not be subset-minimal,
although the answer sets
of every choice program of $\Pi$ are subset-minimal.
This explains why rule (2) of $\Pi$
in Example~\ref{pin-dropping} cannot be replaced by
the following disjunctive rule:
\[black\wedge \neg white \mid\neg black\wedge white \mid black\wedge white \leftarrow drop.\]
Although both rule heads provide the same three choices,
the epistemic-free program $\Pi'$ obtained from $\Pi$
with rule (2) replaced by the above disjunctive rule
has only two answer sets, viz.  $I_1=\{drop, black\}$
and $I_2=\{drop, white\}$.
$I_3=\{drop, black, white\}$,
which is a candidate answer set 
of $\Pi'$ generated from the third choice
$black\wedge white$,
is not an answer set of $\Pi'$
because it is not subset-minimal.

Although disjunctive rule heads and
choice constructs are both used to provide a set of choices
for answer set construction,
they are essentially different. In particular, 
for construction of an answer set $I$,
a choice construct $1 \{a_1,\cdots,a_m\} 1$ 
provides $m$ choices, requiring $I$ to satisfy
$a_1\wedge \neg a_2\wedge \cdots\wedge\neg a_m$,
or alternatively $\neg a_1\wedge a_2\wedge\neg a_3\wedge \cdots\wedge\neg a_m$,
$\cdots$, or alternatively
$\neg a_1\wedge \cdots\wedge\neg a_{m-1}\wedge a_m$.
In contrast, a disjunctive rule head 
$a_1\mid \cdots \mid a_m$
provides $m$ choices,  requiring $I$ to satisfy
$a_1$, or alternatively $a_2$, $\cdots$, or alternatively $a_m$.
As a result, the program $\Pi=\{1\{a, b\}1,\ a\leftarrow b,\ b\leftarrow a\}$
has no answer set, where the choice construct
$1\{a, b\}1$ presents for every answer set $I$ two choices, requiring $I$
to satisfy either $a\wedge \neg b$ or alternatively $\neg a\wedge b$. In contrast,
the program $\Pi=\{a\mid b,\ a\leftarrow b,\ b\leftarrow a\}$
has two answer sets, viz. $I_1=\{a\}$ and $I_2=\{b\}$,
where the disjunctive rule head $a\mid b$
presents for every answer set $I$ two choices, requiring $I$ to satisfy
either $a$ or alternatively $b$.

\section{Using Refined GAS Principles As an Alternative Baseline to Assess Existing ASP Semantics}
\label{sec-existing-semantics}

In this section, we use the refined GAS principles
as an alternative baseline to assess the existing ASP semantics
by showing whether they satisfy or embody these principles.
Formally we have the following definition:
\begin{enumerate}
\item
For an answer set semantics ${\cal X}$,
we say that ${\cal X}$ {\em satisfies} 
the refined GAS principles
if every answer set under ${\cal X}$ is a model
satisfying the principles of well-supportedness
and minimality w.r.t.\ negation by default;
and that ${\cal X}$ {\em embodies} the refined 
GAS principles if the converse also holds,
i.e., a model is an answer set under ${\cal X}$ 
iff it satisfies these two principles.
\item
For a world view semantics ${\cal X}$,
we say that ${\cal X}$ {\em satisfies} 
the refined GAS principles
if every world view under ${\cal X}$ is an epistemic model
satisfying well-supportedness
and minimality w.r.t.\ negation by default at the
answer set level and minimality w.r.t.\ epistemic negation 
at the world view level;
and that ${\cal X}$ {\em embodies} the refined 
GAS principles if the converse also holds,
i.e., an epistemic model is a world view under ${\cal X}$ 
iff it satisfies well-supportedness
and minimality w.r.t.\ negation by default at the
answer set level and minimality w.r.t.\ epistemic negation 
at the world view level.
\end{enumerate}

As both the rational answer set semantics
and the rational world view semantics 
embody the refined GAS principles, 
to show that an existing answer set/world 
view semantics ${\cal X}$ satisfies the refined 
GAS principles
it suffices to show that every answer set/world view
under ${\cal X}$ is also an answer set/world view
under the rational answer set/world view semantics;
and to show that ${\cal X}$ embodies the refined 
GAS principles, it suffices to show that 
${\cal X}$ agrees with 
the rational answer set/world view semantics, i.e.,
it defines the same answer sets/world views
as the rational answer set/world view semantics.

\subsection{Answer set semantics that embody the refined GAS principles}

Major existing ASP semantics that embody 
the refined GAS principles
include the GL$_{nlp}$-semantics for simple normal programs
and the WJ-semantics for epistemic-free normal programs.

\subsubsection{GL$_{nlp}$-semantics for simple normal programs}

For a simple normal program $\Pi$,
the GL$_{nlp}$-semantics \cite{GL88} 
defines an interpretation $I$ to be an answer set of 
$\Pi$ if $I$ is the least model of the GL-reduct $\Pi^I$.

It was shown by \citeA{Fages:JMLCS:1994} 
that a model of a simple normal program $\Pi$
is an answer set under the GL$_{nlp}$-semantics 
iff it is well-supported in $\Pi$ by Definition~\ref{def-well-supported-simp-normal}.
Then it follows immediately from Theorem~\ref{th-simp-normal} and
Corollaries \ref{cor-normal-atomhead-formula} 
and \ref{cor-disj-normal-wellsupp} that a model of a 
simple normal program $\Pi$
is an answer set under the GL$_{nlp}$-semantics 
iff it is well-supported in $\Pi$ by Definition~\ref{def-well-supported-disj}.

Moreover, as shown by \citeA{GL88}, every answer set $I$
of a simple normal program $\Pi$
under the GL$_{nlp}$-semantics is a minimal model of $\Pi$.
This means that no other model of $\Pi$ that is
a proper subset of $I$ is well-supported in $\Pi$.

As a result, for a simple normal program $\Pi$,
a model $I$ of $\Pi$ is an answer set under the GL$_{nlp}$-semantics
iff $I$ is well-supported in $\Pi$ by Definition~\ref{def-well-supported-disj}
and no other model of $\Pi$ that is
a proper subset of $I$ is well-supported in $\Pi$
iff $I$ is a minimal well-supported model of $\Pi$.
The following corollary is now immediate from 
Definition~\ref{def-ideal-ans-semantics}.

\begin{corollary}
\label{cor-rational-simple-normal-pro}
For simple normal programs, 
the GL$_{nlp}$-semantics agrees with 
the rational answer set semantics.
\end{corollary}

Therefore, 
for simple normal programs the GL$_{nlp}$-semantics
embodies the refined GAS principles.

\subsubsection{WJ-semantics for epistemic-free normal programs}

For an epistemic-free normal program $\Pi$,
 \citeA{ShenWEFRKD14} defined
the {\em well-justified FLP answer set semantics}
(simply WJ-semantics) in terms of
the least fixpoint $\mathit{lfp}(T_{\Pi}(\emptyset,\neg I^-))$
of the inference sequence 
$\langle T_{\Pi}^i(\emptyset,\neg I^-)\rangle_{i=0}^\infty$ 
w.r.t.\ an interpretation $I$.

\begin{definition}[WJ-semantics \cite{ShenWEFRKD14}]
\label{def-WJsemantics}
A model $I$ of an epistemic-free normal 
program $\Pi$ is an answer set of $\Pi$ if
$\mathit{lfp}(T_{\Pi}(\emptyset,\neg I^-))\cup \neg I^- \models p$
for every $p\in I$.
\end{definition}

Intuitively, the WJ-semantics says that 
a model $I$ of an epistemic-free normal program $\Pi$
is an answer set 
iff when $\neg I^-$ is assumed to be true, 
all $p\in I$ can be inferred 
by iteratively applying the if-then rules in $\Pi$ 
via a bottom-up fixpoint inference sequence 
$\langle T_{\Pi}^i(\emptyset,\neg I^-)\rangle_{i=0}^\infty$.
The WJ-semantics agrees with the GL$_{nlp}$-semantics
for simple normal programs. 

It follows immediately from Theorem~\ref{th-wellsupport-lfp-normal} 
and Corollary~\ref{cor-disj-normal-wellsupp} 
that a model $I$ of  
an epistemic-free normal program $\Pi$
is an answer set under the WJ-semantics 
iff it is well-supported in $\Pi$ by Definition~\ref{def-well-supported-disj}.

Moreover, it was shown in \cite{ShenWEFRKD14}
that every answer set $I$
of an epistemic-free normal program $\Pi$
under the WJ-semantics is a minimal model of $\Pi$.
This means that no other model of $\Pi$ that is
a proper subset of $I$ is well-supported in $\Pi$.

As a result, for every epistemic-free normal program $\Pi$,
a model $I$ of $\Pi$ is an answer set under the WJ-semantics
iff $I$ is well-supported in $\Pi$ by Definition~\ref{def-well-supported-disj}
and no other model of $\Pi$ that is
a proper subset of $I$ is well-supported in $\Pi$.
The following conclusion then follows immediately from 
Definition~\ref{def-ideal-ans-semantics}.

\begin{corollary}
\label{cor-rational-epistemic-free-normal-pro}
For epistemic-free normal programs,
the WJ-semantics agrees with 
the rational answer set semantics.
\end{corollary}

Therefore, the WJ-semantics
for epistemic-free normal programs 
embodies the refined GAS principles.

\subsection{Answer set semantics that do not embody but satisfy the refined GAS principles}

Major existing ASP semantics that satisfy but do 
not embody the refined GAS principles
include the GL-semantics for simple disjunctive programs
and the three-valued fixpoint semantics for epistemic-free normal 
programs with atomic rule heads.

\subsubsection{GL-semantics for simple disjunctive programs}
\label{subsec-DI-GL-simple-semantics}

For a simple disjunctive program $\Pi$,
the GL-semantics \cite{GL91} 
defines an interpretation $I$ to be an answer set of 
$\Pi$ if $I$ is a minimal model of the GL-reduct $\Pi^I$.
It was observed in \cite{ShenE19} that in this definition, 
the disjunctive rule head operator $\mid$
amounts to the classical connective $\vee$
in that a model $I$ of $\Pi$
is an answer set of $\Pi$ under the GL-semantics
iff $I$ is a minimal model of $\Pi^I$
iff $I$ is a minimal model of $(\Pi^\vee)^I$,
where $\Pi^\vee$ is $\Pi$ with the operator $\mid$
replaced by the connective $\vee$.
It is due to identifying the operator $\mid$
with the connective $\vee$ in disjunctive rule heads
that the GL-semantics differs from 
the rational answer set semantics.

In Example~\ref{ex:rationalanswer-minwellsupp},
we showed that the expected answer set
$I=\{g_1, g_2, c_1, c_2\}$ in 
Example~\ref{ex:generalized-strat-comp}
is a minimal well-supported model  and thus
is an answer set under the rational answer set semantics.
However, as shown in Example~\ref{ex:generalized-strat-comp},
this answer set is not an answer set under the GL-semantics.
This means that
the GL-semantics does not agree with 
the rational answer set semantics, and that
it does not embody the refined GAS principles
because not all models
satisfying both well-supportedness
and minimality w.r.t.\ negation by default
(i.e., minimal well-supported models)
are answer sets under the GL-semantics.

However, the GL-semantics satisfies
the refined GAS principles.

\begin{restatable}{theorem}{RthDIGLnlprationalSemantics}
\label{th-DI-GLnlprational-semantics}
Answer sets of a simple disjunctive program $\Pi$ 
under the GL-semantics
are also answer sets under 
the rational answer set semantics.
\end{restatable}

\subsubsection{Three-valued fixpoint semantics}
\label{three-valued-sem}

For an epistemic-free normal program $\Pi$ with atomic rule heads,
in \cite{DPB01,PDB07} a three-valued fixpoint semantics
was introduced based on a three-valued fixpoint operator $\Phi_\Pi$.
Answer sets under this semantics are called {\em two-valued stable models}.

A three-valued 
interpretation of $\Pi$ is $\hat{I} = (I_1,I_2)$, 
where $I_1\subseteq I_2\subseteq \Sigma$. 
Intuitively, atoms in $I_1$ are assigned
the truth value $\mathbf{t}$, atoms in $I_2\setminus I_1$  
assigned $\mathbf{u}$, and atoms in $\Sigma\setminus I_2$ 
assigned $\mathbf{f}$. These truth values are 
ordered by the {\em truth order} $\le_t$ with
$\mathbf{f}\le_t \mathbf{u}\le_t \mathbf{t}$. 
Negation on these truth values is defined
as $\neg \mathbf{f}=\mathbf{t}, \neg \mathbf{u}=\mathbf{u}$ 
and $\neg \mathbf{t} =\mathbf{f}$.
The truth value of a propositional
formula $F$ under $\hat{I}$, denoted $\hat{I}(F)$,
is defined recursively as follows:
\[ \hat{I}(F) = \left\{ 
   \begin{array}{l l}
      \mathbf{t} \text{ (resp. $\mathbf{u}$ and $\mathbf{f}$)} & \quad \text{if $F$ is in $I_1$ (resp. $I_2\setminus I_1$ and $\Sigma\setminus I_2$)}\\
     \text{min$_{\le_t} \{\hat{I}(F_1), \hat{I}(F_2)\}$} & \quad \text{if $F = F_1\wedge F_2$}\\
     \text{max$_{\le_t} \{\hat{I}(F_1), \hat{I}(F_2)\}$} & \quad \text{if $F = F_1\vee F_2$}\\
     \neg \hat{I}(F_1) & \quad \text{if $F = \neg F_1$}
   \end{array} \right.\]
Note that $F_1 \supset F_2$ is an abbreviation for $\neg F_1 \vee F_2$.
Then $\hat{I}$ {\em satisfies} $F$ if  $\hat{I}(F) =  \mathbf{t}$.

Given a three-valued interpretation $\hat{I} = (I_1,I_2)$,
the three-valued operator $\Phi_\Pi(I_1,I_2) = (I_1^\prime,I_2^\prime)$
is defined such that
\begin{tabbing} 
$\qquad\quad I_1^\prime = \{head(r)\mid$ \text{$r\in \Pi$ and $\hat{I}(body(r)) = \mathbf{t}\}$}, and\\
$\qquad\quad I_2^\prime = \{head(r)\mid$ \text{$r\in \Pi$ and 
$\hat{I}(body(r))=\mathbf{t}$ or $\hat{I}(body(r))=\mathbf{u}\}$.}
\end{tabbing}

Let $\Phi_\Pi^{1}(I_1,I_2)$ denote the first element of $\Phi_\Pi(I_1,I_2)$, i.e. $I_1^\prime$, and
$\Phi_\Pi^{2}(I_1,I_2)$ denote the second element $I_2^\prime$.
When $I_2$ is fixed, we compute a sequence $x_0 = \emptyset,
x_1 = \Phi_\Pi^{1}(x_0,I_2),\cdots,x_{i+1} = \Phi_\Pi^{1} (x_i,I_2),\cdots,$ 
until a fixpoint, denoted 
$St_\Phi^\downarrow(I_2)$, is reached. Similarly, when $I_1$ is fixed, 
we compute a sequence $x_0 = I_1,
x_1 = \Phi_\Pi^{2}(I_1,x_0),\cdots,x_{i+1} = \Phi_\Pi^{2} (I_1,x_i),\cdots,$ 
until a fixpoint 
$St_\Phi^\uparrow(I_1)$ is reached. 
The {\em stable revision operator} $St_\Phi$ on 
$\hat{I} = (I_1,I_2)$ is defined as 
\begin{tabbing}
$\qquad\quad St_\Phi(I_1,I_2) = (St_\Phi^\downarrow(I_2),St_\Phi^\uparrow(I_1)).$
\end{tabbing}

By iteratively applying $St_\Phi$ such that 
$St_\Phi^0(I_1,I_2) = (I_1,I_2)$ and for $i>0$,
$St_\Phi^i(I_1,I_2) = St_\Phi St_\Phi^{i-1}(I_1,I_2),$
we obtain a fixpoint, denoted $\mathit{lfp}(St_\Phi(I_1,I_2))$.
The three-valued fixpoint semantics 
of \cite{DPB01,PDB07} 
is then defined in terms of the fixpoint. That is,
let $\Pi$ be an epistemic-free normal program 
whose rule heads are atoms and $I$ be
a two-valued model of $\Pi$,
as defined in Section~\ref{Preliminaries};
then $I$ is a {\em two-valued stable model} of $\Pi$
if $\mathit{lfp}(St_\Phi(I,I)) = (I,I)$. 

The following example shows that some
minimal well-supported models of 
an epistemic-free program are not answer sets
under the three-valued fixpoint semantics.
That is, the three-valued fixpoint semantics 
does not agree with the 
rational answer set semantics and thus
does not embody the refined GAS principles.

\begin{example}
\label{eg-three-valued fixpoint}
{\em
Consider the following logic program with aggregates:
\begin{tabbing} 
\hspace{.05in} \= $\Pi_1:\ p(2)\leftarrow$ SUM$\langle (\{0,1,2\},X):p(X)\rangle <2\ \vee$ SUM$\langle (\{0,1,2\},X):p(X)\rangle \geq 2$ 
\end{tabbing}
Note that for whatever interpretation $I$,
the body condition of the rule in $\Pi_1$ is always satisfied by $I$, i.e., the rule body amounts to a tautology in classical logic.
Therefore, $\Pi_1$ should have a single expected 
answer set  $I=\{p(2)\}$.
As the aggregate SUM$\langle (\{0,1,2\},X):p(X)\rangle <2$ 
holds iff $\neg p(2)$ is true and 
SUM$\langle (\{0,1,2\},X):p(X)\rangle \geq 2$ 
holds iff $p(2)$ is true,
this program can be converted to the following
epistemic-free program by replacing the two aggregates
with $\neg p(2)$ and $p(2)$, respectively:
\begin{tabbing} 
\hspace{.05in} \= $\Pi_2:\ p(2)\leftarrow \neg p(2) \vee p(2)$ 
\end{tabbing}

$\Pi_1$ and $\Pi_2$ should have
the same answer sets because their rule bodies
express the same condition.
One can check that $I=\{p(2)\}$ 
is a minimal well-supported model
of $\Pi_2$ and thus is an answer 
set of $\Pi_2$ under 
the rational answer set semantics.
$I$ is also an answer 
set of $\Pi_1$ and $\Pi_2$
under the FLP-semantics.
However, $I$ is not a two-valued
stable model under the three-valued fixpoint semantics
because $\mathit{lfp}(St_\Phi(\{p(2)\},\{p(2)\})) = (\emptyset,\{p(2)\})$.
}
\end{example}

It was shown in \cite{ShenWEFRKD14}
that for an epistemic-free normal program 
$\Pi$ with atomic rule heads,
an answer set $I$ of $\Pi$
under the three-valued fixpoint semantics
is also an answer set 
under the WJ-semantics.
The following conclusion follows immediately
from Corollary \ref{cor-rational-epistemic-free-normal-pro}.

\begin{corollary}
\label{corr-three-valued fixpoint-rational}
Answer sets 
of an epistemic-free normal program 
$\Pi$ with atomic rule heads 
under the three-valued fixpoint semantics
are also answer sets under 
the rational answer set semantics.
\end{corollary}

Therefore, the three-valued fixpoint semantics
satisfies the refined GAS principles.

\subsection{Answer set semantics that neither embody nor satisfy the refined GAS principles}
\label{subsec-other-semantics}
Other representative answer set semantics include
the DI--semantics \cite{ShenE19},
the FLP-semantics \cite{FaberPL11}, and 
the equilibrium logic-based semantics 
\cite{Ferr05,Pearce06,FerrarisLL11}.
We show
that these semantics do not satisfy the refined GAS principles.

\subsubsection{DI-semantics}
\label{subsec-DI-GL-semantics}
 
To address the issue
of the GL-semantics with disjunctive rule heads,
\citeA{ShenE19} presented 
an alternative answer set semantics, called {\em determining
inference semantics} (DI-semantics for short). 
It is based on a disjunctive
rule head selection function $sel(head(r), I)$
w.r.t.\ an interpretation $I$, 
which interprets the disjunctive rule head operator $\mid$ 
as a nondeterministic inference operator
and returns one alternative $H_i$ 
from every rule $H_1\mid\cdots\mid H_m\leftarrow  B$ 
in $\Pi$ once $B$ is satisfied by $I$.
This would produce different 
candidate answer sets for $\Pi$ 
(corresponding to different
choices of the alternatives in rule heads). 
The DI-semantics admits as answer sets 
only those candidate answer sets
that are subset-minimal among all candidate answer sets.

As mentioned in Section~\ref{ws-disrh},
the head selection function $sel(head(r), I)$
defined in \cite{ShenE19} differs from the function $sel((n, head(r)), I)$ 
in Definition~\ref{head-sel-gdlp} in that $sel(head(r), I)$ 
additionally requires selecting the same rule head formula
from identical or variant rule heads.
In order to distinguish between the two versions of $sel$,
we may denote the one in \cite{ShenE19} by $sel_v$.

\begin{definition}[Variant rule heads \cite{ShenE19}]
\label{equiv-heads}
Two disjunctive rule heads ${\cal H}_1 = E_1 \mid \cdots \mid E_m$
and ${\cal H}_2 = F_1 \mid \cdots \mid F_n$
are {\em variant rule heads} if
for all $i$ and $j$ $(1\leq i\leq m, 1\leq j\leq n)$,
${\cal H}_1$ has a  head formula $E_i$
if and only if ${\cal H}_2$ has a head formula $F_j$
with $E_i\equiv F_j$.
\end{definition}

\begin{definition}[Head selection function $sel_v(head(r), I)$ \cite{ShenE19}]
\label{head-select-SE19}
Let $\Pi$ be an epistemic-free program and $I$ 
be an interpretation.
Then for every rule $r\in \Pi$, 
\[sel_v(head(r), I) = \left \{\begin{array}{ll}
                                 H_i & \mbox{$head(r)$ has some head formula $H_i$ that is satisfied by $I$,} \\
                                  \bot & \mbox{otherwise}
                                 \end{array} \right. \]
where for any variant rule heads 
${\cal H}_1$ and ${\cal H}_2$ in $\Pi$, 
$sel_v({\cal H}_1, I) \equiv sel_v({\cal H}_2, I)$.
\end{definition}

The DI-semantics defines candidate answer sets in terms of
disjunctive program reducts $\Pi^I_{sel_v}$ 
as defined in Definition~\ref{def:sel-reduct}
except that the head selection function $sel$
is replaced by $sel_v$,
i.e., it transforms $\Pi$ to an epistemic-free
normal program $\Pi^I_{sel_v}$ 
and defines $I$ to be a candidate answer set of $\Pi$
if $I$ is an answer set of $\Pi^I_{sel_v}$.

\begin{definition}[DI-semantics \cite{ShenE19}]
\label{generalSemanticsforDLP}
Let $\Pi$ be an epistemic-free program 
and $\cal X$ be an answer set semantics
for epistemic-free normal programs.
Then a model $I$ is an answer set of $\Pi$ 
w.r.t.\ the base semantics $\cal X$ if
(1) for some head selection function $\sel_v$,
$I$ is an answer set of $\Pi^I_{sel_v}$ under $\cal X$, 
and (2) no other model $J$ that is a proper subset of $I$ 
satisfies condition (1).
\end{definition}

The DI-semantics is generic and applicable to extend
any answer set semantics $\cal X$ for 
epistemic-free normal programs to epistemic-free
programs with disjunctive rule heads.  
In particular, by instantiating $\cal X$ 
to the GL$_{nlp}$-semantics 
we induce a DI-semantics for simple disjunctive
programs, denoted {\em DI-GL$_{nlp}$-semantics},
and by instantiating $\cal X$ 
to the WJ-semantics 
we induce a DI-semantics for
epistemic-free programs,
denoted {\em DI-WJ-semantics}. 
Note that the DI-WJ-semantics reduces to
the DI-GL$_{nlp}$-semantics
for simple disjunctive programs.

The DI-GL$_{nlp}$-semantics is a 
relaxation of the GL-semantics.
It was shown in \cite{ShenE19} that
for a simple disjunctive program $\Pi$,
the GL-semantics can be adapted from 
Definition~\ref{generalSemanticsforDLP}
simply by choosing the GL$_{nlp}$-semantics
as the base semantics $\cal X$ and 
requiring that for {\em every} (instead of {\em some}) 
head selection function $\sel_v$,
$I$ is an answer set of $\Pi^I_{sel_v}$ 
under $\cal X$.
This characterizes the essential difference of the 
DI-GL$_{nlp}$-semantics from the GL-semantics and shows that
the requirement of the GL-semantics
is much stronger than the DI-GL$_{nlp}$-semantics.
As a result, an answer set under the GL-semantics
is also an answer set under the DI-GL$_{nlp}$-semantics,
but not conversely.

Our recent study shows that 
due to that ${sel_v}$ requires selecting the same rule head formula
from identical or variant rule heads,
some expected answer sets may be excluded and some
unexpected ones may be obtained under the DI-semantics.
As a result,
some answer sets under the rational answer set semantics
may not be answer sets under 
the DI-GL$_{nlp}$-semantics/DI-WJ-semantics, and vice versa. 
Therefore, the DI-semantics in general, and
the DI-GL$_{nlp}$-semantics/DI-WJ-semantics 
in particular, neither embody nor
satisfy the refined GAS principles.
Here is an example.

\begin{example}
\label{eg-sel-selv}
{\em
Consider the following simple disjunctive 
program:\footnote{This program encodes an instance of the GSC problem,
where we have three companies $C = \{c_1, c_2, c_3\}$,
two goods $G = \{ g_1, g_2 \}$ and four strategy conditions:
$\sigma^{2}_1 = c_1$, $\sigma^{3}_1 = c_1$, $\sigma^{2}_2 = \neg c_2$
and $\sigma^{3}_2 = \neg c_3$. Suppose that 
both $g_1$ and $g_2$ are produced by 
$c_1$, $c_2$ or $c_3$. This GSC instance has only one strategic set
$C'=\{c_2,c_3\}$.
}
\begin{tabbing} 
\hspace{.3in} \= $\Pi:\ $ \=  $g_1$\`$(1)$\\[2pt]
\>\> $g_2$ \` $(2)$ \\[2pt]
\>\> $c_1 \mid c_2 \mid c_3\leftarrow g_1$   \` $(3)$\\[2pt]
\>\> $c_1 \mid c_2\mid c_3 \leftarrow g_2$ \` $(4)$ \\[2pt]
\> \> $c_2 \leftarrow c_1$  \` $(5)$\\[2pt]
\> \> $c_3 \leftarrow c_1$   \` $(6)$\\[2pt]
\> \> $c_2 \leftarrow \neg c_2$ $\qquad$ \= \text{//~This rule amounts to a constraint
$\bot \leftarrow \neg c_2$}   \` $(7)$\\[2pt]
\> \> $c_3 \leftarrow \neg c_3$ \> \text{//~This rule amounts to a constraint
$\bot \leftarrow \neg c_3$}  \` $(8)$
\end{tabbing} 

Assume that $I = \{g_1, g_2, c_2, c_3\}$, which is a minimal model of $\Pi$,
is an expected answer set.
Intuitively, rule $(3)$ presents two choices, either $c_2$ or alternatively $c_3$,
for construction of $I$ (as $c_1\not\in I$, $c_1$ is not a choice),
and rule $(4)$ presents the same two choices independently.
Let us choose $c_2$ and $c_3$ from
the heads of rules (3) and (4), respectively.
Then, we infer $I=\{g_1, g_2, c_2, c_3\}$
from the rules (1)-(6).
As $I$ satisfies the constraint rules (7) and (8),
it is a candidate answer set and
is expected to be an answer set
provided that it is subset-minimal.
As $I$ is a minimal model of $\Pi$,
it is a minimal candidate answer set
and thus is an expected answer set of $\Pi$.

For the rational answer set semantics,
consider a selection function $sel$ w.r.t.\ $I$,
where $sel((3, c_1 \mid c_2 \mid c_3), I) = c_2$
and $sel((4, c_1 \mid c_2 \mid c_3), I) = c_3$.
Then the disjunctive program reduct of
$\Pi$ w.r.t.\ $I$ and $sel$ is 
\begin{tabbing} 
\hspace{.3in} $\Pi^{I}_{sel}:\ $ \=  $g_1$\`$(1)$\\[2pt]
\> $g_2$ \` $(2)$ \\[2pt]
\> $c_2 \leftarrow g_1$ $\quad$  \` $(3)$\\[2pt]
\> $c_3 \leftarrow g_2$  \` $(4)$ 
\end{tabbing}
We have $\mathit{lfp}(T_{\Pi^{I}_{sel}}(\emptyset,\neg I^-))=\{g_1, g_2, c_2, c_3\}=I$,
so by Theorem~\ref{th-wellsupport-lfp}
$I$ is well-supported in $\Pi^{I}_{sel}$; and
by Definition~\ref{def-well-supported-disj},
$I$ is well-supported in $\Pi$.
As $I$ is a minimal model of $\Pi$,
it is a minimal well-supported model
and thus is an answer set under the rational 
answer set semantics.

For the DI-GL$_{nlp}$-semantics/DI-WJ-semantics, 
however, there is no head selection function 
${sel_v}$ w.r.t.\ $I$ which 
can choose $c_2$ and $c_3$ from
the heads of rules (3) and (4), respectively,
as they are variant rule heads.
Due to this, there is no head selection function
$\sel_v$  w.r.t.\ $I$ such that
$I$ is an answer set of $\Pi^I_{sel_v}$ 
under the GL$_{nlp}$-semantics/WJ-semantics.
Therefore, $I$ is not an answer set under 
the DI-GL$_{nlp}$-semantics/DI-WJ-semantics.

This expected answer set $I$ satisfies the principles 
of well-supportedness and minimality
w.r.t.\ negation by default, but it is 
not an answer set under 
the DI-GL$_{nlp}$-semantics/DI-WJ-semantics.
This means that the DI-GL$_{nlp}$-semantics/DI-WJ-semantics 
does not embody the refined GAS principles.

Next, consider another model $I_1=\{g_1, g_2, c_1, c_2, c_3\}$.
Consider a head selection function $sel$ w.r.t.\ $I_1$,
where $sel((3, c_1 \mid c_2 \mid c_3), I_1) = c_1$
and $sel((4, c_1 \mid c_2 \mid c_3), I_1) = c_1$.
One can easily check that $I_1$ is well-supported in $\Pi$.
As $I=\{g_1, g_2, c_2, c_3\}$ is also
well-supported in $\Pi$,
$I_1$ is not minimal among all well-supported models of $\Pi$.
Therefore,  $I_1$ is not an answer set
under the rational answer set semantics. 

For the DI-GL$_{nlp}$-semantics/DI-WJ-semantics,
consider a head selection function 
${sel_v}$ w.r.t.\ $I_1$ which 
chooses $c_1$ both from
the head of rule (3) and rule (4).
That is, let $sel_v(c_1 \mid c_2 \mid c_3, I_1) = c_1$.
Then, the disjunctive program reduct of
$\Pi$ w.r.t.\ $I_1$ and $sel_v$ is 
\begin{tabbing} 
\hspace{.3in} $\Pi^{I_1}_{sel_v}:\ $ \=  $g_1$\`$(1)$\\[2pt]
\> $g_2$ \` $(2)$ \\[2pt]
\> $c_1 \leftarrow g_1$  \` $(3)$\\[2pt]
\> $c_1 \leftarrow g_2$  \` $(4)$ \\[2pt]
\> $c_2 \leftarrow c_1$    \` $(5)$\\[2pt]
\> $c_3 \leftarrow c_1$   \` $(6)$ 
\end{tabbing}
$I_1$ is an answer set of $\Pi^{I_1}_{sel_v}$ 
under the GL$_{nlp}$-semantics/WJ-semantics.
Moreover, note that among all proper subsets of $I_1$,
only $I=\{g_1, g_2, c_2, c_3\}$ is a model of $\Pi$.
As said above, there is no head selection function
$\sel_v$  w.r.t.\ $I$ such that
$I$ is an answer set of $\Pi^I_{sel_v}$ 
under the GL$_{nlp}$-semantics/WJ-semantics.
Therefore, by Definition~\ref{generalSemanticsforDLP}
$I_1$ is an answer set under 
the DI-GL$_{nlp}$-semantics/DI-WJ-semantics.

$I_1$ satisfies 
well-supportedness, but it is not minimal among 
all well-supported models of $\Pi$.
This shows that the DI-GL$_{nlp}$-semantics/DI-WJ-semantics 
violates minimality w.r.t.\ negation by default; 
thus it does not satisfy the refined GAS principles.
}
\end{example}
  
\subsubsection{FLP-semantics}
The FLP-semantics 
\cite{FLP04,FaberPL11} was originally oriented towards 
giving an answer set semantics to simple disjunctive 
programs extended with aggregates, and was later
adapted to other extensions of logic programs, including
description logic programs 
\cite{EiterIST05,Eiter08,Lukasiewicz10},
modular logic programs \cite{Dao-TranEFK09}, and logic programs with
first-order formulas \cite{BartholomewLM11}.  
Let $\Pi$ be an epistemic-free program 
and $I$ an interpretation;
then $I$ is an answer set
of $\Pi$ under the FLP-semantics 
if $I$ is a minimal model of 
the {\em FLP-reduct} of $\Pi$ w.r.t.\ $I$ given by
\[f\Pi^I = \{r\in \Pi \mid I \textrm{
satisfies } body(r)\}\]  

For simple disjunctive programs, 
the FLP-semantics
coincides with the GL-semantics \cite{FLP04}.
As the GL-semantics does not agree with
the rational answer set semantics and 
does not embody the refined GAS principles,
the FLP-semantics does not embody the refined GAS principles either. 

Next, we use an example 
with aggregates to show that some
answer sets under the FLP-semantics
are not well-supported. 
Therefore, the FLP-semantics does not satisfy 
the refined GAS principles.

\begin{example}
\label{eg-FLP}
{\em
The following logic program with aggregates is
borrowed from \cite{BartholomewLM11}: 
\begin{tabbing} 
\hspace{.05in} \= $\Pi_1:\ $ \= $p(2)\leftarrow$ $\neg$SUM$\langle (\{-1,1,2\},X):p(X)\rangle <2$ \` $(1)$\\
\> \>         $p(-1)\leftarrow$ SUM$\langle (\{-1,1,2\},X):p(X)\rangle \geq 0$  \` $(2)$\\
\>\>          $p(1)\leftarrow p(-1)$  \` $(3)$
\end{tabbing}
SUM$\langle (\{-1,1,2\},X):p(X)\rangle$
is an aggregate function that
for every interpretation $I$ yields the sum $S$ of all $X\in \{-1, 1, 2\}$
such that $p(X)$ is true in $I$.
The aggregate SUM$\langle (\{-1,1,2\},X):p(X)\rangle <2$ 
is satisfied by $I$ if $S< 2$, and 
SUM$\langle (\{-1,1,2\},X):p(X)\rangle \geq 0$
is satisfied by $I$ if $S\geq 0$.

Note that the aggregates 
SUM$\langle (\{-1,1,2\},X):p(X)\rangle <2$ and
SUM$\langle (\{-1,1,2\},X):p(X)\rangle \geq 0$
express the same conditions as the first-order logic formulas
$\neg p(2)\vee (p(-1)\wedge \neg p(1))$ 
and $\neg p(-1)\vee p(1)\vee p(2)$, respectively, 
i.e., an interpretation $I$ satisfies the aggregates iff
$I$ satisfies the respective formulas.
So by replacing the two aggregates 
with their respective formulas,
we can convert $\Pi_1$ to the following
epistemic-free program:
\begin{tabbing} 
\hspace{.05in} \= $\Pi_2:\ $ \= $p(2)\leftarrow \neg (\neg p(2)\vee (p(-1)\wedge \neg p(1)))$ \` $(1)$\\
\> \>         $p(-1)\leftarrow \neg p(-1)\vee p(1)\vee p(2)$  \` $(2)$\\
\>\>          $p(1)\leftarrow p(-1)$  \` $(3)$
\end{tabbing}

$\Pi_1$ and $\Pi_2$ should have
the same answer sets because their rule bodies
express the same conditions.
One can check that $\Pi_1$ and $\Pi_2$ have an answer set $I=\{p(-1),p(1)\}$ 
under the FLP-semantics, but
$I$ is not an answer set
under the rational answer set semantics
because $I$ is not well-supported in $\Pi_2$
by Definition~\ref{def-well-supported-normal-atomhead}. 
Observe that $I$ has 
a circular justification in $\Pi_2$ caused by the  
following self-supporting loop:
\[p(1) \Leftarrow p(-1) \Leftarrow \neg p(-1)\vee p(1)\vee p(2) \Leftarrow p(1).\]
That is, $p(1)$ being in $I$ is due to 
$p(-1)$ being in $I$ (via rule (3)),
while $p(-1)$ being in $I$
is due to $I$ satisfying the condition
$\neg p(-1)\vee p(1)\vee p(2)$ (via rule (2)),
which in turn is due to $p(1)$ being in $I$.

The answer set $I=\{p(-1),p(1)\}$ 
under the FLP-semantics does not satisfy the 
principle of well-supportedness.
Therefore, the FLP-semantics
does not satisfy the refined GAS principles.
}
\end{example}

\subsubsection{Equilibrium logic-based semantics}
\label{subsubsec-Equilibrium logic-based semantics}

Another extensively studied answer set semantics
is by \citeA{Ferr05}, which defines answer sets
for logic programs with propositional formulas
and aggregates based on
a new definition of equilibrium logic \cite{Pearce96}.
This semantics was further extended to 
first-order formulas in terms of a modified circumscription \cite{FerrarisLL11}.
\citeA{Pearce06} proposed to identify 
answer sets with equilibrium models in equilibrium logic.
It turns out that the answer set semantics of \citeA{Pearce06}
coincides with that of \citeA{Ferr05} 
in the propositional case and 
with that of \citeA{FerrarisLL11} 
in the first-order case.
Therefore, we refer to all of 
them as {\em equilibrium logic-based semantics}.

For a propositional formula $F$, 
the {\em Ferraris-reduct} of $F$ w.r.t.\ 
an interpretation $I$, denoted $F^{\underline{I}}$,
is defined recursively as follows \cite{Ferr05}:\footnote{\citeA{Truszczynski-aij10} also 
presented a reduct that slightly differs
from the Ferraris-reduct in the way to handle 
$(G\supset H)^{\underline{I}}$, i.e.,
\begin{tabbing}
\hspace{.3in} \= $\bullet\ (G\supset H)^{\underline{I}} = \left\{ \begin{array}{ll}
                        G\supset H^{\underline{I}} & \mbox{if $I$ satisfies both $G$ and $H$}\\
                        \top & \mbox{if $I$ does not satisfy $G$}\\
                        \bot & \mbox{otherwise}
                   \end{array}  
    \right.  $
\end{tabbing}
}
\begin{tabbing}
\hspace{.3in} \= $\bullet\ \bot^{\underline{I}} = \bot$ \\[2pt]
\> $\bullet\ A^{\underline{I}} = \left\{ \begin{array}{ll}
                        A & \mbox{if $A$ is an atom and $I$ satisfies $A$}\\
                        \bot & \mbox{otherwise}
                   \end{array}  
    \right.  $\\
\> $\bullet\ (\neg F)^{\underline{I}} = \left\{ \begin{array}{ll}
                        \bot & \mbox{if $I$ satisfies $F$}\\
                        \top & \mbox{otherwise}
                   \end{array}  
    \right.  $\\[2pt]
    \> $\bullet\ (G\odot H)^{\underline{I}} = \left\{ \begin{array}{ll}
                        G^{\underline{I}}\odot H^{\underline{I}} & \mbox{if $I$ satisfies $G\odot H\ (\odot\in\{\wedge,\vee,\supset\})$}\\[2pt]
                        \bot & \mbox{otherwise}
                   \end{array}  
    \right.  $
\end{tabbing}
Let $\Pi$ be an epistemic-free program and $\Pi'$
be $\Pi$ with the disjunctive rule head 
operator $\mid$ and the if-then rule operator $\leftarrow$ 
replaced by the connectives $\vee$
and $\subset$, respectively.
Then $I$ is an answer set of $\Pi$ 
under the equilibrium logic-based semantics
if $I$ is a minimal model of the Ferraris-reduct 
$\Pi'^{\underline{I}}$ of $\Pi'$ w.r.t.\ $I$.

For simple disjunctive programs, 
the equilibrium logic-based semantics
coincides with the GL-semantics \cite{Ferr05,Pearce06}.
As the GL-semantics does not agree with
the rational answer set semantics and 
does not embody the refined GAS principles,
the equilibrium logic-based semantics 
does not embody the refined GAS principles either.

For epistemic-free programs in general,
the equilibrium logic-based semantics
differs significantly from 
the rational answer set semantics. 
The first main difference
is that it uses equilibrium logic \cite{Pearce96}
to define satisfaction of formulas in rule bodies and heads,
while the rational answer set semantics uses classical logic.
Specifically, unlike classical logic, equilibrium logic
does not satisfy
the law of the excluded middle nor allows for
double negation elimination. 

Consider the epistemic-free program
$\Pi_2 = \{p(2)\leftarrow \neg p(2) \vee p(2)\}$ 
in Example~\ref{eg-three-valued fixpoint}.
$I=\{p(2)\}$ is an answer 
set of $\Pi_2$ under 
the rational answer set semantics,
where the rule body $\neg p(2) \vee p(2)$
is viewed  as a tautology
and evaluated to $\top$ 
by the law of the excluded middle in classical logic.
In contrast, 
$I=\{p(2)\}$ is not an answer set of $\Pi_2$ 
under the equilibrium logic-based semantics, 
where $\neg p(2)$ in the rule body
is evaluated to $\bot$ w.r.t. $I$,
leading to a Ferraris-reduct
${\Pi_2'}^{\underline{I}}=\{p(2) \subset \bot\vee p(2)\}$
of $\Pi_2$ w.r.t.\ $I$, which has 
a smaller model $J=\emptyset$ than $I$.
This example shows that an answer set
under the rational answer set semantics
may not be an answer set under 
the equilibrium logic-based semantics.

Consider another epistemic-free program
$\Pi = \{p\leftarrow \neg\neg p,\ p\leftarrow \neg p\}$.
It has no answer set under 
the rational answer set semantics,
where the rule body $\neg\neg p$
is evaluated to $p$ by the law
of double negation elimination in classical logic.
In contrast, 
$I=\{p\}$ is an answer set of $\Pi$ 
under the equilibrium logic-based semantics, 
where $\neg \neg p$ in the rule body
is evaluated to $\top$ w.r.t. $I$, 
leading to a Ferraris-reduct
${\Pi'}^{\underline{I}}=\{p\subset \top,\ p\subset \bot\}$
of $\Pi$ w.r.t.\ $I$, which has a minimal model $I$.
This shows that an answer set
under the equilibrium logic-based semantics
may not be an answer set under 
the rational answer set semantics.
Moreover, it is easy to check that the answer set $I=\{p\}$  
under the equilibrium logic-based semantics
is not well-supported in $\Pi$.
Therefore, the equilibrium logic-based semantics
does not satisfy the refined GAS principles.

The second major difference is that they
interpret the disjunctive rule head operator $\mid$ differently.
For an interpretation $I$ and 
a disjunctive rule head
$H_1\mid\cdots\mid H_m$ satisfied by $I$,
the rational answer set semantics 
interprets $\mid$ as a nondeterministic 
inference operator 
that returns one alternative $H_i$ 
satisfied by $I$ from the rule head.
This is formally defined by a rule head selection
function $sel((n, H_1\mid\cdots\mid H_m), I)$ 
in Definition~\ref{head-sel-gdlp}.
In contrast,
the equilibrium logic-based semantics
replaces the operator $\mid$ with the connective $\vee$
and interprets the rule head
as $H_1^{\underline{I}} \vee \cdots \vee H_m^{\underline{I}}$,
where each $H_i^{\underline{I}}$ is
the Ferraris-reduct of $H_i$ w.r.t.\ $I$.
We use the following example
to further illustrate this difference.

\begin{example}
\label{eg-subsumption}
{\em
Consider the following epistemic-free program
(borrowed from \cite{ShenE19}). 
\begin{tabbing} 
\hspace{.3in} \= $\Pi:\ $ \= $a\supset b\mid b\supset a$ \` $(1)$\\
\> \>         $a$  \` $(2)$\\
\> \>         $b \leftarrow \neg b$  \` $(3)$
\end{tabbing}
Intuitively, rule $(1)$ expresses a class subsumption relation with
uncertainty, i.e., either  
$a$ is subsumed by $b$ 
or the other way around.
Rule (2) states that we already have $a$, and rule (3) is a
constraint stating that there is 
no answer set that
does not contain $b$.
Note that one cannot replace the 
disjunctive rule head operator 
$\mid$ with the logical connective $\vee$
in rule (1), as that would make the rule a tautology
$(a\supset b)\vee (b\supset a)$. 

Rule $(1)$ presents two alternatives 
for answer set construction, i.e.,
$a\supset b$ or $b\supset a$.
Suppose that we choose $a\supset b$; then
from rules (1) and (2) we infer $(a\supset b)\wedge a$,
which logically entails both $a$ and $b$ and thus yields
a potential answer set $I=\{a,b\}$. 
$I$ satisfies the constraint rule (3), 
so it is a candidate answer set for $\Pi$.
As $I$ is a minimal model of $\Pi$,
it is a minimal candidate answer set.
Note that this construction of an answer set
is based on the GAS principles,
where the rationality principle 
for knowledge minimization means 
that every answer set of $\Pi$
must be minimal among all
candidate answer sets of $\Pi$. 
Hence we intuitively expect $I$ to be 
an answer set of $\Pi$.

One can easily check that $I=\{a,b\}$
is a minimal well-supported model of $\Pi$;
thus it is an answer set under 
the rational answer set semantics.

However, $\Pi$ has no answer set 
under the equilibrium logic-based semantics.
For the interpretation $I=\{a,b\}$,
this semantics transforms $\Pi$
to a Ferraris-reduct 
\[\Pi'^{\underline{I}}=\{(a\supset b)\vee (b\supset a),\ a,\ \bot\supset b\}\] w.r.t.\ $I$,
where rule (1) is transformed to a tautology
$(a\supset b)\vee (b\supset a)$.
As $I$ is not a minimal model of $\Pi'^{\underline{I}}$,
it is not an answer set of $\Pi$ 
under the equilibrium logic-based semantics.
}
\end{example}

\subsection{The refined GAS principles in existing world view semantics}
\label{sec:ASP-principles-wvs}

Major existing world view semantics include 
the G91-semantics \cite{Gelfond91} and its updated version
the G11-semantics \cite{Gelfond2011},
the K14-semantics \cite{Kahl14,KahlGelfond2015},
the SE16-semantics \cite{ShenEiter16},
the founded autoepistemic equilibrium logic 
(FAEL) \cite{CabalarFC20aij},
the autoepistemic equilibrium logic (AEL) \cite{SuCH20},
and the objective rewriting semantics
\cite{DBLP:conf/iclp/Costantini022}.
We next show whether these world view semantics 
satisfy or embody the refined GAS principles. 

\subsubsection{GL91-semantics, G11-semantics and K14-semantics}

To address the problems of unintended world views
due to recursion through {\bf K} and {\bf M},
\citeA{Gelfond2011} updated the G91-semantics 
to the G11-semantics, and
Kahl et al.~(\citeyear{Kahl14,KahlGelfond2015}) 
further refined the G11-semantics to 
the K14-semantics.

Note that the above three world view 
semantics reduce to the GL-semantics
for a simple disjunctive program $\Pi$, 
i.e., an epistemic model $\cal A$ of $\Pi$ is a world view 
under the three world view semantics
iff $\cal A$ consists of all answer sets of $\Pi$
under the GL-semantics.
As the rational world view semantics
reduces to the rational answer set semantics
for epistemic-free programs
and the rational answer set semantics 
does not agree with the GL-semantics
for simple disjunctive programs,
the three world view semantics do not agree 
with the rational world view semantics
and do not embody the refined GAS principles. 
For instance, in Example~\ref{ex-rationalWVexamples}
we showed that ${\cal W}=\{I\}$ with $I=\{a,b,c\}$ is a world view
of the epistemic program $\Pi_1$
in Example~\ref{ex1-modal:generalized-strat-comp} 
under the rational world view semantics, but
this world view cannot be obtained under the above three semantics.

The following example shows that the GL91-semantics,
G11-semantics and K14-semantics 
do not satisfy the refined GAS principles;
in particular they violate the minimality principle of 
world views w.r.t.\ epistemic negation.

\begin{example}
\label{Kahl-Eg29}
{\em
The following epistemic program is borrowed from 
Example~29 in Appendix~D of \citeA{Kahl14}:
\begin{tabbing} 
\hspace{.2in} $\Pi:\quad$ \= $p\leftarrow {\bf M} q \wedge \neg q$ \`$(1)$\\
\>   $q\leftarrow {\bf M} p \wedge \neg p$ \`$(2)$
\end{tabbing}
This program has two world views
under the G91-semantics, G11-semantics
and K14-semantics, viz.\  
${\cal A}_1 = \{\{p\}, \{q\}\}$ and ${\cal A}_2 = \{\emptyset\}$,
where ${\cal A}_1$ is derived from treating ${\bf M} p$ and ${\bf M} q$
to be true and ${\cal A}_2$ derived from treating them false.
As {\bf M} stands for $\naf \neg$,
we have epistemic negations
$\Phi_1=\{\naf \neg p, \naf\neg q\}$ and $\Phi_2=\emptyset$
occurring in $\Pi$
that are true in ${\cal A}_1$ and ${\cal A}_2$, respectively.
As $\Phi_1\supset \Phi_2$,
these semantics violate the principle of
minimality w.r.t.\ epistemic negation.
Note that ${\cal A}_1$ is the only world view 
under the rational world view semantics.
}
\end{example}

\subsubsection{SE16-semantics}
For epistemic programs,
\citeA{ShenEiter16} presented a generic world view semantics 
based on a program transformation
called SE16-epistemic reducts.

\begin{definition}[{\bf SE16-epistemic reducts} \cite{ShenEiter16}]
\label{se16-reduct}
Let $\Pi$ be an epistemic program and  
$\Phi \subseteq Ep(\Pi)$ be a set of
epistemic negations occurring in $\Pi$.
The {\em SE16-epistemic reduct} $\Pi^\Phi$ 
of $\Pi$ w.r.t.\ $\Phi$ is $\Pi$ 
with every $\naf F\in \Phi$ replaced by $\top$
and every $\naf F\in Ep(\Pi)\setminus \Phi$
replaced by $\neg F$.  $\Pi$ is {\em consistent w.r.t.}\ $\Phi$ if
$\Pi^\Phi$ is consistent.
\end{definition}

Note the difference between an SE16-epistemic reduct $\Pi^\Phi$ 
of $\Pi$ w.r.t.\ $\Phi$ in the above definition
and an epistemic reduct $\Pi^{\cal A}$ 
of $\Pi$ w.r.t.\ an epistemic model $\cal A$ in 
Definition~\ref{reduct}. In the above definition,
$\Phi$ is an arbitrary set of epistemic negations occurring in $\Pi$,
while in Definition~\ref{reduct}, $\Phi$ is the set of
all epistemic negations occurring in $\Pi$
that are true in $\cal A$.

\begin{definition}[SE16-generic semantics \cite{ShenEiter16}]
\label{SE16-generic-semantics}
Let $\Pi$ be an epistemic program and
$\Pi^\Phi$ be a consistent SE16-epistemic reduct
of $\Pi$ w.r.t.\ some $\Phi \subseteq Ep(\Pi)$. 
Let $\cal X$ be an answer set semantics
for epistemic-free programs and 
${\cal A}$ be the collection of all 
answer sets of $\Pi^\Phi$ under $\cal X$.
Then ${\cal A}$ is a {\em candidate world view}
of $\Pi$ w.r.t.\ $\Phi$ and $\cal X$
if (1) ${\cal A} \neq \emptyset$,
(2) every $\mathbf{not} F\in \Phi$ is true in $\cal A$,
and (3) every $\mathbf{not} F\in Ep(\Pi)\setminus \Phi$ is false in $\cal A$.
A candidate world view ${\cal A}$ w.r.t.\ $\Phi$ and $\cal X$ 
is a {\em world view} 
if there exists no other candidate world view $\cal W$
w.r.t.\ $\cal X$ and some $\Phi' \subseteq Ep(\Pi)$ such that 
$\Phi$ is a proper subset of $\Phi'$.
\end{definition} 

The SE16-generic semantics is 
a generic framework for world view semantics,
which relies on a base answer set semantics $\cal X$.
$\cal X$ can be instantiated 
to any answer set semantics
for epistemic-free programs, thus yielding different
world view semantics.
For example, by replacing 
$\cal X$ with GL$_{nlp}$-semantics,
GL-semantics and WJ-semantics, 
we induce three world view semantics,
denoted {\em SE16-GL$_{nlp}$-semantics},
{\em SE16-GL-semantics} and {\em SE16-WJ-semantics}, 
respectively. 

An important question is: Do these
SE16-generic semantics based world view semantics adhere to 
the refined GAS principles?
We answer this question
and establish the connection
of the rational world view semantics
with the SE16-generic semantics.

Note that the above definition of a world view 
${\cal A}$ w.r.t.\ $\Phi$ and $\cal X$ 
requires $\Phi$ to be {\em maximal}, i.e.,
there exists no other candidate world view $\cal W$
w.r.t.\ $\cal X$ and some $\Phi' \subseteq Ep(\Pi)$ such that 
$\Phi$ is a proper subset of $\Phi'$.
This fulfills the principle of
minimality of world views w.r.t.\ epistemic negation. 
Therefore, the SE16-generic semantics 
with a base answer set semantics $\cal X$ 
embodies/satisfies the refined GAS principles
provided that $\cal X$ embodies/satisfies the refined GAS principles.
This claim is supported by the following theorem.

\begin{restatable}{theorem}{RthSESixteenDIWJrationalSemantics}
\label{th-SE16-DI-WJrational-semantics}
Let SE16-X-semantics be the SE16-generic semantics 
with a base answer set semantics $\cal X$.
Then SE16-X-semantics agrees with 
the rational world view semantics iff $\cal X$ agrees with 
the rational answer set semantics.
\end{restatable}

The proof of this theorem uses the following two lemmas,
where the second one says that the SE16-generic semantics 
can be characterized in terms of 
an epistemic reduct $\Pi^{\cal A}$ 
as in Definition~\ref{reduct}.

\begin{restatable}{lemma}{RlemSESixteenSemanticsEpModel}
\label{lem-se16-semantics-epModel}
Every candidate world view ${\cal A}$ of 
an epistemic program $\Pi$ w.r.t.\ 
an answer set semantics $\cal X$ 
and some $\Phi \subseteq Ep(\Pi)$
is an epistemic model of $\Pi$.
\end{restatable}

\begin{restatable}{lemma}{RlemCharaSESixteenSemantic}
\label{lem-chara-se16-semantics}
Let $\Pi$ be an epistemic program,
$\cal A$ be an epistemic model of $\Pi$,
and $\Phi$ be the set of
all epistemic negations occurring in $\Pi$
that are true in ${\cal A}$.
Let $\cal X$ be an answer set semantics
for epistemic-free programs.
Then $\cal A$ is a world view of $\Pi$ 
w.r.t.\ $\Phi$ and $\cal X$
under the SE16-generic semantics
iff (1) $\cal A$ coincides with the collection of 
all answer sets under ${\cal X}$ of
the epistemic reduct $\Pi^{\cal A}$ of $\Pi$ w.r.t.\ ${\cal A}$,
and (2) there exists no other epistemic model ${\cal W}$ of $\Pi$ 
satisfying condition (1) such that $\Phi$
is a proper subset of $\Phi_{\cal W}$ which is the set of
all epistemic negations occurring in $\Pi$
that are true in $\cal W$.
\end{restatable}

Therefore, as the GL$_{nlp}$-semantics and WJ-semantics 
agree with the rational answer set semantics,
the SE16-GL$_{nlp}$-semantics and SE16-WJ-semantics
agree with the rational world view semantics and thus
embody the refined GAS principles.
As the GL-semantics satisfies the refined GAS principles,
the SE16-GL-semantics satisfies these principles.

\subsubsection{Equilibrium logic based world view semantics}
 \citeA{CabalarFC20aij} and \citeA{SuCH20}
 presented founded autoepistemic equilibrium logic (FAEL)
and autoepistemic equilibrium logic (AEL), respectively.
They coincide with the GL-semantics
for a simple disjunctive program $\Pi$, 
i.e., an epistemic model $\cal A$ of $\Pi$ is a world view 
under FAEL and AEL
iff $\cal A$ consists of all answer sets of $\Pi$
under the GL-semantics.
Therefore, FAEL and AEL
do not agree with the rational world view semantics
and thus do not embody the refined GAS principles.
For Example, ${\cal W}=\{I\}$ with $I=\{a,b,c\}$ is a world view
of the epistemic program $\Pi_1$
in Example~\ref{ex1-modal:generalized-strat-comp} 
under the rational world view semantics, but
this world view cannot be obtained under FAEL and AEL.

Both FAEL and AEL are based on equilibrium logic, and
they reduce to the equilibrium logic-based 
answer set semantics for epistemic-free programs.
As discussed in 
Section \ref{subsubsec-Equilibrium logic-based semantics},
the equilibrium logic-based answer set semantics
and the rational answer set semantics
differ significantly.
As  the equilibrium logic-based answer set semantics
does not satisfy the refined GAS principles,
FAEL and AEL do not satisfy them either.

\subsubsection{Objective Rewriting Semantics}
\citeA{DBLP:conf/iclp/Costantini022} presented
a world view semantics  that 
rewrites epistemic programs to epistemic-free programs. 
More in detail, 
subjective literals $\mathbf{K} G$ and $\mathbf{K} \neg G$, called {\em knowledge
atoms}, are regarded as objective literals. A set $S$ of objective and
knowledge atoms is called {\em knowledge consistent}, if (i) $S$ contains $G$ whenever it
contains the knowledge atom $\mathbf{K} G$, and (ii) $S$ does not contain $G$
whenever it contains $\mathbf{K}\neg G$.
For an epistemic program $\Pi$, $\mathit{SMC}(\Pi)$ denotes the collection of all answer
sets of $\Pi$ (viewing knowledge atoms as objective atoms) that are
knowledge consistent.

A {\em CF22 world view} is a collection $\mathcal{A}$
of interpretations such that
$\mathcal{A} = \mathit{SMC}'(\Pi\mbox{-}\mathcal{A})$, where $\Pi\mbox{-}\mathcal{A}$ is the program $\Pi$
augmented with all knowledge facts $\mathbf{K} G$ and $\mathbf{K} \neg G$ such that 
$G$ occurs in every resp.\ no interpretation in $\mathcal{A}$, and
$\mathit{SMC}'(\cdot)$ is the restriction of $\mathit{SMC}(\cdot)$ to objective atoms,
i.e., all knowledge atoms are discarded. Furthermore, a   {\em CF22+SE16C world view} is a
CF22 world view that minimizes epistemic knowledge, i.e., there is no
CF22 world view $\mathcal{A}'$ such that in constructing $\Pi$-$\mathcal{A}'$ a proper
subset of knowledge atoms is added to $\Pi$ compared to  $\Pi$-$\mathcal{A}$.

Constantini and Formisano pointed out that in contrast to other
world view semantics, checking whether an epistemic model is a CF22
world view can be done with an ordinary answer set solver and no specialized
world view solver is needed. They examined the CF22 and CF22+SE16C world view
semantics and showed that the latter agrees with the SE16-GL
semantics on a number of examples. However, both the CF22 and CF22+SE16C
semantics yield for the program in Example~\ref{Kahl-Eg29} the same world views as the G11-semantics;
hence CF22 and CF22+SE16C do not satisfy the refined GAS principles.

\section{Computation and Complexity}
\label{sec-Computation and Complexity}

In this section, we turn to computational aspects of
well-supportedness and the rational answer set and world view semantics.  Specifically, we provide
concrete algorithms for deciding well-supportedness and constructing a
strict well-founded partial order for certain epistemic-free programs,
which are then exploited to characterize the complexity of deciding
well-supportedness and constructing  such an order. Furthermore, we
consider for rational semantics the canonical problems of
answer set resp.\ world view existence, as well as brave and cautious
inference from the answer sets resp.\ truth of a formula in some world view.

\subsection{Algorithms for well-supportedness}
\label{sec:alg-well-supportedness}

Theorem~\ref{th-wellsupport-lfp} supports us to  
develop Algorithm~\ref{algo:well-support-with-partialorder}
for epistemic-free normal programs with atomic rule heads,
which
determines the well-supportedness of a model $I$
and further constructs a strict well-founded
partial order $\prec$ on $I$ when it is well-supported.

Specifically, we compute the least fixpoint 
$\mathit{lfp}(T_{\Pi}(\emptyset,\neg I^-))$
via the sequence $\langle S_i
\rangle_{i=0}^\infty$, where for every $i\geq 0$,
$S_i = T_{\Pi}^i(\emptyset,\neg I^-)$.
Starting from $S_0 = \emptyset$ (line 2),
for every $i\geq 0$, $S_{i+1}$ is $S_i$
plus all rule heads $head(r)$ 
such that a rule $head(r)\leftarrow body(r)$
is in $\Pi$ and $S_i\cup \neg I^- \models body(r)$
(lines 4 and 7). When $S_i\cup \neg I^- \models body(r)$, 
a relation $q\prec head(r)$ is added to 
${\cal O}$ for every $q\in S_i$; this constructs a
strict well-founded partial order $\prec$ on $S_{i+1}$ (line 8). 
To avoid redundant computation,
we also remove all rules whose heads $head(r)$ are already in $S_{i+1}$ (line 9). 
When $S_{i+1}= S_i$ for some $i\geq 0$ (i.e., when flag$=0$ after some 
while loop (line 3)),
we reach the least fixpoint $S_i = \mathit{lfp}(T_{\Pi}(\emptyset,\neg I^-))$.
If $S_i=I$ (line 15), then by Theorem~\ref{th-wellsupport-lfp},
$I$ is a well-supported model and thus 
a strict well-founded partial order $\prec$ on $I$
as defined in ${\cal O}$ is returned (line 16).
Otherwise, $I$ is not well-supported (line 19).

\begin{algorithm}[t]
	\caption{Determining well-supportedness
and constructing a strict well-founded
partial order for an epistemic-free normal program
with atomic rule heads}
	\label{algo:well-support-with-partialorder}
	\KwIn{an epistemic-free normal program $\Pi$
with atomic rule heads, and
a model $I$ of $\Pi$.}
	\KwOut{a strict well-founded partial order $\prec$ on $I$
 as defined in ${\cal O}$ when $I$ is well-supported w.r.t.\ $\Pi$; ``NOT WELL-SUPPORTED", otherwise.}
${\cal R} = \{r \mid r\in \Pi$ \mbox{and $I$ satisfies} $body(r)\}$\;
$S_0 = \emptyset$; ${\cal O}=\emptyset$; $i=0$; flag $=1$\;
  \While{${\cal R}\neq \emptyset$ and flag $\neq 0$}
  { $S_{i+1} = S_i$; flag$=0$\;
    \For{every rule $r\in {\cal R}$}
    {
      \If{$S_i\cup \neg I^- \models body(r)$}
      {Add $head(r)$ to $S_{i+1}$\;
       Add $q\prec head(r)$ to ${\cal O}$ for every $q\in S_i$\;
        Remove from ${\cal R}$ all rules whose heads are $head(r)$\;  
        flag$=1$\;    
      }
    }
     $i=i+1$\;
  }
   \If{$S_i=I$}
      {return $\cal O$\;}
   \Else{return ``NOT WELL-SUPPORTED"}        
\end{algorithm}

As shown in Section~\ref{sec:complexity} below, Algorithm \ref{algo:well-support-with-partialorder} 
can be implemented to run in polynomial time with an \NP{}  oracle.


Theorem~\ref{th-wellsupport-lfp-normal}  
supports us to develop Algorithm~\ref{algo:well-support-with-partialorder-2},
which automatically determines the well-supportedness
of a model $I$ for an unrestricted epistemic-free 
normal program $\Pi$
and constructs a strict well-founded
partial order $\prec$ on $I$ when it is well-supported.

\begin{algorithm}[t]
	\caption{Determining well-supportedness
and constructing a strict well-founded
partial order for an epistemic-free normal program}
	\label{algo:well-support-with-partialorder-2}
	\KwIn{an epistemic-free normal program $\Pi$, and
a model $I$ of $\Pi$.}
	\KwOut{a strict well-founded partial order $\prec$ on $I$
 as defined in ${\cal O}$ when $I$ is well-supported w.r.t.\ $\Pi$; ``NOT WELL-SUPPORTED", otherwise.}
${\cal R} = \{r \mid r\in \Pi$ \mbox{and $I$ satisfies} $body(r)\}$\;
$S_0 = \emptyset$; ${\cal O}=\emptyset$; ${\cal O}_h=\emptyset$; $i=0$; flag $=1$\;
  \While{${\cal R}\neq \emptyset$ and flag $\neq 0$}
  {$R(S_i) = \{r\mid r\in {\cal R}$ and $S_i\cup \neg I^- \models body(r)\}$; ${\cal R} ={\cal R} \setminus R(S_i)$\;
  \If{$R(S_i)\neq \emptyset$}
  {$S_{i+1} = S_i\cup \{head(r)\mid r\in R(S_i)\}$\;
    \For{every rule $r\in R(S_i)$}
    {
       Add $F\prec_h head(r)$ to ${\cal O}_h$ for every $F\in S_i$\;
       Add $q\prec p$ to ${\cal O}$ for every $q$ and $p$ in $I$, where $S_{i+1}\cup \neg I^- \models p$, $S_i\cup \neg I^- \not\models p$ and $S_i\cup \neg I^- \models q$\;
        Remove from ${\cal R}$ all rules $r'$ such that $head(r')\equiv head(r)$\;     
      }
       $i=i+1$\;
    }
    \Else{flag$=0$\;}
  }
   \If{$S_i\cup \neg I^-\models p$ for every $p\in I$}
      {return $\cal O$\;}
   \Else{return ``NOT WELL-SUPPORTED"}        
\end{algorithm}

Like Algorithm~\ref{algo:well-support-with-partialorder},
we compute the least fixpoint 
$\mathit{lfp}(T_{\Pi}(\emptyset,\neg I^-))$
via the sequence $\langle S_i
\rangle_{i=0}^\infty$, where for every $i\geq 0$,
$S_i = T_{\Pi}^i(\emptyset,\neg I^-)$.
Starting from $S_0 = \emptyset$ (line 2),
for every $i\geq 0$, $S_{i+1}$ is $S_i$
plus all rule heads $head(r)$ 
such that a rule $head(r)\leftarrow body(r)$
is in $\Pi$ and $S_i\cup \neg I^- \models body(r)$
(lines 4 - 6). 
When $S_i\cup \neg I^- \models body(r)$, 
a relation $F\prec_h head(r)$ is added 
to ${\cal O}_h$ for every $F\in S_i$;
this constructs a strict well-founded
partial order $\prec_h$ on rule heads 
in $S_{i+1}$ (line 8).
Meanwhile, a relation
$q\prec p$ is added to 
${\cal O}$ for every $q$ and $p$ in $I$, 
where $S_{i+1}\cup \neg I^- \models p$, 
$S_i\cup \neg I^- \not\models p$ 
and $S_i\cup \neg I^- \models q$;
this constructs a strict well-founded
partial order $\prec$ on $I$ (line 9). 
To avoid redundant computation,
we remove all rules $r'$ 
whose heads $head(r')$ are logically 
equivalent to $head(r)$,
i.e., $head(r')\equiv head(r)$ (line 10). 
This amounts to simplifying
the least fixpoint 
$\mathit{lfp}(T_{\Pi}(\emptyset,\neg I^-))$
by removing from it all rule heads $head(r')$ 
that are logically equivalent to $head(r)$. 

When $S_{i+1}= S_i$ for some $i\geq 0$ 
(i.e., when ${\cal R} = \emptyset$ or flag$=0$ after some 
while loop (line 15)),
we reach the least fixpoint $\delta=S_i$ 
which is $\mathit{lfp}(T_{\Pi}(\emptyset,\neg I^-))$ 
except that all rule heads $head(r')$ that are 
logically equivalent to $head(r)$ are removed at 
line 10 of Algorithm~\ref{algo:well-support-with-partialorder-2}. 
That is, let $\alpha = 
\mathit{lfp}(T_{\Pi}(\emptyset,\neg I^-)) \setminus \delta$;
then for every $H'\in \alpha$, there is a unique rule head
$H$ in $\delta$ with $H\equiv H'$. 
This means that for any formula $F$, $\delta\models F$
iff $\mathit{lfp}(T_{\Pi}(\emptyset,\neg I^-))\models F$. 
Therefore, if $\delta\cup \neg I^-\models p$ 
for every $p\in I$
(line 18), then 
$\mathit{lfp}(T_{\Pi}(\emptyset,\neg I^-))\cup \neg I^-\models p$ for every $p\in I$
and thus by Theorem~\ref{th-wellsupport-lfp-normal},
$I$ is a well-supported model; in this case, 
a strict well-founded partial order on $I$ as defined in 
${\cal O}$ is returned (line 19).
Otherwise, $I$ is not well-supported (line 21). 
This leads to the following result.

\begin{restatable}{theorem}{RthAlgoWellSupportWithPartialorderTwo}
\label{th-algo:well-support-with-partialorder-2}
Given a finite epistemic-free normal program $\Pi$ 
and a model $I$ of $\Pi$, Algorithm \ref{algo:well-support-with-partialorder-2} 
terminates and outputs a strict well-founded 
partial order $\prec$ on
$I$ w.r.t.\ $\Pi$ satisfying the conditions of 
Definition~\ref{def-well-supported-normal}
if $I$ is well-supported in $\Pi$, and 
"NOT WELL-SUPPORTED" otherwise. 
\end{restatable}

A shown in Section \ref{sec:complexity} below,
Algorithm \ref{algo:well-support-with-partialorder-2}
can also be implemented to run in polynomial time with an \NP{} oracle.

\begin{example}
\label{eg-norm-pro-1}
{\em
Consider the following epistemic-free normal program:
\begin{tabbing} 
\hspace{.3in} 
$\Pi:\ $ \=  $e\vee d$\`$(1)$\\[2pt]
\> $e\vee \neg d$\`$(2)$\\[2pt]
\> $a\vee (\neg b\wedge c)$\`$(3)$\\[2pt]
\> $a\vee d\leftarrow b\vee c$ \` $(4)$ \\[2pt]
\> $(a\vee \neg b) \wedge (a\vee c) \leftarrow c\wedge d \wedge e$  \` $(5)$
\end{tabbing}

Consider a model $I = \{c,d,e\}$ of $\Pi$.
Let $\neg I^-=\{\neg a, \neg b\}$.
The least fixpoint of the inference sequence 
$\langle T_{\Pi}^i(\emptyset,\neg I^-)
\rangle_{i=0}^\infty$ w.r.t.\ $I$ is
$\mathit{lfp}(T_{\Pi}(\emptyset,\neg I^-))=\{e\vee d, \ e\vee \neg d, \ a\vee (\neg b\wedge c), \ a\vee d, \ (a\vee \neg b) \wedge (a\vee c)\}$, where 
$T_{\Pi}^0(\emptyset,\neg I^-)=\emptyset$,
$T_{\Pi}^1(\emptyset,\neg I^-)=\{e\vee d, \ e\vee \neg d, \ a\vee (\neg b\wedge c)\}$,
$T_{\Pi}^2(\emptyset,\neg I^-)=\{e\vee d, \ e\vee \neg d, \ a\vee (\neg b\wedge c), \ a\vee d\}$
and $T_{\Pi}^3(\emptyset,\neg I^-)=\{e\vee d, \ e\vee \neg d, \ a\vee (\neg b\wedge c), \ a\vee d, \ (a\vee \neg b) \wedge (a\vee c)\}$.
As $\mathit{lfp}(T_{\Pi}(\emptyset,\neg I^-))\cup \neg I^-\models c\wedge d\wedge e$,
$I$ is well-supported in $\Pi$ 
by Theorem~\ref{th-wellsupport-lfp-normal}.
Note that the head of rule (3) is
logically equivalent to that of rule (5),
i.e., $a\vee (\neg b\wedge c)\equiv (a\vee \neg b) \wedge (a\vee c).$

We now illustrate the process of 
Algorithm \ref{algo:well-support-with-partialorder-2}.
Initially, as $I$ satisfies all rule bodies of $\Pi$,
we let ${\cal R} = \Pi$.
Let $S_0 = \emptyset$, ${\cal O}=\emptyset$
and ${\cal O}_h=\emptyset$.

For $i=0$, 
as rules (1) - (3) have no bodies,
we obtain $S_1=\{e\vee d, \ e\vee \neg d, \ a\vee (\neg b\wedge c)\}$.
As the head of rule (3) is
logically equivalent to that of rule (5),
rule (5) is also removed from ${\cal R}$.
As a result, after this cycle (lines 3 - 17),
${\cal R}$ contains only rule (4).

For $i=1$, for the rule 
$a\vee d\leftarrow b\vee c$ in ${\cal R}$,
as $S_1\cup \neg I^- \models b\vee c$,
we obtain $S_2=\{e\vee d, \ e\vee \neg d, \ a\vee (\neg b\wedge c), \ a\vee d\}$. We also obtain
${\cal O}_h=\{e\vee d\prec_h a\vee d, \ e\vee \neg d\prec_h a\vee d, \ a\vee (\neg b\wedge c)\prec_h a\vee d\}$
and ${\cal O}=\{e\prec d, \ c\prec d\}$,
where $e \prec d$ and $c\prec d$ are 
obtained due to
$S_2\cup\neg I^- \models d$,
$S_1\cup \neg I^- \not\models d$ and
$S_1\cup \neg I^- \models e\wedge c$.

For $i=2$, as $S_3=S_2$, we reach 
the least fixpoint $\delta = S_2$.
As $\delta \cup \neg I^- \models c\wedge d\wedge e$, 
$\mathit{lfp}(T_{\Pi}(\emptyset,\neg I^-))\cup \neg I^-\models c\wedge d\wedge e$ 
and $I$ is well-supported  in $\Pi$ 
by Theorem~\ref{th-wellsupport-lfp-normal}.
Consequently, Algorithm \ref{algo:well-support-with-partialorder-2} 
outputs a strict well-founded
partial order ${\cal O} =\{e\prec d, \ c\prec d\}$ on $I$.
}
\end{example}

\subsection{Computational complexity of well-supportedness}
\label{sec:complexity}

We next consider computational complexity of  well-supportedness,
Specifically, we focus on recognizing 
well-supported models of a logic program, but also address the complexity of computing a
strict partial order that witnessed well-supportedness of a model. 

\subsubsection{Complexity of recognizing well-supported models}

In our analysis, we consider different axes of syntactic
restrictions on logic programs:  

\begin{enumerate}
 \item as for building blocks, a) simple programs, which allow
in heads atoms and in bodies atoms $A$ and negated atoms $\neg A$, b) atomic
head formula (atomic head for short)
programs, which allow in heads atoms and in bodies arbitrary
formulas, and c) general programs with no restriction on formulas in
heads and bodies; 
\item as for alternative conclusions, a) normal programs with a head $H_1$  and b)
 disjunctive programs with a head $H_1\mid H_2 \mid\cdots \mid H_m$;
\item as for introspection, a) epistemic-free programs, and b)
epistemic programs. For the latter, we generalize the notion of simple
programs to fit \citeS{Gelfond91} epistemic specifications; i.e., in rule bodies also 
$(\neg)\naf A$ and $(\neg)\naf \neg A$ are allowed.
\end{enumerate}

Furthermore, for epistemic programs we consider 
both explicit representation of an epistemic model $\cal A$ in the
input (i.e., as a set of models) and implicit representation by an
abstract storage $D({\cal A})$ with a Boolean access function $in(I)$
that is evaluable in polynomial time and returns true iff $I\in {\cal
 A}$ holds. This accounts for the fact that in general, an epistemic model $\cal
A$ may consist of exponentially many interpretations, and thus storing
$\cal A$ simply as a set of models is infeasible. A more compact form
of storage is needed, for which a variety of possibilities exist. We
do not commit to any particular method, and only require that
interpretations in $\cal A$ can be looked up efficiently, similar to
words in an abstract dictionary.  In practice, $D({\cal A})$ may e.g.\
be a model tree, where like in a trie data structure interpretations
are stored where common parts are shared, a simple normal
epistemic-free logic program, a Boolean formula or a Boolean circuit,
possible in a special form (CNF, DNF, negation normal form graph etc),
a binary decision program, or simply a piece of code that can be run
in polynomial time.  Notably, implicit representation can be
exponentially more succinct than explicit representation 
and thus may lead
to a complexity increase.

\paragraph{Overview of complexity results for recognizing well-supported models}

Our complexity results for the recognition problem 
are summarized in Table~\ref{tab:complexity}.
\begin{table} 
\renewcommand{\arraystretch}{1.2}
\centering\begin{tabular}{|l|ccc|ccc|}
\hline
  \multirow{2}{*}{program}
  & \multicolumn{3}{|c|}{normal }    &   \multicolumn{3}{|c|}{disjunctive} \\
  & simple & atomic head & general   &   simple & atomic head & general \\ \cline{1-7}
epistemic-free & \Pol & \coNP & \coNP & \NP & $\Sigma^p_2$ & $\Sigma^p_2$ \\
epistemic (explicit) & \Pol & \coNP & \coNP & \NP & $\Sigma^p_2$ & $\Sigma^p_2$ \\
epistemic (implicit)  & \PNPPar & \PNPPar & \PNPPar & $\Pi^p_2$ & $\Pi^p_3$
        & $\Pi^p_3$ \\
\hline        
\end{tabular}
\caption{Complexity of recognizing well-supported models (completeness results)}
\label{tab:complexity}
\end{table}

Informally, the complexity results can be explained as follows.  
For normal epistemic-free programs, 
we have to check the condition of 
Theorem~\ref{th-wellsupport-lfp-normal}, i.e., 
$\mathit{lfp}(T_{\Pi}(\emptyset,\neg I^-))\cup \neg I^-\models p$
for every $p\in I$. The key for establishing that this
is feasible in \coNP \ is that we do not need to compute  
$\mathit{lfp}(T_{\Pi}(\emptyset,\neg I^-))$ in order to decide the
condition; by monotonicity of classical logic, we may use
an over-approximation of it to
refute $\mathit{lfp}(T_{\Pi}(\emptyset,\neg I^-))\cup \neg I^-\models
p$ for some atom $p$, which means $I$ is not well-supported.
On the other hand, \coNP \ is a lower bound if rule bodies can be
arbitrary propositional formulas, as testing $S \cup \neg I^-\models
body(r)$ is \coNP-complete in general. For a simple program $\Pi$, the
latter problem is tractable (given that $S$ is a set of atoms), and
thus testing well-supportedness is feasible in \Pol; on the other
hand, testing $\mathit{lfp}(T_{\Pi}(\emptyset,\neg I^-))=I$ subsumes
minimal model checking of positive Horn logic programs, which is well-known
to be \Pol-complete, cf.\ \citeA{Leone2006}.
 
For epistemic-free programs with disjunctive heads, the complexity
increases by one level in the Polynomial Hierarchy (PH), which is
intuitively due to the need for a head selection function in
order to establish the well-supportedness of a given model. 
Notably, variant rule heads can be recognized in \coNP \ in general and
in polynomial time for atomic rule heads; thus a guess for $\sel$ can
be checked in polynomial time with an \NP{} oracle, resp.\ without one
for simple programs.

For epistemic programs, deciding well-supportedness of an epistemic
model $\cal A$ reduces to the conjunction of well-supported model checks
for each model $I$ in $\cal A$ w.r.t.\ the epistemic reduct $\Pi^{\cal A}$.
Under explicit representation of $\cal A$, the latter can be built in
polynomial time. Since \Pol, \coNP, \NP{} and $\Sigma^p_2$ are all
closed under polynomial time transformation and conjunctions, the same
complexity upper bounds as in the epistemic-free cases follow; the
lower bounds are inherited from the latter. 

Under implicit representation, computing the epistemic reduct
$\Pi^{\cal A}$ is no longer tractable, but feasible in polynomial time
with parallel \NP{} oracle calls (i.e., is in the class
\FPNPPar). Well-supportedness of $\cal A$ in $\Pi$ can then be
disproved with a guess for a model $I$ in $\cal A$ that is not
well-supported in the epistemic reduct $\Pi^{\cal A}$. For normal
epistemic programs, the guess can be checked in polynomial time, which
leads to membership in \PNPPar; the matching \PNPPar-hardness is
intuitively induced by the \FPNPPar-hardness of computing $\Pi^{\cal A}$.
For simple resp.\ general disjunctive epistemic programs, the check of
the guess $I$ needs a \coNP{} resp.\ $\Pi^p_2$ oracle, which places
the problem into $\Pi^P_2$ resp.\ $\Pi^p_3$; the matching hardness
parts are shown by lifting the constructions for the epistemic-free
cases.  In fact, the hardness results for disjunctive epistemic programs are
established for Boolean formulas that are factorized into a
conjunction of small CNFs on disjoint variables. Furthermore, the
formulas can be easily converted to (ordered) binary decision diagrams
(OBDDs), which are a subclass of sentential decision diagrams, cf.\
\citeA{DBLP:conf/ijcai/Darwiche11}, which in turn amount to a subclass
of d-DNNFs formulas; the latter are a popular formalism for
knowledge compilation that is also utilized for storing the models of
logic programs, cf.\ \cite{DBLP:journals/tplp/FierensBRSGTJR15,KR2021-26}.

\paragraph{Derivation of the results}

\noindent{\em Epistemic-free programs.} Table~\ref{tab:complexity}, first
row: 
\smallskip

For simple normal epistemic-free programs, recognizing
well-supportedness is tractable. 

\begin{restatable}{theorem}{Rthcompwellsupportsimplenormal}
\label{th-comp-wellsupport-simple-normal}
Deciding whether a model $I$ of a simple epistemic-free normal program $\Pi$
is well-supported in $\Pi$ is \Pol-complete.
\end{restatable}

The following lemma is key to show that deciding well-supportedness of
a model $I$ of an epistemic-free normal program is in
\coNP{} in general.

\begin{lemma}
\label{lem:guess}
A model $I$ of an epistemic-free normal program $\Pi$
is well-supported in $\Pi$ iff every sequence $S'_0,S'_1,\ldots,S'_i$ 
such that $S_j\subseteq S'_j$, $j=0,\ldots,i$, where $S_j$ is the set
constructed by Algorithm~\ref{algo:well-support-with-partialorder-2} and $S_i$ is the least fixpoint,
fulfills $S_i \cup \neg I^- \models p$ for every $p \in I$.
\end{lemma}

\begin{restatable}{theorem}{Rthcompwellsupportlfpnormal}
\label{th-comp-wellsupport-lfp-normal}
Deciding whether a model $I$ of an epistemic-free normal program $\Pi$
is well-supported in $\Pi$ is \coNP-complete, and is \coNP-hard even
if all rule heads are atomic.
\end{restatable}

We next consider the case of programs with disjunctions in the
head. 

\begin{restatable}{theorem}{Rthcompwellsupportdisjunctivesimple}
\label{th-comp-wellsupport-disjunctive-simple}
Deciding whether a model $I$ of a simple epistemic-free program $\Pi$
is well-supported in $\Pi$ is \NP-complete, and is \NP-hard even
if all rule head formulas are atomic.
\end{restatable}

\begin{restatable}{theorem}{Rthcompwellsupportdisjunctivegeneral}
\label{th-comp-wellsupport-disjunctive-general}
Deciding whether a model $I$ of an epistemic-free program $\Pi$
is well-supported in $\Pi$ is $\Sigma^p_2$-complete, and
$\Sigma^p_2$-hardness holds even if all rule head formulas are atomic.
\end{restatable}

\noindent{\em Epistemic programs.} Table~\ref{tab:complexity}, second row:
\smallskip

The following lemma is easy to show from Definitions~\ref{satisfaction-epis-program} and \ref{reduct}.
\begin{lemma}
\label{lem:epi-reduct}
Given an epistemic program $\Pi$ and a set $\cal A\neq \emptyset$ of
interpretations, deciding whether $\cal A$ is an epistemic model of $\Pi$ is
feasible in polynomial time, as well as computing the set
$\Phi\subseteq Ep(\Pi)$ and the epistemic reduct $\Pi^{\cal A}$.
\end{lemma}

We then derive the following results for well-supportedness checking.

\begin{restatable}{theorem}{Rthcompwellsupportepistemicnormal}
\label{th-comp-wellsupport-epistemic-normal}
Deciding whether an epistemic model $\cal A$ of an epistemic normal program $\Pi$
is well-supported in $\Pi$ is (i) \Pol-complete if $\Pi$ is simple and
(ii) \coNP-complete otherwise, where \coNP-hardness holds even if all rule
head formulas are atomic.
\end{restatable}

\begin{restatable}{theorem}{Rthcompwellsupportepistemicgeneral}
\label{th-comp-wellsupport-epistemic-general}
Deciding whether an epistemic model $\cal A$ of an epistemic program $\Pi$
is well-supported in $\Pi$ is (i) \NP-complete if $\Pi$ is simple and
(ii) $\Sigma^p_2$-complete otherwise, where $\Sigma^p_2$-hardness holds even if all rule
head formulas are atomic.
\end{restatable}

\noindent Table~\ref{tab:complexity}, third row: \smallskip

We finally consider well-supportedness checking when an epistemic
model is implicitly represented. In this case, computing the epistemic
reduct is no longer tractable.

\begin{restatable}{lemma}{Rlemepireductimplicit}
\label{lem:epi-reduct-implicit}
Given an epistemic program $\Pi$ and an epistemic model $\cal A$ of $\Pi$
represented by a Boolean formula $E$, computing the epistemic reduct
$\Pi^{\cal A}$ is \FPNPPar-complete.
\end{restatable}

\begin{restatable}{theorem}{ThCompWellsupportEpistemicNormalImplicit}
\label{th-comp-wellsupport-epistemic-normal-implicit}
Deciding whether an epistemic model $\cal A$ of an epistemic normal program $\Pi$
is well-supported in $\Pi$ is \PNPPar-complete under implicit
representation. Morevoer, \PNPPar-hardness holds if $D(\cal A)$ is
a Boolean formula and $\Pi$ is simple.
\end{restatable}

\begin{restatable}{theorem}{Rthcompwellsupportepistemicgeneralimplicit}
\label{th-comp-wellsupport-epistemic-general-implicit}
Deciding whether an implicitly represented epistemic model $\cal A$ of an epistemic program $\Pi$
is well-supported in $\Pi$ is (i) $\Pi^p_2$-complete if $\Pi$ is simple and
(ii) $\Pi^p_3$-complete otherwise, where $\Pi^p_3$-hardness holds even if all rule
head formulas are atomic.
\end{restatable}

\subsubsection{Complexity of computing a witnessing strict partial order}

We finally briefly address the complexity of computing a strict partial order $\cal O$
witnessing that a given model $I$ of a logic program $\Pi$ is well-supported
in it, where we confine to the output produced by
Algorithm~\ref{algo:well-support-with-partialorder-2} and give an
informal overview of results; formal statements with proofs are
provided in Appendix~\ref{WitnessingPorder}.

For simple normal epistemic-free programs, the problem is
solvable in polynomial time, since all logical consequence tests are
feasible in polynomial time and the sets $S_i$, $i \geq 0$, increase
while $\cal R$ decreases.

In case of normal epistemic-free programs with arbitrary propositional
formulas in rule bodies, the problem is no longer tractable since
deciding logical entailment $S \cup \neg I^- \models body(r)$ is
\coNP-complete. However, it is only mildly harder than the latter problem as access to
parallel \NP{} oracle queries is sufficient; in fact, it can be
shown that computing the strict partial order
$\cal O$ for $I$ generated by Algorithm~\ref{algo:well-support-with-partialorder-2}  (if it exists)
is \FPNPPar-complete, and \FPNPPar-hard if $I$ is well-supported and all rule
heads are atomic. In the light of this, Algorithm~\ref{algo:well-support-with-partialorder-2} is not
worst-case optimal as it makes polynomially many adaptive \NP{} oracle calls, i.e.\ the input of an oracle call may depend on the output of
another call. However, the parallelization of these calls (which essentially underlies the \FPNPPar{} upper bound result) will
introduce larger oracle inputs (specifically, SAT instances) and requires parallel computing devices to be effective. 

For disjunctive epistemic-free programs, Algorithm~\ref{algo:well-support-with-partialorder-2} would have
to be extended by producing also a suitable head selection function
$\sel$ such that $I$ is well-supported in the disjunctive program reduct
$\Pi^I_{\sel}$, and $\sel$ should be output with 
the partial order $\cal O$. The emerging complexity matches for 
simple resp.\ general disjunctive epistemic-free programs MV\NP{}
resp.\ MV$\Sigma^p_2$, which is the class of multi-valued functions
computable on a nondeterministic Turing machine with no resp.\ with an \NP{} oracle in polynomial time;
this can be derived from the proofs of Theorems~\ref{th-comp-wellsupport-disjunctive-simple}
and~\ref{th-comp-wellsupport-disjunctive-general}, respectively.

For epistemic programs, a collection of pairs $(\sel,{\cal O})$
witnessing that every interpretation $I$ in a given epistemic model 
$\cal A$ is well-supported in the reduct $\Pi^{\cal A}$ would constitute
a natural notion of witness for the well-supportedness of $\cal A$ in
$\Pi$. However, we refrain here from analyzing the complexity of such
witnesses in order to avoid a proliferation of results. 

\subsection{Complexity of Rational Semantics}

In this section, we turn to the complexity of the rational answer set and
rational world view semantics.  It appears that it is closely related to
the complexity of the DI-WJ semantics for epistemic-free programs and the SE16
semantics for epistemic programs based on the DI-WJ-semantics, respectively.

Notably, by Corollary~\ref{cor-rational-epistemic-free-normal-pro} the 
complexity of standard reasoning tasks, viz.\ (R1) deciding 
answer set existence, (R2) brave inference resp.\ cautious inference
of an epistemic-free formula
from the answer sets of a program, coincides in the case of
epistemic-free normal programs for rational answer
sets with the one for WJ-semantics, derived in \cite{ShenWEFRKD14}.%
\footnote{The membership parts of the results were shown for a formula $F$
that is a literal resp.\ an atom, but the proofs can be trivially
extended to arbitrary $\naf$-free formulas $F$.\label{fn1}
}

For disjunctive and epistemic programs, the following lemmas are useful.

\begin{restatable}{lemma}{LemWJEqualsRational}
\label{lem-wj-rational}
For a program $\Pi$ that has no head variants, 
(i) its rational answer sets  coincide with its 
DI-WJ answer sets (where $\Pi$ is epistemic-free), and (ii)
the rational world views coincide with its SE16-DI-WJ world views. 
\end{restatable}

\begin{restatable}{lemma}{LemReduceRationalNoVariant}
\label{lem-reduce-rational-no-variant}
Every disjunctive program $\Pi$ is reducible in polynomial time to a
disjunctive program $\Pi'$ with no head variants that 
has the same rational answer sets (given $\Pi$ is epistemic-free)
resp.\ rational world views.
Furthermore, the reduction preserves simple programs and atomic rule heads.
\end{restatable}

Consequently, the complexity of rational answer sets resp.\ rational
world views coincides for all
classes of programs considered in Table~\ref{tab:complexity} with the
complexity of  DI-WJ-answer sets 
resp.\ SE16-DI-WJ world views.
Since the DI-WJ answer sets of
a general normal program are its well-justified answer sets 
\cite{ShenEiter16}[Corollary~2], it coincides in case of normal programs with the 
complexity of WJ-answer sets resp.\ SE16-WJ world views. 

Based on results in the literature \cite{ShenWEFRKD14,ShenEiter16,ShenE19}, we
thus establish the complexity picture of rational semantics for the
reasoning tasks (R1)-(R3) on answer sets, as well as for (E1) deciding world
view existence and (E2) deciding truth of a formula in some world view, shown in 
Table~\ref{tab:complexity2}. 

\begin{table} 
\resizebox{\textwidth}{!}{
\renewcommand{\arraystretch}{1.2}
\centering\begin{tabular}{|l|ccc|ccc|}
\hline
  \multirow{2}{*}{program}
  & \multicolumn{3}{|c|}{normal }    &   \multicolumn{3}{|c|}{disjunctive} \\
  & simple & atomic head & general   &   simple & atomic head & general \\ \cline{1-7}
epistemic-free & $\NP \mid \NP \mid \coNP$ &
\multicolumn{2}{c|}{$\Sigma^p_2 \mid \Sigma^p_2 \mid \Pi^p_2$} &
$\NP \mid \Sigma^p_2 \mid \Pi^p_2$ &
\multicolumn{2}{c|}{$\Sigma^p_2 \mid \Sigma^p_3 \mid \Pi^p_3$} \\
epistemic & {$\Sigma^p_2 \mid \Sigma^p_3$} &
\multicolumn{2}{c|}{$\Sigma^p_3 \mid \Sigma^p_4$} &
{$\Sigma^p_3 \mid \Sigma^p_4$} & \multicolumn{2}{c|}{$\Sigma^p_4 \mid \Sigma^p_5$}
 \\
\hline        
\end{tabular}
}
\caption{Complexity of  rational semantics (propositional
  case). Entries list (R1) $\mid$ (R2) $\mid$ (R3), where (R1) is existence of some rational answer set, (R2)
 is brave resp.\ (R3) cautious inference of an epistemic-free formula $F$ from the rational answer sets, as well as (E1) $\mid$ (E2), where
 (E1) is existence of some rational world view and (E2) is truth of a
 formula $F$ in some rational world view, i.e., in all its rational answer sets (completeness results;
hardness holds for $F$ being an atom)}
\label{tab:complexity2}
\end{table}

\paragraph{Epistemic-free programs} 

Specifically, for epistemic-free programs, all results in 
Table~\ref{tab:complexity2} follow by the consideration above from results for 
WJ answer sets and DI-WJ answer sets in \cite{ShenWEFRKD14,ShenE19}
excepting the hardness results for programs with atomic
heads.%
\footnote{Cf.\ Footnote~\footref{fn1}}
Notably, the hardness proofs do not involve variant heads and hold for
literal formulas. 

The $\Sigma^p_2$-hardness of WJ-answer set existence for normal
programs with atomic heads follows from the proof of Theorem~14 in
\cite{ShenWEFRKD14}; the $\Sigma^p_2$-hardness
resp.\ $\Pi^p_2$-hardness of brave and cautious inference of an atom
are easy corollaries.

It thus remains to establish the $\Sigma^p_3$-
resp.\ $\Pi^p_3$-hardness results for brave resp.\ cautious reasoning
from an epistemic-free disjunctive programs with atomic heads.  This
can be shown by a slight adaptation of the proof of the respective result for
epistemic-free disjunctive programs with literals in heads
\cite{ShenE19}[Theorem~6].

\begin{restatable}{theorem}{RtDJWatomicheadshardness}
\label{th-DJ-WJ-atomic-heads-hardness}
Given an epistemic-free program $\Pi$ with atomic heads and an atom
$A$, deciding whether $A$ is true in some DI-WJ answer set of $\Pi$ is $\Sigma^p_3$-hard.
\end{restatable}

\paragraph{Epistemic programs} 

For simple normal epistemic programs, the results in Table~\ref{tab:complexity2}
follow from the results for EFLP-semantics in $X'\cup \neg(X\setminus
X') \models \forall Y.E(X,Y)$\cite{ShenEiter16}, for
which world view existence and
truth of a literal in some world view 
were shown to be $\Sigma^p_2$-complete and
$\Sigma^p_3$-complete,%
\footnote{This is easily generalized to truth of an epistemic-free formula in some
world view,}
respectively. This is because
for such programs, the SE16-WJ world views (called EWJ world views in
\cite{ShenEiter16}) coincide with the EFLP
views as remarked in \cite{ShenEiter16}, for which these complexity
results were proven in Theorems~8 and~9 ibidem.


Beyond simple normal epistemic programs, in Theorems~5 and ~6 of
\cite{ShenEiter16} world view existence and
truth of a literal in some EFLP world views were shown to be 
$\Sigma^p_3$- and $\Sigma^p_4$-complete, respectively for general normal
epistemic programs, and it was remarked that EWJ world views have the
same complexity as (but do not coincide with) the EFLP world views. 

The memberships parts of all results in
the second row of Table~\ref{tab:complexity2} are  easily obtained
from the following lemma, which hinges on the results about rational answer set semantics in the first row
of Table~\ref{tab:complexity2}. 

\begin{restatable}{lemma}{Rlemepiratmembership}
\label{lem-epi-rat-membership}
Given a program $\Pi$ from a class of epistemic programs such that for
its epistemic-free fragments brave and cautious inference of an epistemic-free formula
are in $\Sigma^p_k$ and $\Pi^p_k$, respectively, deciding (i) whether $\Pi$ has
some rational world view is in $\Sigma^p_{k+1}$ and (ii) whether a formula
$F$ is true in some rational world view of $\Pi$ is in $\Sigma^p_{k+2}$.
\end{restatable}

For normal epistemic programs, $\Sigma^p_3$-
resp.\ $\Sigma^p_4$-hardness of rational world views can be
established by adapting the proofs for EFLP world views in
\cite{ShenEiter16}, with a sharpening to atomic heads. 

\begin{restatable}{theorem}{RthEWJexistencenormalatomic}
\label{th-EWJ-existence-normal-atomic}
Given a normal epistemic program $\Pi$ with atomic heads, deciding whether $\Pi$ has
some rational world view is  $\Sigma^p_3$-complete. 
\end{restatable}

\begin{restatable}{theorem}{RthEWJinference}
\label{th-EWJ-inference}
Given a normal epistemic program $\Pi$ with atomic heads and a
formula $F$, deciding whether $F$ is true in some rational world view
of $\Pi$ is  $\Sigma^p_4$-complete. 
The $\Sigma^p_4$-hardness holds even if $F$ is an atom.
\end{restatable}

For simple disjunctive epistemic programs, we obtain the same
complexity results as for normal epistemic programs. Intuitively, 
disjunction $\mid$ of atoms in rule heads
compensates for complex formulas in
rule bodies of normal epistemic programs, which make inference of rule bodies in well-justified
computation harder, as reflected by the complexity results for
rational answer sets in row~1 of Table~\ref{tab:complexity2}.
The $\Sigma^p_3$ resp.\ $\Pi^p_4$-hardness results are shown by
alterations of the hardness proofs for EFLP world views of general normal programs in
\cite{ShenEiter16}[Theorems~5 and~6].

\begin{restatable}{theorem}{RthEWJsimpledisjunctiveexistence}
\label{th-EWJ-simple-disjunctive-existence}
Given a simple disjunctive epistemic program $\Pi$,
deciding whether $\Pi$ has some rational world view is $\Sigma^p_3$-complete.
\end{restatable}

\begin{restatable}{theorem}{RthEWJsimpledisjunctiveinference}
\label{th-EWJ-simple-disjunctive-inference}
Given a simple disjunctive epistemic program $\Pi$ and a formula 
$F$, deciding whether $F$ is true in some rational world view of $\Pi$
is  $\Sigma^p_4$-complete.
The $\Sigma^p_4$-hardness holds even if $F$ is an atom.
\end{restatable}

When complex formulas in rule bodies of disjunctive epistemic programs
are allowed, the complexity reaches the full complexity of epistemic
programs under rational world views, which is at the fourth resp.\ fifth
level of the Polynomial Hierarchy (PH). The hardness results are
established by lifting reductions for brave resp.\ cautious inference from rational
answer sets, which is $\Sigma^p_3$- resp.\ $\Pi^p_3$-complete, to 
the rational world views problems (E1) and (E2) using techniques in the proofs of Theorems~5
and~6 in \cite{ShenEiter16}, respectively.

\begin{restatable}{theorem}{RthEWJdisjunctiveexistence}
\label{th-EWJ-disjunctive-existence}
Given an epistemic program $\Pi$,
deciding whether $\Pi$ has some rational world view is
$\Sigma^p_4$-complete, and $\Sigma^p_4$-hardness holds even for
atomic heads.
\end{restatable}

\begin{restatable}{theorem}{RthEWJdisjunctiveatomicinference}
\label{th-EWJ-disjunctive-atomic-inference}
Given an epistemic program $\Pi$ and a formula $F$,
deciding whether $F$ is true in some rational world view of $\Pi$ is
$\Sigma^p_5$-complete, and $\Sigma^p_5$-hardness holds even for
atomic heads and $F$ being an atom.
\end{restatable}

In conclusion, the results show that rational answer sets and rational
world views cover a wide initial range of PH up to its fifth
level. This suggests that rational semantics as such has rich capacity as a
backend formalism for problem solving inside polynomial space.
We remark that few propositional rule-based formalisms have inference
complexity at the fourth resp.\ fifth level of PH, unless they are
recursively nested, e.g. \cite{DBLP:conf/lpnmr/EiterGV97}, or of
parametric nature with expressive atoms,
e.g.\ \cite{DBLP:conf/lpnmr/AmendolaCRT22,DBLP:conf/lpnmr/PolletR97}.

Notably, the complexity of inference from rational world views drops
by one level of PH, if condition~2 (maximality) of world views is
abandoned; several world view semantics do not have such a condition
(cf.\ Section~\ref{sec:ASP-principles-wvs})
, and then have comparable complexity. Furthermore,
many answer set semantics for general propositional programs have
typically complexity $\Sigma^p_2$ for problems (R1) and (R2),
respectively $\Pi^p_2$ for problem (R3), in contrast to $\Sigma^p_3$
resp.\ $\Pi^p_3$-completeness of rational semantics in the general
case. Intuitively, this is explained by a lack of minimal model
selection across multiple reducts that result from different
alternatives in rule application. For scenarios that require 
$\Sigma^p_3$ resp.\ $\Pi^p_3$ expressiveness, the lack may be
compensated by ad-hoc selection on top of answer sets, though.

\section{Related Work}
\label{RelatedW}
In this paper, we study principles for ASP semantics
in general and answer set and world view
construction in particular.
The GAS principles \cite{Gelfond2008}
were the first guide for the construction of an answer set
in general cases,
which requires that an answer set must satisfy all rules of 
a logic program and should adhere to the rationality principle.
The Gelfond's rationality principle 
is an abstract notion and can be interpreted in different ways,
thus leaving rooms for different ASP semantics.
In this section, we mention some closely related work. 

\subsection{Properties of ASP semantics}

Early works focused on studying major properties
of the GL-semantics for simple disjunctive programs, including 
the minimal model property \cite{GL91},
constraint monotonicity \cite{LTT99} and
foundedness \cite{LeoneRS97}.
These properties could be viewed as a kind of interpretation
of the Gelfond's rationality principle, i.e., 
every answer set of a simple disjunctive program $\Pi$
should be a minimal model of $\Pi$, contain no unfounded set
and satisfy constraint monotonicity.

Recent works 
\cite{iclp-KahlL18,CabalarFC20aij}   
extended the notions of constraint monotonicity and foundedness
to world view semantics for epistemic specifications.
\citeA{lpnmr-CabalarFC19} also defined the notion of epistemic splitting, which is much stronger than constraint monotonicity, intending
to establish a minimal requirement
for world view semantics. 
\citeA{abs-2108-07669} provided a survey  
and compared the existing world view semantics 
based on how they comply with
these properties.
It turns out that most of the existing
world view semantics such as those defined in 
\cite{KahlGelfond2015,Fa2015,ShenEiter16,SuCH20} 
violate these properties.

The question whether constraint monotonicity, 
epistemic splitting and foundedness
should be regarded
as mandatory requirements on
every ASP semantics 
has incurred a debate in recent ASP research.
\citeA{shen-eiterCns20} observed that these
properties may be too strong in general and they used some examples 
to demonstrate the observation. 
\citeA{SuICLP2021} said that the soundness of the three 
properties is still under debate.
\citeA{Costantini21} and \citeA{Su2024} proposed
alternative yet less restrictive versions of
the epistemic splitting property,
and \citeA{lpnmrCostantiniF22} further studied these properties. 

In this paper,
we use more examples
to demonstrate that requiring minimal models,
constraint monotonicity and
foundedness as mandatory conditions may  
exclude expected answer sets 
for some simple disjunctive programs 
as well as  world views
for some epistemic specifications.
Some of these examples appeared in \cite{ShenEiter22,shenEiter2025arxiv}.

\subsection{ASP rationality principles}

We propose to refine
the Gelfond's rationality principle
to well-supportedness,
minimality of answer sets w.r.t.\ negation by default
and minimality of world views w.r.t.\ epistemic negation,
and consider them
as three alternative principles for definition
and construction of answer sets and world views.
This leads to the refined GAS principles.

\paragraph{Well-supportedness.}
The notion of well-supportedness
was first introduced by \citeA{Fages:JMLCS:1994} 
as a key property to characterize
answer sets of simple normal programs
under the GL$_{nlp}$-semantics,
where the well-supportedness of an interpretation
$I$ of a program $\Pi$
is defined by establishing a strict well-founded
partial order $\prec$ on $I$ w.r.t.\ $\Pi$.
This notion was extended  by \citeA{ShenWEFRKD14}
to the class of epistemic-free normal programs, 
where the well-supportedness of 
$I$ is defined by constructing 
a level mapping on $I$ w.r.t.\ $\Pi$.
In \cite{CabalarFC0V17}, 
well-supportedness was defined for the class of
epistemic-free normal programs with atomic
rule heads, where the base logic used to define
satisfaction of a rule body is parameterized 
with different intermediate logics, from
intuitionistic logic (such as equilibrium logic) to classical logic.
When the basic logic is classical logic,
their definition of well-supportedness
(see Proposition 1 and Definition 2 in 
\cite{CabalarFC0V17})
corresponds to that in \cite{ShenWEFRKD14}
and is similar to our Definition~\ref{def-well-supported-normal-atomhead} in this paper.

In this paper, we have extended Fages's
well-supportedness to answer sets 
of all epistemic-free programs
and world views of all epistemic programs,
where the well-supportedness of a model $I$
is defined in terms of a strict well-founded
partial order $\prec$ on $I$,
and the base logic language 
${\cal L}_\Sigma$ that defines
satisfaction and entailment of formulas 
in rule bodies and heads is in classical logic.
We have also presented Algorithms \ref{algo:well-support-with-partialorder} and \ref{algo:well-support-with-partialorder-2},
which determine the well-supportedness
of a model $I$ for an epistemic-free 
normal program
and construct a strict well-founded
partial order $\prec$ on $I$ when it is well-supported.

\paragraph{Minimality.}
Minimization of answer sets w.r.t.\ negation by default
looks like the closed 
world assumption \cite{Reiter77}
and the minimal model semantics \cite{minker82},
but it differs essentially from them
in that it does not require that answer sets
must be minimal models.
Minimization of world views w.r.t.\ epistemic negation
was first introduced in \cite{ShenEiter16} and used to
resolve the problem of unintended world views 
due to recursion through the modal operator 
{\bf M}\ \cite{Gelfond2011,Kahl14}. 

The intuitive behavior of the epistemic 
program $\Pi = \{p\leftarrow {\bf M} p\}$
has been an open debate
\cite{Gelfond2011,KahlGelfond2015,Fa2015,ShenEiter16,ZhangZ17a,SuCH20,CabalarFC20aij}.
In this paper, we  justify
in view of the refined GAS principles
that ${\cal A} = \{\{p\}\}$
is the only world view  
of this program,
which satisfies the well-supportedness and 
minimality principles both at the level of
answer sets and of world views.

\subsection{Comparison to other ASP semantics}
Guided by the refined GAS principles, we have defined 
the rational answer set semantics
for epistemic-free programs and the
rational world view semantics
for epistemic programs. 
The former fulfills at the 
answer set level well-supportedness and  
minimality w.r.t.\ negation by default,
and the latter additionally fulfills
at the world view level minimality w.r.t.\ epistemic negation.

We propose to use the refined GAS principles  
as an alternative baseline (in contrast to using the 
minimal model property, constraint monotonicity 
and foundedness as a baseline) to intuitively assess
the existing ASP semantics.
Specifically, the extent to which an existing ASP semantics
adheres to the refined GAS principles
is determined by showing whether it satisfies 
or embodies the three general principles (RP1)-(RP3).

\paragraph{Answer set semantics for epistemic-free programs.}
The GL$_{nlp}$-semantics \cite{GL88} 
is the first and up to now the only answer set semantics
for simple normal programs.
We showed that this semantics agrees with 
the rational answer set semantics
(see Corollary~\ref{cor-rational-simple-normal-pro}),
i.e., its answer sets are exactly the minimal 
well-supported models,
and thus it embodies the refined GAS principles.
The WJ-semantics \cite{ShenWEFRKD14} extends
the GL$_{nlp}$-semantics to epistemic-free normal programs.
We showed that this semantics agrees with the 
rational answer set semantics
(see Corollary~\ref{cor-rational-epistemic-free-normal-pro})
and thus it embodies the refined GAS principles.
The GL-semantics \cite{GL91} is the first 
answer set semantics
for simple disjunctive programs. We showed that
it satisfies the refined GAS principles 
in that all of its answer sets 
are minimal well-supported models.
However, it does not embody the refined GAS principles 
because some minimal well-supported models 
may not be answer sets under the GL-semantics
(see Example~\ref{ex:rationalanswer-minwellsupp}).
We also showed that the three-valued fixpoint semantics
for epistemic-free normal programs with atomic rule heads
\cite{DPB01,PDB07} 
satisfies the refined GAS principles,
but it does not embody them
(see Example~\ref{eg-three-valued fixpoint}).
For epistemic-free programs in general,
we showed that the DI-semantics \cite{ShenE19}, 
the FLP-semantics \cite{FaberPL11} and
the equilibrium logic-based semantics
\cite{Ferr05,FerrarisLL11,Pearce96,Pearce06}
neither embody nor satisfy the refined GAS principles.
Specifically,
the DI-semantics does not agree
with the rational answer set semantics
because the disjunctive rule head selection function
${sel_v}$ used by the DI-semantics
requires selecting the same rule head formula
from identical or variant rule heads,
so that some expected answer sets 
may be excluded and some
unexpected ones may be obtained
(see Example~\ref{eg-sel-selv});
the FLP-semantics does not agree
with the rational answer set semantics
because some answer sets under 
the FLP-semantics are not well-supported
(see Example~\ref{eg-FLP}); and 
the equilibrium logic-based semantics does not agree
with the rational answer set semantics
because (i) it uses equilibrium logic \cite{Pearce96}
to define satisfaction of formulas in rule bodies and heads,
while the rational answer set semantics uses classical logic,
and (ii) it interprets the disjunctive 
rule head operator $\mid$ differently
(see Example~\ref{eg-subsumption}).
 
\paragraph{World view semantics for epistemic programs.}
The G91-semantics \cite{Gelfond91}
is the first world view semantics
for epistemic specifications.
To address the problems of unintended world views
due to recursion through {\bf K} and {\bf M},
\citeA{Gelfond2011} updated the G91-semantics 
to the G11-semantics, and
Kahl et al.~(\citeyear{Kahl14,KahlGelfond2015}) 
further refined the G11-semantics to 
the K14-semantics by appealing to
nested expressions \cite{LTT99}.
These three world view semantics 
do not satisfy the refined GAS principles; 
specifically they violate the minimality principle of 
world views w.r.t.\ epistemic negation
(see Example~\ref{Kahl-Eg29}).
We showed that the SE16-generic semantics \cite{ShenEiter16}
with a base answer set semantics $\cal X$ 
embodies/satisfies the refined GAS principles
provided that $\cal X$ embodies/satisfies 
the refined GAS principles
(see Theorem~\ref{th-SE16-DI-WJrational-semantics}).
For example, the SE16-GL$_{nlp}$-semantics and SE16-WJ-semantics
agree with the rational world view semantics and thus
embody the refined GAS principles, and
the SE16-GL-semantics does not embody 
the refined GAS principles, but it satisfies them.
The rewriting based CF22- and the CF22+SE16-semantics
\cite{DBLP:conf/iclp/Costantini022}, where the latter bears similarities to
the SE16-GL-semantics,  
do not satisfy the refined GAS principles.
The founded autoepistemic equilibrium logic 
(FAEL) \cite{CabalarFC20aij}
and the autoepistemic equilibrium logic (AEL) \cite{SuCH20}
extended the equilibrium logic-based answer set semantics
to equilibrium logic-based world view semantics.
As the equilibrium logic-based answer set semantics
does not satisfy the refined GAS principles,
FAEL and AEL do not satisfy these principles either.

\section{Conclusion}
\label{conclusion}

Different answer set and world view semantics
have been proposed in the literature.
In general, it does not seem possible to formally prove 
whether an ASP semantics defines answer sets/world views
that correspond exactly to the solutions
of  any problem represented by a logic program.
Therefore, it is necessary to develop some general principles 
and use them as a baseline  to intuitively compare and assess
different ASP semantics.

Towards such a baseline,
in this paper we have addressed two important questions 
that we raised in the Introduction, viz.\ (Q1) whether 
in general, minimal models,  
constraint monotonicity and foundedness should be mandatory conditions 
for an ASP semantics, and (Q2) what other properties could be considered 
as alternative principles for an ASP semantics. 

It seems that the three properties in Q1
are sometimes too strong. We  used examples to demonstrate that requiring
minimal models, constraint monotonicity
and foundedness as mandatory conditions 
may exclude expected answer sets 
for some simple disjunctive programs and
world views for some epistemic specifications. 

To explore alternative principles for ASP semantics, 
we appealed to the Gelfond answer set (GAS)
principles for the construction of an answer set \cite{Gelfond2008}.
Specifically we evolved the GAS principles by refining the
Gelfond's rationality principle to the principle of
well-supportedness, the principle of minimality of answer sets
w.r.t.\ negation by default and the principle of minimality of world
views w.r.t.\ epistemic negation, and proposed to consider these three
principles as alternative principles for definition and construction
of answer sets and world views.  This leads to the refined
GAS principles and thus offers an answer to question Q2.

To fulfill the refined GAS principles,
we have extended Fages's well-supportedness 
\cite{Fages:JMLCS:1994} for answer sets of simple normal programs
substantially to answer sets of all epistemic-free programs
and world views of all epistemic programs.

We have then defined ASP semantics
in terms of the refined GAS principles, viz.
the rational answer set semantics
for epistemic-free programs and the
rational world view semantics
for epistemic programs. 
The former fulfills at the 
answer set level well-supportedness and  
minimality w.r.t.\ negation by default,
and the latter additionally fulfills
at the world view level minimality w.r.t.\ epistemic negation.
Furthermore, we have formally defined answer set semantics for
epistemic-free programs extended with choice constructs.

We have used the refined GAS principles
as an alternative baseline (in contrast to 
minimal models, constraint monotonicity 
and foundedness as a baseline) 
to intuitively assess the existing ASP semantics.
Specifically, the extent to which an existing ASP semantics
adheres to the refined GAS principles
is determined by showing whether it satisfies or 
embodies the three general principles.

Finally, we have analyzed the computational complexity of 
well-supportedness and the rational answer set and world view semantics.

The proposed rational answer set semantics
for epistemic-free programs and the
rational world view semantics
for epistemic programs could be viewed 
as an amendment of the GL-semantics \cite{GL91}
for simple disjunctive programs 
and the G91 semantics \cite{Gelfond91} for epistemic 
specifications, respectively, making them embody
the refined GAS principles. 
As every answer set of a simple disjunctive program 
under the GL-semantics  
is also an answer set under the rational answer set semantics,
existing applications of the GL-semantics can be relaxed to 
the rational answer set semantics simply by
checking whether they have additional answer sets
under the rational answer set semantics.
In particular, if a simple disjunctive program has 
no answer set under the GL-semantics, rational
answer sets can be used as a relaxation that, differently from
paraconsistent resp.\ paracoherent answer sets as in
\cite{DBLP:journals/ai/AmendolaEFLM16}, does not need to inject
unfounded (non-provable) positive assumptions.

For future work, methods for efficiently implementing  
the proposed rational answer set and world view semantics 
for different classes of logic programs (see Table~\ref{tab:complexity2}) 
present an important open issue. 
Moreover, as discussed with Michael Gelfond, 
it is open to see some killer applications to test 
the refined GAS principles along with the rational ASP semantics.

\section*{Acknowledgments}
We particularly thank Michael Gelfond for his encouraging and constructive
discussion and comments on this work. 
He mentioned the pin-touching-colored board problem,
which we have borrowed in Example~\ref{pin-dropping}.
\bibliographystyle{theapa}
\bibliography{Epistemic-Logic}

\appendix

\renewcommand\thelemma{\thesection.\arabic{lemma}}
\renewcommand\theproposition{\thesection.\arabic{proposition}}
\renewcommand\thecorollary{\thesection.\arabic{corollary}}
\renewcommand\thedefinition{\thesection.\arabic{define}}
\renewcommand\thetheorem{\thesection.\arabic{theorem}}
\setcounter{theorem}{0}
\setcounter{lemma}{0}
\setcounter{proposition}{0}

\section{Proofs of Theorems and Lemmas}
\label{app:proofs}  

\medskip \noindent $\hookrightarrow$\  {Theorem~\ref{th-simp-normal}}
\RthSimpNormal*
\begin{proof}
$(\Longrightarrow)$ Assume that a model $I$ of 
a simple normal program $\Pi$ is well-supported by
Definition~\ref{def-well-supported-simp-normal}.
That is, there exists a strict well-founded
partial order $\prec$ on atoms in $I$ 
such that for every $p\in I$ 
there is a rule $r$ of the form
\begin{tabbing}
\hspace{.3in} $p\leftarrow q_1\wedge ... \wedge q_n\wedge \neg c_1\wedge ... \wedge \neg c_k$ 
\end{tabbing} 
in $\Pi$, where
$I$ satisfies $body(r)$ and $q_i\prec p$
for every $q_i$ in $body(r)$.
Note that every $q_i$ is in $I$ and every $\neg c_j$ is in $\neg I^-$.
Let $S_p=\{q_1, ..., q_n\}$.
Then, $\prec$ is also a strict well-founded
partial order on atoms in $I$ 
such that for every $p\in I$
there is a rule $r$ as above in $\Pi$ and
$S_p\subset I$, where 
$S_p\cup \neg I^- \models body(r)$
and $q_i\prec p$
for every $q_i\in S_p$. 
Hence,
$I$ is well-supported by 
Definition~\ref{def-well-supported-normal-atomhead}.

$(\Longleftarrow)$ Conversely, assume that a model $I$ of 
a simple normal program $\Pi$ is well-supported by
Definition~\ref{def-well-supported-normal-atomhead}.
That is, there exists a strict well-founded
partial order $\prec$ on atoms in $I$ 
such that for every $p\in I$
there is a rule $r$ as above in $\Pi$ and
some $S\subset I$, where 
$S\cup \neg I^- \models body(r)$
and $q\prec p$
for every $q\in S$.
As $S\subset I$ and
$S\cup \neg I^- \models body(r)$,
we have (1) $I$ satisfies $body(r)$, and 
(2) all $q_i$s in $body(r)$ are in $S$ and thus in $I$.
Then, $\prec$ is also a strict well-founded
partial order on atoms in $I$ 
such that for every $p\in I$ 
there is a rule $r$ as above in $\Pi$,
where
$I$ satisfies $body(r)$ and $q_i\prec p$
for every $q_i$ in $body(r)$.
Hence,
$I$ is well-supported by 
Definition~\ref{def-well-supported-simp-normal}.
\end{proof}

\noindent $\hookrightarrow$\  {Proof of Theorem~\ref{th-wellsupport-lfp}}
\RthWellsupportLfp*
\begin{proof}
Let $\Pi$ be an epistemic-free normal program 
with atomic rule heads, and let $I$
be a model of $\Pi$. If $I=\emptyset$, then there is no
rule $p \leftarrow body(r)$ in $\Pi$, where $p$ is an atom
and $I$ satisfies $body(r)$. In this case, 
$\mathit{lfp}(T_{\Pi}(\emptyset,\neg I^-)) = \emptyset$.
Therefore, next we assume $I\neq \emptyset$.

$(\Longrightarrow)$ Assume that a model $I$ is well-supported 
in $\Pi$ by Definition~\ref{def-well-supported-normal-atomhead}.
That is, there exists a strict well-founded
partial order $\prec$ on atoms in $I$ 
such that for every $p\in I$
there is a rule $p \leftarrow body(r)$ in $\Pi$ and
some $S\subset I$, where 
$S\cup \neg I^- \models body(r)$
and $q\prec p$ for every $q\in S$.

We first prove $I\subseteq \mathit{lfp}(T_{\Pi}(\emptyset,\neg I^-))$.
As $\prec$ is well-founded,
every chain $p\succ q\succ \cdots$ 
starting from any $p\in I$ is finite.
Let the length of the longest chains be $n$.
Then we can make a partition
$I=S_1\cup \cdots \cup S_n$,
where for every $p\in I$, $p$ is in $S_i$
($1\leq i\leq n$)
if the length of the longest chains 
starting from $p$ is $i$. 
As $\prec$ is a strict partial order,
for any $i\neq j$, $S_i\cap S_j=\emptyset$.
We now prove by induction that for every $1\leq i\leq n$,
$S_i\subseteq \mathit{lfp}(T_{\Pi}(\emptyset,\neg I^-))$.
When $i=1$, for every $p\in S_1$ there is no $q\in I$
with $q\prec p$. So there is a rule $p \leftarrow body(r)$ 
in $\Pi$ such that $\neg I^- \models body(r)$.
This means that for every $p\in S_1$,
$p$ is in $T_{\Pi}^1(\emptyset,\neg I^-)$
and thus is in $\mathit{lfp}(T_{\Pi}(\emptyset,\neg I^-))$.
For induction, assume that for $1\leq i\leq k< n$,
every $p\in S_i$ is in $\mathit{lfp}(T_{\Pi}(\emptyset,\neg I^-))$,
i.e., $(S_1\cup \cdots \cup S_k)\subseteq \mathit{lfp}(T_{\Pi}(\emptyset,\neg I^-))$.
For every $p\in S_{k+1}$, by the given condition
there is a rule $p \leftarrow body(r)$ in $\Pi$ and
some $S\subset I$, where 
$S\cup \neg I^- \models body(r)$
and $q\prec p$ for every $q\in S$.
As every $q\in I$ with $q\prec p$ is in 
$S_1\cup \cdots \cup S_k$, 
$S\subseteq (S_1\cup \cdots \cup S_k)$.
Then by the induction hypothesis,
$S\subseteq \mathit{lfp}(T_{\Pi}(\emptyset,\neg I^-))$.
As $S\cup \neg I^- \models body(r)$,
$\mathit{lfp}(T_{\Pi}(\emptyset,\neg I^-))\cup \neg I^- \models body(r)$.
By the rule $p \leftarrow body(r)$, 
$p$ is in $\mathit{lfp}(T_{\Pi}(\emptyset,\neg I^-))$.
Therefore,  all atoms in $I$ are in the least fixpoint 
$\mathit{lfp}(T_{\Pi}(\emptyset,\neg I^-))$,
i.e., $I\subseteq \mathit{lfp}(T_{\Pi}(\emptyset,\neg I^-))$.

We next prove by induction 
$\mathit{lfp}(T_{\Pi}(\emptyset,\neg I^-))\subseteq I$,
i.e., for every $i\geq 0$, $T_{\Pi}^i(\emptyset,\neg I^-)\subseteq I$.
It trivially holds for $i=0$, where $T_{\Pi}^0(\emptyset,\neg I^-)=\emptyset$.
For induction, assume that for $0\leq i\leq k$, 
$T_{\Pi}^i(\emptyset,\neg I^-)\subseteq I$.
For every $p$ in $T_{\Pi}^{k+1}(\emptyset,\neg I^-)$,
there must be a rule $p \leftarrow body(r)$ in $\Pi$ 
such that $T_{\Pi}^k(\emptyset,\neg I^-)\cup \neg I^- \models body(r)$.
By induction hypothesis,
$T_{\Pi}^k(\emptyset,\neg I^-)\subseteq I$,
so we have $I\cup \neg I^- \models body(r)$ and thus
$I$ satisfies $body(r)$. 
As $I$ is a model of $\Pi$, $p$ is also in $I$.
This shows $T_{\Pi}^{k+1}(\emptyset,\neg I^-)\subseteq I$.
Therefore, we have $\mathit{lfp}(T_{\Pi}(\emptyset,\neg I^-))\subseteq I$. 

As a result, we have $\mathit{lfp}(T_{\Pi}(\emptyset,\neg I^-)) = I$.

$(\Longleftarrow)$ Conversely,  assume $\mathit{lfp}(T_{\Pi}(\emptyset,\neg I^-))=I$.
Then all atoms in $I$ can be iteratively derived
starting from $T_{\Pi}^0(\emptyset,N) = \emptyset$ 
via the sequence $\langle T_{\Pi}^i(\emptyset,\neg I^-)
\rangle_{i=0}^\infty$. That is,
for every $p\in T_{\Pi}^i(\emptyset,\neg I^-)$ with $i>0$,
there must be a rule $p \leftarrow body(r)$ in $\Pi$ 
such that $T_{\Pi}^{i-1}(\emptyset,\neg I^-)\cup \neg I^- \models body(r)$.
Consider a strict 
partial order $\prec$ on atoms in $I$, where for every $p\in I$,
if $p$ is in $T_{\Pi}^i(\emptyset,\neg I^-)\setminus T_{\Pi}^{i-1}(\emptyset,\neg I^-)$
for some $i>0$,
which is derived from a rule $p \leftarrow body(r)$ in $\Pi$ and some
$S\subseteq T_{\Pi}^{i-1}(\emptyset,\neg I^-)$ 
such that $S\cup \neg I^- \models body(r)$,
then we have $q\prec p$ for every $q\in S$.
For every $p\in I$
every chain $p\succ q\succ \cdots$
starting from $p$
is finite and will end at some $s\in I$, where
there is a rule $s \leftarrow body(r)$ in $\Pi$ 
such that $\neg I^- \models body(r)$. 
This means that
there is no infinite decreasing chain
$p\succ q\succ \cdots$.
So $\prec$ is a strict well-founded
partial order on $I$.
Therefore, $I$ is well-supported 
by Definition~\ref{def-well-supported-normal-atomhead}.
\end{proof}

\noindent $\hookrightarrow$\  {Proof of Theorem~\ref{th-wellsupport-ruleheads-lfp}}
\RthWellsupportRuleheadsLfp*
\begin{proof}
Let $\Pi$ be an epistemic-free normal program 
and $I$ be a model of $\Pi$.
Let ${\cal V}$ be a set of rule heads in $\Pi$.

$(\Longrightarrow)$ Assume that ${\cal V}$ is well-supported 
in $\Pi$ w.r.t.\ $I$ by Definition~\ref{def-well-supportedHeads-normal}.
We first show by induction
$\mathit{lfp}(T_{\Pi}(\emptyset,\neg I^-))\subseteq {\cal V}$,
i.e., for every $i\geq 0$ we have 
$T_{\Pi}^i(\emptyset,\neg I^-)\subseteq {\cal V}$.
It trivially holds for $i=0$, 
where $T_{\Pi}^0(\emptyset,\neg I^-)=\emptyset$.
Recall that condition (1) in Definition~\ref{def-well-supportedHeads-normal} 
says that for every rule $H\leftarrow body(r)$ in $\Pi$, 
if ${\cal V}\cup \neg I^- \models body(r)$
then $H$ is in ${\cal V}$. 
For induction, assume that for $0\leq i\leq k$ we have 
$T_{\Pi}^i(\emptyset,\neg I^-)\subseteq {\cal V}$.
For every $H\in T_{\Pi}^{k+1}(\emptyset,\neg I^-)$,
by definition
there must be a rule $H\leftarrow body(r)$ in $\Pi$ such that
$T_{\Pi}^k(\emptyset,\neg I^-)\cup \neg I^- \models body(r)$.
Then by the induction hypothesis,
${\cal V}\cup \neg I^- \models body(r)$. 
By condition (1) of Definition~\ref{def-well-supportedHeads-normal},
$H$ is in ${\cal V}$ and thus 
$T_{\Pi}^{k+1}(\emptyset,\neg I^-)\subseteq {\cal V}$.
Hence, we have $\mathit{lfp}(T_{\Pi}(\emptyset,\neg I^-))\subseteq {\cal V}$.
 
We next show ${\cal V}\subseteq \mathit{lfp}(T_{\Pi}(\emptyset,\neg I^-))$.
Recall that condition (2) in Definition~\ref{def-well-supportedHeads-normal} 
says that there exists a strict well-founded
partial order $\prec_h$ on rule heads in ${\cal V}$ 
such that for every $H\in {\cal V}$
there is a rule $H\leftarrow body(r)$ in $\Pi$ and
some $S\subset \cal V$, where 
$S\cup \neg I^- \models body(r)$
and $F\prec_h H$ for every $F\in S$.
As $\prec_h$ is well-founded,
every chain $H\succ_h G\succ_h \cdots$ 
starting from any $H\in {\cal V}$ is finite.
Let the length of the longest chains be $n$.
Then we can make a partition
${\cal V}=S_1\cup \cdots \cup S_n$,
where for every $H\in {\cal V}$, $H$ is in $S_i$
($1\leq i\leq n$)
if the length of the longest chains 
starting from $H$ is $i$. 
As $\prec_v$ is a strict partial order,
for any $i\neq j$, $S_i\cap S_j=\emptyset$.
We now prove by induction that for every $1\leq i\leq n$,
$S_i\subseteq \mathit{lfp}(T_{\Pi}(\emptyset,\neg I^-))$.
When $i=1$, for every $H\in S_1$ there is no $G\in {\cal V}$
with $G\prec_v H$. So there is a rule $H \leftarrow body(r)$ 
in $\Pi$ such that $\neg I^- \models body(r)$.
This means that for every $H\in S_1$,
$H$ is in $T_{\Pi}^1(\emptyset,\neg I^-)$
and thus is in $\mathit{lfp}(T_{\Pi}(\emptyset,\neg I^-))$.
For induction, assume that for $1\leq i\leq k< n$,
every $H\in S_i$ is in $\mathit{lfp}(T_{\Pi}(\emptyset,\neg I^-))$,
i.e., $(S_1\cup \cdots \cup S_k)\subseteq \mathit{lfp}(T_{\Pi}(\emptyset,\neg I^-))$.
For every $H\in S_{k+1}$, by condition (2)
there is a rule $H \leftarrow body(r)$ in $\Pi$ and
some $S\subset {\cal V}$, where 
$S\cup \neg I^- \models body(r)$
and $G\prec_h H$ for every $G\in S$.
As every $G\in {\cal V}$ with $G\prec_h H$ is in 
$S_1\cup \cdots \cup S_k$, 
$S\subseteq (S_1\cup \cdots \cup S_k)$.
Then by the induction hypothesis,
$S\subseteq \mathit{lfp}(T_{\Pi}(\emptyset,\neg I^-))$.
As $S\cup \neg I^- \models body(r)$,
$\mathit{lfp}(T_{\Pi}(\emptyset,\neg I^-))\cup \neg I^- \models body(r)$.
By the rule $H \leftarrow body(r)$, 
$H$ is in $\mathit{lfp}(T_{\Pi}(\emptyset,\neg I^-))$.
Therefore,  all rule heads in ${\cal V}$ are in the least fixpoint 
$\mathit{lfp}(T_{\Pi}(\emptyset,\neg I^-))$,
i.e., ${\cal V}\subseteq \mathit{lfp}(T_{\Pi}(\emptyset,\neg I^-))$.

As a result, we have $\mathit{lfp}(T_{\Pi}(\emptyset,\neg I^-)) = {\cal V}$.

$(\Longleftarrow)$ Conversely,  
assume $\mathit{lfp}(T_{\Pi}(\emptyset,\neg I^-)) = {\cal V}$.
As $I$ is a model of $\Pi$, it is a model of 
$\mathit{lfp}(T_{\Pi}(\emptyset,\neg I^-))$ 
and thus is a model of ${\cal V}$.
For every $H\in \mathit{lfp}(T_{\Pi}(\emptyset,\neg I^-))$,
there must be a rule $H\leftarrow body(r)$ in $\Pi$ such that 
$\mathit{lfp}(T_{\Pi}(\emptyset,\neg I^-))\cup \neg I^- \models body(r)$.
As $I$ is a model of 
$\mathit{lfp}(T_{\Pi}(\emptyset,\neg I^-))\cup \neg I^-$,
$I$ satisfies $body(r)$. 
This show that for every $H\in {\cal V}$,
there is a rule $H\leftarrow body(r)$ in $\Pi$ such that 
$I$ satisfies $body(r)$. 
Moreover,
for every rule $H\leftarrow body(r)$ in $\Pi$, 
if $\mathit{lfp}(T_{\Pi}(\emptyset,\neg I^-))\cup \neg I^- \models body(r)$,
$H$ must be in $\mathit{lfp}(T_{\Pi}(\emptyset,\neg I^-))$.
This shows that ${\cal V}$ satisfies condition (1) of 
Definition~\ref{def-well-supportedHeads-normal}.

We next show that ${\cal V}$ satisfies 
condition (2) of Definition~\ref{def-well-supportedHeads-normal}.
As  $\mathit{lfp}(T_{\Pi}(\emptyset,\neg I^-)) = {\cal V}$,
all rule heads in ${\cal V}$ can be iteratively derived
starting from $T_{\Pi}^0(\emptyset,N) = \emptyset$ 
via the sequence $\langle T_{\Pi}^i(\emptyset,\neg I^-)
\rangle_{i=0}^\infty$. That is,
for every $H$ in 
$T_{\Pi}^i(\emptyset,\neg I^-)\setminus T_{\Pi}^{i-1}(\emptyset,\neg I^-)$ with $i>0$,
there is a rule $H \leftarrow body(r)$ in $\Pi$ 
such that $T_{\Pi}^{i-1}(\emptyset,\neg I^-)\cup \neg I^- \models body(r)$.
Consider a strict 
partial order $\prec_h$ on rule heads in ${\cal V}$, where for every $H\in {\cal V}$,
if $H$ is in $T_{\Pi}^i(\emptyset,\neg I^-)\setminus T_{\Pi}^{i-1}(\emptyset,\neg I^-)$
for some $i>0$,
which is derived from a rule $H \leftarrow body(r)$ in $\Pi$ and some
$S\subseteq T_{\Pi}^{i-1}(\emptyset,\neg I^-)$ 
such that $S\cup \neg I^- \models body(r)$,
then we have $F\prec_h H$ for every $F\in S$.
Note that for every $H\in {\cal V}$, 
every chain $H\succ_h F\succ_h \cdots$ starting from $H$
is finite and will end at some $G\in {\cal V}$, where
there is a rule $G \leftarrow body(r)$ in $\Pi$ 
such that $\neg I^- \models body(r)$. 
That is,
there is no infinite decreasing chain
$H\succ_h F\succ_h \cdots$.
So $\prec_h$ is a strict well-founded
partial order on ${\cal V}$.
This shows that ${\cal V}$ satisfies condition (2) of Definition~\ref{def-well-supportedHeads-normal}. 

Therefore,  
by Definition~\ref{def-well-supportedHeads-normal},
${\cal V}$ is well-supported in $\Pi$ w.r.t.\ $I$.
\end{proof}

\noindent $\hookrightarrow$\  {Proof of Theorem~\ref{th-wellsupport-lfp-normal}}

\RthWellsupportLfpNormal*
\begin{proof}
Let $\Pi$ be an epistemic-free normal program and $I$
be a model of $\Pi$.

$(\Longrightarrow)$ Assume that $I$ is well-supported in $\Pi$. 
By Definition~\ref{def-well-supported-normal},
the two conditions hold:
(1) there exists a well-supported set $\cal V$ of rule heads
w.r.t.\ $I$;
and (2) for every $p\in I$
there are $k\geq 1$ rules 
$head(r_1)\leftarrow body(r_1), \cdots,
head(r_k)\leftarrow body(r_k)$ in $\Pi$ and  
some $S\subset {\cal V}$, where 
for $1\leq i\leq k$, $S\cup \neg I^- \models body(r_i)$, 
$F\prec_h head(r_i)$
for every $F\in S$,
$S\cup\{head(r_1), \cdots, head(r_k)\}\cup \neg I^- \models p$,
and $q\prec p$
for every $q\in I$ such that 
$S\cup \neg I^- \models q$.
It follows from condition (1) and Theorem~\ref{th-wellsupport-ruleheads-lfp}
that ${\cal V}=\mathit{lfp}(T_{\Pi}(\emptyset,\neg I^-))$.
Moreover, in condition (2)
the rule heads $\{head(r_1), \cdots, head(r_k)\}$
must be in ${\cal V}$ and thus
we have ${\cal V}\cup \neg I^- \models p$.
This shows $\mathit{lfp}(T_{\Pi}(\emptyset,\neg I^-))\cup \neg I^-\models p$
for every $p\in I$.

$(\Longleftarrow)$ Conversely,  assume $\mathit{lfp}(T_{\Pi}(\emptyset,\neg I^-))\cup \neg I^-\models p$ for every $p\in I$. 
Let ${\cal V}=\mathit{lfp}(T_{\Pi}(\emptyset,\neg I^-))$. Then by Theorem~\ref{th-wellsupport-ruleheads-lfp}, ${\cal V}$ is a well-supported set of rule heads
of $\Pi$ w.r.t.\ $I$ with a strict well-founded
partial order $\prec_h$ on ${\cal V}$.
So condition (1) of Definition~\ref{def-well-supported-normal}
is satisfied.

To show condition (2) of Definition~\ref{def-well-supported-normal}, 
consider an atom $p\in I$. Then 
$\mathit{lfp}(T_{\Pi}(\emptyset,\neg I^-))\cup \neg I^-\models p$
and thus ${\cal V}\cup \neg I^-\models p$.
As ${\cal V}$ is well-supported in $\Pi$ w.r.t.\ $I$,
there must exist some $k\geq 1$ rules 
$head(r_1)\leftarrow body(r_1), \cdots,
head(r_k)\leftarrow body(r_k)$ in $\Pi$ and  
some $S\subset {\cal V}$, where 
for $1\leq i\leq k$, 
$S\cup \neg I^- \models body(r_i)$, $F\prec_h head(r_i)$
for every $F\in S$,
$S\cup\{head(r_1), \cdots, head(r_k)\}\cup \neg I^- \models p$ 
and $S\cup \neg I^- \not\models p$.
Now consider a strict 
partial order $\prec$ on $I$, where
we let $q\prec p$
for every $q\in I$ such that 
$S\cup \neg I^- \models q$.
Note that every chain $p\succ q\succ \cdots$
starting from $p$
is finite and will end at some $s\in I$, where
there are some $n\geq 1$ rules  
$head(r^1)\leftarrow body(r^1), \cdots,
head(r^n)\leftarrow body(r^n)$ in $\Pi$ 
such that $\{head(r^1), \cdots, head(r^n)\}\cup \neg I^- \models s$
and $\neg I^- \models body(r^j)$
for $1\leq j\leq n$. 
That is,
there is no infinite decreasing chain
$p\succ q\succ \cdots$ starting from $p$.
So $\prec$ is a strict well-founded
partial order on $I$.
As a result, condition (2) of Definition~\ref{def-well-supported-normal}
is satisfied.
Therefore, $I$ is well-supported in $\Pi$
by Definition~\ref{def-well-supported-normal}.
\end{proof}

\noindent $\hookrightarrow$\  {Proof of Theorem~\ref{th-DI-GLnlprational-semantics}}
\RthDIGLnlprationalSemantics*
\begin{proof}
Recall that a model $I$ of a simple normal program $\Pi$ 
is an answer set of $\Pi$ under the rational answer set semantics
iff $I$ is a minimal well-supported model of $\Pi$
iff $I$ is an answer set of $\Pi$ under the GL$_{nlp}$-semantics
(Corollary~\ref{cor-rational-simple-normal-pro})
iff $I$ is a minimal model of the GL-reduct $\Pi^I$
of $\Pi$ w.r.t.\ $I$.

Let $I$ be an answer set of $\Pi$ under the GL-semantics.
Then $I$ is both a minimal model of $\Pi$
and a minimal model of $\Pi^I$.
Assume towards a contradiction that 
$I$ is not an answer set of $\Pi$ under the
rational answer set semantics;
i.e., for any head selection $\sel$,
$I$ is not an answer set of $\Pi^I_{sel}$ 
under the GL$_{nlp}$-semantics.
Then $I$ is not a minimal model of $(\Pi^I_{sel})^I$.
Let $J\subset I$ be a minimal model of $(\Pi^I_{sel})^I$.
For any rule $r$ 
in $\Pi^I$ of the form
$A_1\mid\cdots\mid A_k\leftarrow B_1\wedge\cdots\wedge B_m$
such that $J$ satisfies $body(r)$ (and thus $I$ satisfies $body(r)$),
there must be a rule  
in $(\Pi^I_{sel})^I$ of the form
$A_i\leftarrow B_1\wedge\cdots\wedge B_m$,
where $sel(A_1\mid\cdots\mid A_k, I) = A_i$,
such that $J$ satisfies $A_i$ and thus satisfies $head(r)$.
This means that $J$ is a model of $\Pi^I$,
contradicting the assumption that $I$ is a minimal model 
of $\Pi^I$. Hence for some head selection function $\sel$,
$I$ is an answer set of $\Pi^I_{sel}$ 
under the GL$_{nlp}$-semantics.
That is, $I$ is an answer set of $\Pi$ 
under the rational answer set semantics.
\end{proof}

\noindent $\hookrightarrow$\  {Proof of Lemma~\ref{lem-se16-semantics-epModel}}

\RlemSESixteenSemanticsEpModel*
\begin{proof}
Let ${\cal A}$ be a candidate world view of 
an epistemic program $\Pi$ w.r.t.\ 
an answer set semantics $\cal X$ 
and some $\Phi \subseteq Ep(\Pi)$.
By Definition~\ref{SE16-generic-semantics}, 
every $\mathbf{not} F\in \Phi$ is true in $\cal A$ and
every $\mathbf{not} F\in Ep(\Pi)\setminus \Phi$ is false in $\cal A$.
This means that for every $\mathbf{not} F\in Ep(\Pi)\setminus \Phi$,
$F$ is true and $\neg F$ is false in every $I\in {\cal A}$.
As ${\cal A} \neq \emptyset$
and every $I\in {\cal A}$ is an answer set of $\Pi^\Phi$ under $\cal X$,
by the Gelfond answer set principles
every $I\in {\cal A}$ is a model of $\Pi^\Phi$.
Let $r$ be a rule in $\Pi$ and $r'$ be the rule in $\Pi^\Phi$
that is obtained from 
$r$ with every $\mathbf{not} F\in \Phi$
replaced by $\top$ and 
every $\mathbf{not} F\in Ep(\Pi)\setminus \Phi$
replaced by $\neg F$.
As every $I\in {\cal A}$ is a model of $r'$,
by Definition~\ref{satisfaction-epis-program}
every $I\in {\cal A}$ is a model of $r$ w.r.t.\ ${\cal A}$.
So every $I\in {\cal A}$ is a model of $\Pi$ w.r.t.\ ${\cal A}$
and thus ${\cal A}$ is an epistemic model of $\Pi$.
\end{proof}

\noindent $\hookrightarrow$\  {Proof of Lemma~\ref{lem-chara-se16-semantics}}

\RlemCharaSESixteenSemantic*
\begin{proof}
Let $\Pi$ be an epistemic program
and $\cal A$ be an epistemic model of $\Pi$.
Then by Definition~\ref{satisfaction-epis-program},
$\cal A\neq \emptyset$.

$(\Longrightarrow)$ Assume that $\cal A$ is a world view of $\Pi$ 
w.r.t.\ $\Phi$ and $\cal X$
under the SE16-generic semantics.
We first show that $\cal A$ coincides with the collection of 
all answer sets of $\Pi^{\cal A}$ under ${\cal X}$.
By Definition~\ref{SE16-generic-semantics}, 
${\cal A}$ is the collection of all 
answer sets of $\Pi^\Phi$ under $\cal X$,
every $\mathbf{not} F\in \Phi$ is true in $\cal A$ and
every $\mathbf{not} F\in Ep(\Pi)\setminus \Phi$ is false in $\cal A$.
Then by Definition~\ref{reduct}, $\Pi^\Phi=\Pi^{\cal A}$ and thus
${\cal A}$ coincides with the collection of 
all answer sets of $\Pi^{\cal A}$ under ${\cal X}$.

We next prove the condition (2) 
in Lemma~\ref{lem-chara-se16-semantics}.
Assume on the contrary that there exists 
some epistemic model ${\cal W}$ of $\Pi$ 
such that (i) $\cal W$ coincides with the collection of 
all answer sets under ${\cal X}$ of
the epistemic reduct $\Pi^{\cal W}$ of $\Pi$ w.r.t.\ ${\cal W}$
and (ii) $\Phi$
is a proper subset of $\Phi_{\cal W}$,
where $\Phi$ and $\Phi_{\cal W}$ 
are the sets of
all epistemic negations occurring in $\Pi$
that are true in ${\cal A}$ and $\cal W$, respectively.
By Definitions \ref{reduct} and \ref{se16-reduct}, 
$\Pi^{\Phi_{\cal W}}=\Pi^{\cal W}$ and thus
by Definition~\ref{SE16-generic-semantics}, 
${\cal W}$ is a candidate world view of $\Pi$ 
w.r.t.\ $\Phi_{\cal W}$ and $\cal X$.
Then by Definition~\ref{SE16-generic-semantics},
$\cal A$ should not be a world view of $\Pi$ 
w.r.t.\ $\Phi$ and $\cal X$, contradicting our assumption.

As a result, if $\cal A$ is a world view of $\Pi$ 
w.r.t.\ $\Phi$ and $\cal X$
under the SE16-generic semantics, then
(1) $\cal A$ coincides with the collection of 
all answer sets of $\Pi^{\cal A}$,
and (2) there exists no other epistemic model ${\cal W}$ 
satisfying condition (1)
such that $\Phi$
is a proper subset of $\Phi_{\cal W}$.

$(\Longleftarrow)$ Conversely,  
assume that (1) $\cal A$ coincides with the collection of 
all answer sets of $\Pi^{\cal A}$ under ${\cal X}$,
and (2) there exists no other epistemic model ${\cal W}$  
such that (i) $\cal W$ coincides with the collection of 
all answer sets of $\Pi^{\cal W}$ under ${\cal X}$,
and (ii) $\Phi$
is a proper subset of $\Phi_{\cal W}$,
where $\Phi$ and $\Phi_{\cal W}$ 
are the sets of
all epistemic negations occurring in $\Pi$
that are true in ${\cal A}$ and $\cal W$, respectively.
Then every $\mathbf{not} F\in \Phi$ is true in $\cal A$ and
every $\mathbf{not} F\in Ep(\Pi)\setminus \Phi$ is false in $\cal A$.
By Definitions \ref{reduct} and \ref{se16-reduct}, 
$\Pi^{\Phi}=\Pi^{\cal A}$ and thus
$\cal A$ is the collection of 
all answer sets of $\Pi^{\Phi}$ under ${\cal X}$.
By Definition~\ref{SE16-generic-semantics}, 
${\cal A}$ is a candidate world view of $\Pi$ 
w.r.t.\ $\Phi$ and $\cal X$.

Now assume on the contrary that 
${\cal A}$ is not a world view of $\Pi$ 
w.r.t.\ $\Phi$ and $\cal X$, i.e.,
there exists some candidate world view $\cal W$
w.r.t.\ $\cal X$ and some $\Phi_{\cal W} \subseteq Ep(\Pi)$ such that 
$\Phi$ is a proper subset of $\Phi_{\cal W}$.
Then by Definition~\ref{SE16-generic-semantics}, 
${\cal W}$ is the collection of all 
answer sets of $\Pi^{\Phi_{\cal W}}$ under $\cal X$,
every $\mathbf{not} F\in \Phi_{\cal W}$ is true in $\cal W$ and
every $\mathbf{not} F\in Ep(\Pi)\setminus \Phi_{\cal W}$ is false in $\cal W$.
By Definitions \ref{reduct} and \ref{se16-reduct},
$\Pi^{\Phi_{\cal W}}=\Pi^{\cal W}$ and thus
$\cal W$ coincides with the collection of 
all answer sets of $\Pi^{\cal W}$ under ${\cal X}$.
This contradicts our assumption (2) above:
there exists no other epistemic model ${\cal W}$  
such that (i) $\cal W$ coincides with the collection of 
all answer sets of $\Pi^{\cal W}$ under ${\cal X}$,
and (ii) $\Phi$
is a proper subset of $\Phi_{\cal W}$.
Therefore, we conclude that ${\cal A}$ is a world view of $\Pi$ 
w.r.t.\ $\Phi$ and $\cal X$ under the SE16-generic semantics.
\end{proof}

\noindent $\hookrightarrow$\  {Proof of Theorem~\ref{th-SE16-DI-WJrational-semantics}}

\RthSESixteenDIWJrationalSemantics*
\begin{proof}
Let $\Pi$ be an epistemic program,
$\cal A$ be an epistemic model of $\Pi$,
and $\Phi$ be the set of
all epistemic negations occurring in $\Pi$
that are true in ${\cal A}$.

$(\Longrightarrow)$ 
Assume that SE16-X-semantics agrees with 
the rational world view semantics.
That is, an epistemic model is a world view
under the SE16-X-semantics iff 
it is a world view 
under the rational world view semantics.
To prove that $\cal X$ agrees with 
the rational answer set semantics,
we distinguish between two cases. 
\begin{myenumerate}
\item
Assume that $\cal A$ is a world view of $\Pi$
under the rational world view semantics.
Then $\cal A$ is a world view of $\Pi$
under the SE16-X-semantics.
Assume on the contrary that $\cal X$ does not agree with 
the rational answer set semantics, i.e., we have
either (i) $\cal A$ does not coincide with the collection of 
all answer sets under $\cal X$ of
the epistemic reduct $\Pi^{\cal A}$,
or (ii) $\cal A$ coincides with the collection of 
all answer sets under $\cal X$ of $\Pi^{\cal A}$ and
there exists some epistemic model ${\cal W}$ of $\Pi$,
where $\cal W$ coincides with the collection of 
all answer sets under $\cal X$ of $\Pi^{\cal W}$ and $\Phi$ 
is a proper subset of $\Phi_{\cal W}$ 
which is the set of all epistemic negations occurring in $\Pi$
that are true in $\cal W$.
Then, by Lemma~\ref{lem-chara-se16-semantics} 
$\cal A$ is not a world view of $\Pi$ 
w.r.t.\ $\Phi$ and $\cal X$
under the SE16-generic semantics,
contradicting our assumption that
$\cal A$ is a world view of $\Pi$
under the SE16-X-semantics. 
\item
Assume that $\cal A$ is not a world view of $\Pi$
under the rational world view semantics.
Then $\cal A$ is not a world view of $\Pi$
under the SE16-X-semantics.
Assume on the contrary that $\cal X$ does not agree with 
the rational answer set semantics, i.e., we have
(i) $\cal A$ coincides with the collection of 
all answer sets of $\Pi^{\cal A}$ under $\cal X$,
and (ii) there exists no other epistemic model 
${\cal W}$ of $\Pi$ satisfying condition (i),
where $\Phi$ is a proper subset of $\Phi_{\cal W}$ 
which is the set of all epistemic negations occurring in $\Pi$
that are true in $\cal W$.
Then, by Lemma~\ref{lem-chara-se16-semantics} 
$\cal A$ is a world view of $\Pi$ 
w.r.t.\ $\Phi$ and $\cal X$
under the SE16-generic semantics,
contradicting our assumption that
$\cal A$ is not a world view of $\Pi$
under the SE16-X-semantics.
\end{myenumerate}

$(\Longleftarrow)$ Conversely,  assume that 
$\cal X$ agrees with 
the rational answer set semantics.
Then $\cal X$ can be replaced by 
the rational answer set semantics, and vice versa.
Assume on the contrary that the SE16-X-semantics 
does not agree with 
the rational world view semantics, i.e., we have 
the following two cases:
\begin{myenumerate}
\item
$\cal A$ is a world view of $\Pi$
under the rational world view semantics,
but not under the SE16-X-semantics.
Then,  as $\cal X$ agrees with 
the rational answer set semantics,
by Lemma~\ref{lem-chara-se16-semantics}, 
either (i) $\cal A$ does not coincide with the collection of 
all answer sets of $\Pi^{\cal A}$ under the 
rational answer set semantics,
or (ii) $\cal A$ coincides with the collection of 
all answer sets under the rational 
answer set semantics of $\Pi^{\cal A}$ and
there exists some epistemic model ${\cal W}$ of $\Pi$,
where $\cal W$ coincides with the collection of 
all answer sets under the rational 
answer set semantics of $\Pi^{\cal W}$
and $\Phi$ 
is a proper subset of $\Phi_{\cal W}$ 
which is the set of all epistemic negations occurring in $\Pi$
that are true in $\cal W$.
This means that $\cal A$ is not a world view of $\Pi$ 
under the rational answer set semantics,
contradicting our assumption. 

\item
$\cal A$ is a world view of $\Pi$
under the SE16-X-semantics, but not
under the rational world view semantics.
Then,  as $\cal X$ agrees with 
the rational answer set semantics,
by Lemma~\ref{lem-chara-se16-semantics}, 
we have (i) $\cal A$ coincides with the collection of 
all answer sets of $\Pi^{\cal A}$ under 
the rational answer set semantics,
and (ii) there exists no other epistemic model 
${\cal W}$ of $\Pi$ satisfying condition (i),
where $\Phi$ is a proper subset of $\Phi_{\cal W}$ 
which is the set of all epistemic negations occurring in $\Pi$
that are true in $\cal W$.
This means that $\cal A$ is a world view of $\Pi$ 
under the rational answer set semantics,
contradicting our assumption. 
\end{myenumerate}
\end{proof}

\noindent $\hookrightarrow$\  {Proof of Theorem~\ref{th-algo:well-support-with-partialorder-2}}

\RthAlgoWellSupportWithPartialorderTwo*
\begin{proof}
Let $\Pi$ consist of a finite number $n\geq 1$ of rules.
Then Algorithm \ref{algo:well-support-with-partialorder-2} 
runs the while cycle (lines 3 - 17) at most $n$ times.
Observe that the most time-consuming part in each cycle
is to compute the entailment of propositional formulas
like $S_i\cup \neg I^- \models body(r)$, which
can be completed in finite time.
Therefore, Algorithm \ref{algo:well-support-with-partialorder-2} 
will terminate eventually.

Starting from $S_0 = \emptyset$ (line 2),
for every $i\geq 0$, $S_{i+1}$ is $S_i$
plus all rule heads $head(r)$ 
such that a rule $head(r)\leftarrow body(r)$
is in $\Pi$ and $S_i\cup \neg I^- \models body(r)$
(lines 4 and 6). 
That is, Algorithm~\ref{algo:well-support-with-partialorder-2}
computes the least fixpoint 
$\mathit{lfp}(T_{\Pi}(\emptyset,\neg I^-))$
via the sequence $\langle S_i
\rangle_{i=0}^\infty$, where for every $i\geq 0$,
$S_i = T_{\Pi}^i(\emptyset,\neg I^-)$.
When $S_{i+1}= S_i$ for some $i\geq 0$ 
(i.e., when flag$=0$ after some 
while loop (line 15)),
we reach the least fixpoint $\delta=S_i$ 
which is $\mathit{lfp}(T_{\Pi}(\emptyset,\neg I^-))$ 
except that all rule heads $head(r')$ that are 
logically equivalent to $head(r)$ are removed at 
line 10 of Algorithm~\ref{algo:well-support-with-partialorder-2}. 

Let $\alpha = 
\mathit{lfp}(T_{\Pi}(\emptyset,\neg I^-)) \setminus \delta$;
then for every $H'\in \alpha$, there is a unique rule head
$H$ in $\delta$ with $H\equiv H'$. 
This means that for any formula $F$, $\delta\models F$
iff $\mathit{lfp}(T_{\Pi}(\emptyset,\neg I^-))\models F$. 
Therefore, if $\delta\cup \neg I^-\models p$ 
for every $p\in I$
(line 18), then 
$\mathit{lfp}(T_{\Pi}(\emptyset,\neg I^-))\cup \neg I^-\models p$ for every $p\in I$
and thus by Theorem~\ref{th-wellsupport-lfp-normal},
$I$ is a well-supported model; in this case, 
a strict well-founded partial order $\prec$ on $I$ as defined in 
${\cal O}$ is returned (line 19).
Otherwise, 
if $\delta\cup \neg I^-\not\models p$ 
for some $p\in I$
(line 21), then 
$\mathit{lfp}(T_{\Pi}(\emptyset,\neg I^-))\cup \neg I^-\not\models p$
and thus by Theorem~\ref{th-wellsupport-lfp-normal},
$I$ is not well-supported in $\Pi$; in this case, 
Algorithm~\ref{algo:well-support-with-partialorder-2} outputs
"NOT WELL-SUPPORTED"  (line 22).

The above argument shows that given a finite epistemic-free normal program $\Pi$ 
and a model $I$ of $\Pi$, Algorithm \ref{algo:well-support-with-partialorder-2} 
correctly determines the well-supportedness of $I$ and
terminates.

Now, assume that $I$ is a well-supported model.
We next show that the strict well-founded partial order $\prec$
defined by ${\cal O}$ 
in Algorithm~\ref{algo:well-support-with-partialorder-2}
satisfies the conditions of 
Definition~\ref{def-well-supported-normal}.
To achieve this, we need to find
a well-supported set $\cal V$ of rule heads in $\Pi$
and construct a strict well-founded partial order $\prec_h$
on $\cal V$ as required by 
Definition~\ref{def-well-supported-normal}.

Consider ${\cal V} = \mathit{lfp}(T_{\Pi}(\emptyset,\neg I^-))$.
Then,  ${\cal V} = \delta\cup \alpha$.
By Theorem~\ref{th-wellsupport-ruleheads-lfp},
$\cal V$ is well-supported in $\Pi$ w.r.t.\ $I$.
Now let us expand the strict well-founded partial order $\prec_h$
defined by ${\cal O}_h$ 
in Algorithm~\ref{algo:well-support-with-partialorder-2} to 
cover all rule heads in $\cal V$. Note that ${\cal O}_h$ 
only covers rule heads in $\delta$.
Let ${\cal O}^{\cal V}_h = {\cal O}_h \cup \{F\prec_h H 
\mid H \in \alpha \text{ and } F\in \delta\}$. Note that for 
every $H \in \alpha$, there must be a rule $H\leftarrow body(r)$ in $\Pi$ 
such that $\delta\cup \neg I^- \models body(r)$.
As $\prec_h$ defined by ${\cal O}_h$
is a strict well-founded partial order on $\delta$,
$\prec_h$ defined by ${\cal O}^{\cal V}_h$
is a strict well-founded partial order on ${\cal V}$.

For every $p\in I$,
as $I$ is a well-supported model in $\Pi$,
$p$ must be inferred from $S_{i+1}\cup \neg I^-$ for some $i\geq 0$ 
in Algorithm~\ref{algo:well-support-with-partialorder-2}
(line 9),
where $S_i\subset {\cal V}$,
$S_{i+1} = S_i\cup \{head(r)\mid r\in R(S_i)\}$,
$R(S_i) = \{r\mid r\in \Pi$ and 
$S_i\cup \neg I^- \models body(r)\}$,
$S_i\cup \neg I^- \not\models p$
and $S_{i+1}\cup \neg I^- \models p$.
In this case, for every rule $r\in R(S_i)$,
$F\prec_h head(r)$ is in ${\cal O}_h$ (line 8)
(and thus in ${\cal O}^{\cal V}_h$)
for every $F\in S_i$,
and $q\prec p$ is in ${\cal O}$ (line 9)
for every $q\in I$ such that $S_i\cup \neg I^- \models q$.
This means that when $I$ is well-supported in $\Pi$,
the strict well-founded partial order $\prec$
defined by ${\cal O}$ 
in Algorithm~\ref{algo:well-support-with-partialorder-2}
satisfies the conditions of 
Definition~\ref{def-well-supported-normal};
i.e., we have (1) a well-supported set 
${\cal V} = \mathit{lfp}(T_{\Pi}(\emptyset,\neg I^-))$
of rule heads
in $\Pi$ w.r.t.\ $I$ with a strict well-founded
partial order $\prec_h$ on ${\cal V}$ as defined by ${\cal O}^{\cal V}_h$,
and (2) a strict well-founded
partial order $\prec$ on $I$ 
as defined by ${\cal O}$, such that for every $p\in I$
we have $k\geq 1$ rules 
$head(r_1)\leftarrow body(r_1), \cdots,
head(r_k)\leftarrow body(r_k)$ in $\Pi$ 
as defined by $R(S_i)$, and  
$S_i\subset {\cal V}$, where 
for $1\leq j\leq k$, $S_i\cup \neg I^- \models body(r_j)$, $F\prec_h head(r_j)$
for every $F\in S_i$,
$S_i\cup\{head(r_1), \cdots, head(r_k)\}\cup \neg I^- \models p$,
and $q\prec p$
for every $q\in I$ such that $S_i\cup \neg I^- \models q$.
This concludes our proof.
\end{proof}

\noindent $\hookrightarrow$\  {Proof of Theorem~\ref{th-comp-wellsupport-simple-normal}}

\Rthcompwellsupportsimplenormal*
\begin{proof}
By Theorem~\ref{th-wellsupport-lfp}, we need to check whether
$\mathit{lfp}(T_{\Pi}(\emptyset,\neg I^-))=I$ holds. Since $\Pi$
is simple, evaluating each stage $T^i_\Pi(\emptyset,\neg I^-)$ of the
fixpoint iteration for $\mathit{lfp}(T_{\Pi'}(\emptyset,\neg I^-))$ is
feasible in polynomial time, and thus computing
$\mathit{lfp}(T_{\Pi'}(\emptyset,\neg I^-))$ and testing $\mathit{lfp}(T_{\Pi'}(\emptyset,\neg I^-))=I$
is feasible in polynomial time.

The \Pol-hardness follows from the \Pol-completeness of checking
whether a model $I$ of a Horn program
(i.e., a simple $\neg$-free normal program) is its (unique) minimal
model \cite{eite-etal-07b}, which amounts to $\mathit{lfp}(T_{\Pi'}(\emptyset,\neg
I^-))=I$.
\end{proof}

\noindent $\hookrightarrow$\  {Proof of Theorem~\ref{th-comp-wellsupport-lfp-normal}}

\Rthcompwellsupportlfpnormal*
\begin{proof}
As for the \coNP\ membership, for
refuting that $I$ is well-supported
in $\Pi$
we can by Lemma~\ref{lem:guess} guess 
(i)
sets $S'_j$,
$j=0,\ldots,i$ and an atom $p\in I$, 
(ii) for each rule head $head(r)$ that we do not
include in $S'_{j+1}$ (assuming it is not already in $S'_j$), 
an interpretation $I_{r,j}$ which satisfies $S'_j \cup \neg I^- \cup
body(r)$, and (iii) an interpretation $I_p$ 
that satisfies $S_i \cup \neg I^- \cup \{ \neg p\}$. The guess is of
polynomial size and can be verified in polynomial time.

As for the \coNP-hardness, we reduce deciding whether 
a propositional formula $E$ on atoms $p_1,\ldots,p_n$
that is satisfied by $I = {p_1,\ldots,p_n}$ satisfies $E$ has no other model; this is easily seen to
be \coNP-complete, and is equivalent to $I$ being well-supported in
$\Pi = \{ p_i \leftarrow E \mid 1 \leq i
\leq n\}$.
\end{proof}

\noindent $\hookrightarrow$\  {Proof of Theorem~\ref{th-comp-wellsupport-disjunctive-simple}}

\Rthcompwellsupportdisjunctivesimple*

\begin{proof}
Membership in \NP{} follows as a guess for a head selection function
$\sel$ such that $I$ is well-supported in $\Pi' = \Pi^I_{\sel}$ can be
checked in polynomial time, as checking that $\sel$ respects
variant rule heads, computing $\Pi' = \Pi^I_{\sel}$ 
and $\mathit{lfp}(T_{\Pi'}(\emptyset,\neg I^-))$, and testing 
$\mathit{lfp}(T_{\Pi'}(\emptyset,\neg I^-))=I$, are all feasible in
polynomial time. 

\NP-hardness is shown by a reduction from SAT. Given a CNF $E =
\bigwedge_{i=1}^m C_i$ on variables $X=\{x_1,\ldots,x_n\}$, let $\bar
X = \{ \bar x_i \mid 1 \leq i \leq n\}$ and $\{c_1,\ldots,c_m\}$ be
sets of fresh variables, and define
the following rules for $\Pi$:
\begin{align*}
  x_i \mid \bar x_i \leftarrow ~~ & \text { for all } x_i \in X\\
  c_j \leftarrow  x_i              & \text { for all } x_i \in C_j, 1 \leq j \leq m\\
  c_j \leftarrow  \bar x_i              & \text { for all }\neg x_i \in C_j, 1 \leq j \leq m\\
  p \leftarrow   c_1\land\cdots\land c_m          & \text { for all } p \in  X \cup \bar{X}. 
\end{align*}
Then $I = X \cup \bar{X} \cup\{c_1,\ldots,c_m\}$ is 
a model of $\Pi$;
furthermore, it is not hard to see that $I$ is well-supported, i.e.,
has a witnessing head selection $\sel$ with associated partial
 order $\cal O$ computed by
Algorithm~\ref{algo:well-support-with-partialorder-2} iff $E$ is satisfied by the
assignment $\sigma$ to $X$ given by $\sigma(x_i) =$ true iff $\sel(  x_i \mid \bar x_i, I) =
x_i$, for all $x_i\in X$.
\end{proof}

\noindent $\hookrightarrow$\  {Proof of Theorem~\ref{th-comp-wellsupport-disjunctive-general}}

\Rthcompwellsupportdisjunctivegeneral*
\begin{proof}
For $\Sigma^p_2$ membership, we can guess a head selection function
$\sel$ and check, with the help of an \NP{} oracle, in polynomial time
whether (i) $\sel$ respects variant rule heads and (ii)
$\mathit{lfp}(T_{\Pi'}(\emptyset,\neg I^-))\cup \neg I^-\models p$
for every $p\in I$, where $\Pi' = \Pi^I_{\sel}$ and
$T_{\Pi'}(\emptyset,\neg I^-)$  can be computed also in
polynomial time with an \NP{} oracle.

The $\Sigma^p_2$-hardness is shown by a reduction from evaluating a QBF
$\exists X\forall Y\,E(X,Y)$ where $E$ is in DNF. 
Let $\bar X = \{ \bar x \mid x \in X\}$ and define the following rules for $\Pi$:
\begin{align*}
  x \mid \bar x \leftarrow ~~~ & \text { for all } x \in X\\
  p \leftarrow E[\neg X/\bar X]  & \text { for all } p \in  X \cup \bar{X} \cup Y, 
\end{align*}
where $E[\neg X/\bar X]$ results from $E$ by replacing each literal
$\neg x$ by $\bar x$, $x\in X$. Then $I = X \cup \bar{X} \cup Y$ is a model of $\Pi$;
furthermore, it is not hard to see that $I$ is well-supported in $\Pi$, i.e.,
has a witnessing head selection $\sel$ with associated partial
order $\cal O$ computed by 
Algorithm~\ref{algo:well-support-with-partialorder-2} iff $\forall
Y\,E(\sigma(X),Y)$ evaluates to true where the
assignment $\sigma$ to $X$ is given by $\sigma(x_i) =$ true iff $\sel(  x_i \mid \bar x_i, I) =
x_i$, for all $x_i\in X$. 
\end{proof}

\noindent $\hookrightarrow$\  {Proof of Theorem~\ref{th-comp-wellsupport-epistemic-normal}}

\Rthcompwellsupportepistemicnormal*
\begin{proof}
As for the membership parts, we note that the epistemic-free program
$\Pi'$ obtained from $\Pi^{\cal A}$ by canceling double negation in
rule bodies, i.e., replacing formulas $\neg \neg A$ with $A$, is such that for every interpretation $I$, $I$ is a model
of $\Pi^{\cal A}$ iff $I$ is a model of $\Pi'$ and moreover $I$ is
well-supported in $\Pi^{\cal A}$ iff $I$ is well-supported in $\Pi'$.
From Lemma~\ref{lem:epi-reduct}, it follows that we can compute $\Pi'$
in polynomial time, which in case (i) is simple and in case (ii) has
only atomic rule heads. We can  thus decide for a single resp.\ for all $I\in \cal A$ whether $I$
is well-supported in $\Pi'$ in case (i) in polynomial time by Theorem~\ref{th-comp-wellsupport-simple-normal}, and in
case (ii) in \coNP{} by Theorem~\ref{th-comp-wellsupport-lfp-normal}. 

The hardness parts are inherited from epistemic-free normal programs,
viz.\ Theorems~\ref{th-comp-wellsupport-simple-normal} and ~\ref{th-comp-wellsupport-lfp-normal}, 
where $\cal A = \{ I \}$ consists of a single model of the input program $\Pi$.
\end{proof}

\noindent $\hookrightarrow$\  {Proof of Theorem~\ref{th-comp-wellsupport-epistemic-general}}

\Rthcompwellsupportepistemicgeneral*

\begin{proof}
The argument is similar to the one in the proof of
Theorem~\ref{th-comp-wellsupport-epistemic-normal}, 
where one appeals instead of
Theorems~\ref{th-comp-wellsupport-simple-normal} and~\ref{th-comp-wellsupport-lfp-normal}
to Theorems~\ref{th-comp-wellsupport-disjunctive-simple}
and~\ref{th-comp-wellsupport-disjunctive-general}, respectively.
\end{proof}

\noindent $\hookrightarrow$\  {Proof of Lemma~\ref{lem:epi-reduct-implicit}}

\Rlemepireductimplicit*
\begin{proof}
For each epistemic negation $\naf F \in EP(\Pi)$, we can decide
whether $\naf F \in \Phi$ with an \NP{} oracle call; thus computing
$\Phi$ and then $\Pi^{\cal A}$ is feasible in \FPNPPar.

The \FPNPPar-hardness is shown by a reduction from computing the
answers to SAT instances $E_1,\ldots,E_n$ on disjoint sets of
variables $X_1,\ldots,X_n$. Let $p_1,\ldots,p_n$ be fresh variables,
and define $E = \bigwedge_{i=1}^n p_i$ and $\Pi = \{
p_i \leftarrow \naf \neg E_i \mid 1 \leq i \leq n\}$.
Then $\naf \neg E_i \in \Phi$ iff $E_i$ is satisfiable iff $\Pi^{\cal
A}$ contains the rule $p_i \leftarrow \top$. Hence the answers to
$E_1,\ldots,E_n$ are easily extracted from $\Pi^{\cal A}$, which shows \FPNPPar-hardness.
\end{proof}

\noindent $\hookrightarrow$\  {Proof of Theorem~\ref{th-comp-wellsupport-epistemic-normal-implicit}}

\ThCompWellsupportEpistemicNormalImplicit*
\begin{proof}
By Lemma~\ref{lem:epi-reduct-implicit}, 
we can compute $\Pi^{\cal A}$ in polynomial time with one round of \NP{}
oracle calls. Exploiting Theorem~\ref{th-comp-wellsupport-lfp-normal}, we can then consult a further \NP{} oracle whether some 
interpretation $I$ exists such that $in(I)$ holds and $I$ is not
well-supported in $\Pi^{\cal A}$. 
Since two rounds of parallel \NP{} oracles can be reduced to a single
round, it follows that deciding well-supportedness of $\cal A$ in
$\Pi$ under implicit representation is in \PNPPar.

The \PNPPar-hardness under the stated restrictions is shown by a
reduction from deciding, given Boolean formulas $E_1,\ldots,E_n$ in
DNF on pairwise disjoint sets $X_1,\ldots,X_n$ of variables, the number of
tautologies among them is even, where without loss of generality,
$E_i$ is a tautology only if $E_{i-1}$ is a tautology, for each $i>1$, cf.\
\cite{DBLP:journals/siamcomp/Wagner90}. 

Assuming that $E_i = \bigvee_{j=1}^{m_i} D_{i,j}$ where $D_{i,j}$ is a
conjunction of literals on $X_i$, we let $F_i$ be the Boolean formula
$$ F_i = \bigwedge_{i=1}^{m_i} (d_{i,j} \leftrightarrow D_{i,j})\land
(p_i \leftrightarrow \bigvee_{j=1}^{m_i} d_{i,j}) \land
\bigwedge_{x\in X_i} (x \equiv \neg \bar x),$$
where $p_i$ and all $d_{i,j}$ and $\bar x$ are fresh variables. Then 
it is easy to see that $E_i$ is a tautology iff $F_i \models p_i$. The
purpose of the $\bar x$ variables, which take the opposite value of
$x$, is to allow for the derivation of $x$ in a well-supported model.

We now let $E= \bigwedge_{i=1}^n F_i \land q$, where $q$
is a further fresh variable, and include in $\Pi$ the following rules:
\begin{align}
  d_{i,j} \leftarrow \ & D_{i,j}  & &\text{ for all } 1\leq i \leq n, 1\leq
  j \leq m_j, & \tag{r1} &
  \\
  p_i   \leftarrow \ & d_{i,j}   & &\text{ for all } 1\leq i \leq n,
  1\leq j \leq m_j, & \tag{r2} & \\
  q  \leftarrow \ & \textstyle \bigwedge_{i=1}^{2j}
  \neg\naf p_i \land \bigwedge_{i=2j+1}^{n}  \naf p_i & &\text{
                  for all } 0\leq 2j \leq n, & \tag{r3} & \\
  x \leftarrow \ & \neg \bar x & &\text{ for all } x \in X_i, 1 \leq i
  \leq n, & \tag{r4} & \\
  \bar x  \leftarrow \ & \neg x & &\text{ for all } x \in X_i, 1 \leq i
  \leq n.  & \tag{r5} &
\end{align}
Consider an arbitrary model $I$ of $E$. We show that $I$ is well-supported in $\Pi'=\Pi^{\cal A}$ iff  the number of
tautologies among $E_1,\ldots,E_n$ is even.
In this model $I$, $\bar x$  takes the
opposite value of $x$, and thus by the rules (r4) $x \in LFP$ where $LFP =
\mathit{lfp}(T_{\Pi'}(\emptyset,\neg I^-))$ for $\Pi' = \Pi^{\cal A}$,
iff $x\in I$ and by the rules (r5) $\bar x \in LFP$
iff $\bar x \in I$. Furthermore, by the rules (r1) $d_{i,j} \in LFP$
iff $I$ satisfies $D_{i,j}$, and by the rules (r2) $p_i \in LFP$ iff $I$ satisfies
$p_i$. Clearly, $I$ is a model of $\Pi'$. For $I$ being well-supported, it thus remains to show that $q
\in LFP$ holds.
The only rule to derive $q$ in $\Pi$ is by rules
(r3). For $\Pi'$, each formula $\naf p_i$ is replaced by $\top$ if $E_i$
is not a tautology, and by $\neg p$ otherwise. Thus, if $2j'$ many of
the $E_i$ formulas are tautologies, the rule 
$$
q \leftarrow \neg\neg p_1\land\cdots\land\neg\neg p_{2j'} \land
\top\land\cdots\land\top
$$ 
is present in $\Pi'$, and thus $q \in LFP$ holds. Otherwise, if $2j'+1$
many of the $E_i$ are tautologies, for constructing $\Pi'$ in the rule (r3) 
$\naf p_{2j'+1}$ will be replaced by $\neg p_{2j'+1}$ if $2j\leq 2j'$
and  $\neg\naf p_{2j'+1}$ by $\neg\top$ ($\equiv \bot$) if $2j > 2j'$. Thus, none of
these rules is applicable, and $q \notin LFP$ holds.
Consequently, $I$ is well-supported in $\Pi'$ iff the number of
tautologies among $E_1,\ldots,E_n$ is even. Since $I$ was arbitrary,
this proves \PNPPar-hardness.
\end{proof}

\noindent $\hookrightarrow$\  {Proof of Theorem~\ref{th-comp-wellsupport-epistemic-general-implicit}}

\Rthcompwellsupportepistemicgeneralimplicit*
\begin{proof}
As for the membership parts, in order to refute that $\cal A$ is
well-supported in $\Pi$, we can guess some interpretation $I$ such
that $in(I)$ returns true and $I$ is not well-supported in
$\Pi'=\Pi^{\cal A}$. By Lemma~\ref{lem:epi-reduct-implicit}, we can
compute $\Pi'$ in polynomial time with an \NP{} oracle, and then use
an \NP{} resp.\ $\Sigma^p_2$ oracle by
Theorems~\ref{th-comp-wellsupport-disjunctive-simple}
and~\ref{th-comp-wellsupport-disjunctive-general}, respectively, where
in $\Pi'$ double negation in rule bodies is canceled, i.e., $\neg\neg A$
is replaced by $A$.

For the hardness parts, we lift the reductions in the proofs of
Theorems~\ref{th-comp-wellsupport-disjunctive-simple}
and~\ref{th-comp-wellsupport-disjunctive-general}, where in fact we do
not need epistemic negation.

As for (i), we lift the encoding of SAT instances $E(x_1,\dots,x_n)$
in Theorem~\ref{th-comp-wellsupport-disjunctive-simple}
to QBF instances $\Gamma =\forall x_1\dots x_k \exists x_{k+1}\cdots
x_n\,E(x_1,\ldots,x_n)$, by simply restricting in $\Pi$ the rules with body
$c_1\land\cdots\land c_m$ to all $p \in A_{k+1} = \{x_{k+1}, \bar
 x_{k+1},\ldots,x_n, \bar x_n\}$, and by defining 
 $${\cal A} = \{ X' \cup \{\bar x_i\mid x_i \notin X', 1\leq i\leq k\} \cup A_{k+1} \cup\{c_1,\ldots,c_m\} \mid X'
\subseteq \{x_1,\ldots, x_k\}\,\}.$$
As easily seen, $\cal A$ is an epistemic model of the modified
program $\Pi$, denoted $\Pi'$, and is represented by the Boolean formula
$$ \bigwedge_{i=1}^k (x_i \leftrightarrow \neg \bar x_i)\land
        \bigwedge_{i=k+1}^n (x_i\land \bar x_i) \land
        c_1\land\cdots\land c_m.
$$
We note that this formula can be easily converted into an OBDD. 
Furthermore, $\cal A$ is well-supported in $\Pi'$ iff $\Gamma$
evaluates to true. This proves $\Pi^p_2$-hardness.

For (ii), we lift the encoding of QBFs $\exists X\forall Y\,E(X,Y)$
where $E$ is in DNF, in the proof of Theorems~\ref{th-comp-wellsupport-disjunctive-general}.
no need of epistemic negation. Similarly as above, we split $X=X_1 \cup
X_2$ into disjoint subsets $X_1$ and $X_2$, and we restrict in $\Pi$ the rules
with body $E[\neg X/\bar X]$ to all  $p \in  X_2 \cup \bar X_2 \cup
Y$, where
for any set $S$ of atoms, $\bar S= \{ \bar p \mid p \in S\}$.

We then define 
 $${\cal A} = \{ X' \cup \bar (X_1\setminus X) \cup X_2 \cup \bar X_2 \cup\{c_1,\ldots,c_m\} \mid X'
\subseteq X_1\,\}.
$$
Clearly, $\cal A$ is an epistemic model of the restricted program
$\Pi'$ and represented by the Boolean formula 
$$ \bigwedge_{x\in X_1}
(x \leftrightarrow \neg \bar x)\land
\bigwedge_{x\in X_2}^n (x\land \bar x) \land \bigwedge_{y\in Y} y
\land c_1\land\cdots\land c_m.
$$
Again, this formula can be easily converted into an OBDD. 
Furthermore, $\cal A$ is well-supported in $\Pi'$ iff the QBF
$\forall X_1\exists X_2\forall Y\,E(X_1,X_2,Y)$
evaluates to true. This proves $\Pi^p_3$-hardness.
\end{proof}

\noindent $\hookrightarrow$\  {Proof of Lemma~\ref{lem-wj-rational}}

\LemWJEqualsRational*
\begin{proof}
If $\Pi$ has no head variants (in particular, no identical disjunctive
heads), then any selection function $\sel(n,head(r),I)$ in
Definition~\ref{head-sel-gdlp} is admissible as a head-selection
function $\sel_v(head(r),I)$ for DI-WJ answer sets in
Definition~\ref{head-select-SE19}, and vice versa; the parameter $n$
can be omitted.  If $\Pi$ is epistemic-free, then by
Corollary~\ref{cor-rational-epistemic-free-normal-pro} the rational
answer sets of $\Pi^I_{\sel}$ coincide with the WJ-answer sets of
$\Pi^I_{\sel}$. Consequently, if $I$ is a rational answer set of $\Pi$
witnessed by $\Pi^I_\sel$, then $I$ must be a DI-WJ-answer set of
$\Pi$, witnessed by $\Pi^I_\sel$: towards a contradiction, if $I$ were
not a DI-WJ-answer set of $\Pi$, then some $I'\subset I$ and
$\sel'(head(r),I')$ must exist such that $I'$ is a WJ-answer set of
$\Pi^{I'}_{\sel'}$. However, then $I'$ is a well-justified model of
$\Pi^{I'}_{\sel'}$, which contradicts that $I$ is a rational answer
set of $\Pi$. The reverse direction -- every DI-WJ-answer set of $\Pi$ is
a rational answer set of $\Pi$-- is shown analogously.
\end{proof}

\noindent $\hookrightarrow$\  {Proof of Lemma~\ref{lem-reduce-rational-no-variant}}

\LemReduceRationalNoVariant*
\begin{proof}
To achieve this, simply add for each disjunctive rule $r$ in
$\Pi$ in the head a fresh atom $p_r$ and add a constraint $\bot\leftarrow p_r$; then, no head
variants do exist. In every model $I$ of the resulting program $\Pi'$,
$p_r$ must be false, i.e., $p_r\notin I$; it is thus easy to see
that $I$ is a rational answer set of $\Pi'$ iff $I$ is a rational
answer set of $\Pi$, provided $\Pi$ is epistemic-free.  From
Definition~\ref{def-ideal-wv-semantics}, it follows then readily for an epistemic
program $\Pi$  that an epistemic model $\mathcal{A}$ is a
world view of $\Pi'$ iff it is a world view of $\Pi$, which proves the
claimed result.
\end{proof}

\noindent $\hookrightarrow$\  {Proof of Theorem~\ref{th-DJ-WJ-atomic-heads-hardness}}

\RtDJWatomicheadshardness*
\begin{proof}
The $\Sigma^p_3$-hardness of brave inference of an atom from the 
DI-WJ-answers of a disjunctive programs $\Pi$ with literals in the heads
was  shown in Theorem~6 in \cite{ShenE19} by a
reduction from MINQASAT. The latter is deciding, given a QBF $\Phi = \forall Y.E(X,Y)$, where $X = \{ X_1,\ldots, X_n\}$ and $Y=\{
Y_1,\ldots, Y_m\}$, and some $A\in X$, whether $A$ is in some $\subseteq$-minimal set $X'\subseteq X$ such that $X'\cup
\neg(X\setminus X') \models E(X,Y)$.

The logic program $\Pi$ constructed from $\Phi$ consists of the rules
\begin{eqnarray}
   X_i \mid \neg X_i & \leftarrow & \phantom{ E(X,Y) } \qquad\quad
   \text{ for all } X_i \in X \label{choose-X} \\   
   Y \land p & \leftarrow & E(X,Y) \label{aaa} \\
    p & \leftarrow & \neg p \label{kill}
\end{eqnarray}
where $p$ is a fresh atom and $Y$ stands for $Y_1\land\cdots\land
Y_m$; then $A$ is a brave consequence of $\Pi$ w.r.t.\ DI-WJ answer
sets iff  $(\Phi,A)$ is a yes-instance of MINQASAT.

The program $\Pi$ is adapted as follows: using a fresh aom $q$, 
we (i) replace $\neg X_i$ in rules (\ref{choose-X}) with $q$, (ii)
rewrite the rule (\ref{aaa}) into rules with atomc heads, 
and (iii) add a rule $q \leftarrow E(X,Y)$.
For any selection $\sel$, the disjunctive reduct ${\Pi'}^I_\sel$
of the resulting program $\Pi'$ with respect to an interpretation $I$ encodes a choice of a subset $X'\subseteq X$ given by $X' = \{ X_i \in
X \mid  X_i \leftarrow~ \in {\Pi'}^I_\sel\}$. For $I$ being a WJ-answer set
of ${\Pi'}^I_\sel$, $p$ must be derived, thus $X'\cup \neg(X\setminus
X') \models E(X,Y)$ must hold, i.e.,  $X'\cup \neg(X\setminus
X') \models \forall Y.E(X,Y)$, and $I$ is  of the form $I = X'\cup
Y\cup \{p,q\}$. The DI-WJ answer sets of $\Pi'$ are the $\subseteq$-minimal such $I$,
and they correspond 1-1 to the $\subseteq$-minimal sets $X'$ such that $X'\cup \neg(X\setminus
X') \models \forall Y.E(X,Y)$. Thus $A$ is true in every DI-WJ answer
set (tantamount, in every rational answer set)
of $\Pi'$ iff $(\forall Y.E(X,Y),A)$ is a yes-instance of MINQASAT,
which proves the result.
\end{proof}

\noindent $\hookrightarrow$\  {Proof of Lemma~\ref{lem-epi-rat-membership}}

\Rlemepiratmembership*

\begin{proof}
We can guess
$\Phi$ for some epistemic model $\mathcal{A}$
that satisfies condition~1 in Defn.~\ref{def-ideal-wv-semantics} of a rational world view for $\Pi$, and
check the guess in polynomial time with a $\Sigma^p_k$ oracle:
we construct
$\Pi'=\Pi^{\mathcal{A}}$ in polynomial time
(Lemma~\ref{lem:epi-reduct}),
and test using the
oracle (i) whether  $\Pi'$ has some rational answer set (simply set $F=\top$), and (ii) for each
$\naf F \in EP(\Pi)$ whether $F$ is false in some, resp.\ true in every, rational answer set of
$\Pi'$  if $\naf F \in \Phi$ resp.\ $\naf F \notin \Phi$. This shows
item (i).

As for item (ii), membership in $\Sigma^p_{k+2}$ holds as we check
condition~2 (maximality) of Defn.~\ref{def-ideal-wv-semantics} for a
guess $\Phi$ that satisfies condition~1 of a rational world view
$\mathcal{A}$ with the help of a $\Sigma^p_{k+1}$ oracle, and then
check using a $\Sigma^p_k$ oracle in polynomial time whether $F$ is true in
$\mathcal{A}$. To this end, we construct $F^\mathcal{A}$ from Defn.\,\ref{satisfaction-epis-program}, by 
replacing each epistemic negation $\naf G$ in $F$ by $\top$ if $\Pi^{\mathcal{A}}$ cautiously
entails $G$ and by $\bot$ otherwise, and then 
check whether $\Pi^{\mathcal{A}}$ cautiously
entails $F^\mathcal{A}$.
\end{proof}

\noindent $\hookrightarrow$\  {Proof of Theorem~\ref{th-EWJ-existence-normal-atomic}}

\RthEWJexistencenormalatomic*
\begin{proof}
Membership in $\Sigma^p_3$ follows from
Lemma~\ref{lem-epi-rat-membership}.(i) and $\Sigma^p_2$ /
$\Pi^p_2$-completeness of brave / cautious inference for normal
epistemic-free programs (row~1 of Table~\ref{tab:complexity2}).

We show  $\Sigma^p_3$-hardness by a reduction from evaluating QBFs of the form 
\begin{equation}
\label{qbf3}
\exists X\forall Y \exists Z.\phi, 
\end{equation}
where $X=X_1,\ldots,X_{n_X}$, $Y=Y_1,\ldots,Y_{n_Y}$, and
$X=Z_1,\ldots,Z_{n_Z}$ are lists (viewed as sets) of distinct atoms, and 
$\phi = \phi(X,Y,Z)$ is a propositional formula over $X\cup Y \cup Z$,
following ideas of a reduction for EFLP-world views in Theorems~5 of
\cite{ShenEiter16}, but modifying it for EWJ-world views (thus
rational world views) and atomic heads in programs.

For each atom $A \in X\cup Y$, we introduce a fresh atom
$\bar{A}$, and we use further fresh atoms $S$, $U$, and $V$.  Let
\begin{eqnarray}
\Pi_X &=& \{  X_i \leftarrow \naf  \bar{X_i};\ \ \bar{X_i}\leftarrow
\naf  X_i \mid X_i\in X\}, \label{x-rules}\\
\Pi_Y &=& \{  Y_i \leftarrow \neg \bar{Y_i};\ \ \bar{Y_i}\leftarrow
\neg  Y_i \mid Y_i\in Y\}, \label{y-rules}\\
\Pi_S &=& \{  S \gets \neg \bar{S},\; \bar{S} \gets \neg S \}, \label{s-rules}\\
\Pi_\phi &=& \{ A \leftarrow  \bar{S} \lor \neg \phi \mid A \in Z\cup\{U\} \} \cup \label{phi-rules} \\
        && \{ U \leftarrow \neg U, \bar{S} \}\cup \label{sat-rule}\\
        && \{ V \leftarrow \naf V, \naf S\} \label{kill-rule},
\end{eqnarray}
and let $\Pi = \Pi_X \cup \Pi_y \cup \Pi_S \cup \Pi_\phi$.
Informally, the rules (\ref{x-rules}) guess an assignment $\sigma$ to the
$X$ atoms, which gives rise to the set $\Phi_{\mathcal{A}}$ of an
epistemic model $\mathcal{A}$ as a candidate for condition~1 of a
rational world view. The rules (\ref{y-rules}) resp.\  (\ref{s-rules})
guess an  assignment $\mu$ to the
$Y$ atoms resp.\ to $S$. Each such assignment to $Y$ gives rise to a
rational answer set (i.e., WJ answer set) of the epistemic reduct
$\Pi^{\mathcal{A}}$ in which $S$ is true, and possibly to another
one in which is $S$ is false (thus $\bar{S}$ is true) and $U$ and all
atoms in $Z$ are true; this is the case iff $\forall
Z.\neg \phi(\sigma(X),\mu(Y),Z)$ evaluates to true. The rule (\ref{kill-rule})
is a constraint that eliminates $\mathcal{A}$  iff $\naf S$ evaluates
to true, i.e., if there is some WJ answer set $I$ in
$\mathcal{A}$ such that $S\notin I$; thus $\mathcal{A}$ survives iff
$\forall Y\exists Z.\phi(\sigma(X),Y,Z)$ evaluates to true.
Consequently, $\Pi$ has some EWJ world view (tantamount a rational
world view) iff the QBF (\ref{qbf3}) evaluates to true.
\end{proof}

\noindent $\hookrightarrow$\  {Proof of Theorem~\ref{th-EWJ-inference}}

\RthEWJinference*

\begin{proof}
Membership in $\Sigma^p_4$ follows from
Lemma~\ref{lem-epi-rat-membership}.(ii) and $\Sigma^p_2$ / $\Pi^p_2$-
completeness of brave / cautious inference for normal epistemic-free
programs (row~1 of Table~\ref{tab:complexity2}).

The $\Sigma^p_4$-hardness is shown by generalizing the reduction
in Theorem~\ref{th-EWJ-existence-normal-atomic} to encode the
evaluation of a QBF 
\begin{equation}
\label{qbf4}
\exists W\forall X \exists Y \forall Z.\psi
\end{equation}
where $\psi$ is a propositional formula over $W\cup X\cup Y\cup Z$,
following the steps in the analog generalization for EFLP world views
in the proof of Theorem~6 in \cite{ShenEiter16}.

Assuming that $W$ is void, $\neg \forall X \exists Y \forall Z.\psi =
\exists X\forall Y\exists Z.\phi$ where $\phi = \neg \psi$, is encoded by the
program $\Pi$ in the proof of
Theorem~\ref{th-EWJ-existence-normal-atomic}, 
i.e., some rational world view (tantamount, EWJ-world view) $\mathcal
A$ exists  iff $\exists X\forall Y\exists Z.\phi$ evalutes to true;
$\Phi\subseteq EP(\Pi)$ for $\mathcal{A}$ must be of the form $\Phi=\{ \naf X_i, \naf \bar{X_j}
\mid X_i \in X', X_j \in X\setminus X'\} \cup \{\naf V\}$ for some $X'\subseteq X$.

A modification of $\Pi$ yields that 
maximality testing of a guess $\Phi\subseteq EP(\Pi)$ for $\mathcal{A}$ is $\Pi^p_3$-hard.
To this end, let 
\begin{eqnarray}
\Pi_0  &=& \{ H \leftarrow B \land A \mid H \leftarrow B \in \Pi\}\label{normal-tag-pi-rule}
\cup \\
        && \{ X_i \leftarrow \neg A;\ \ \bar{X_i} \leftarrow \neg A
\mid X_i \in X \} \cup \{ V \leftarrow \neg A \} \cup \label{sat-X}
\\
 && \{ A \leftarrow \naf \neg A \}, \label{choose-A}
\end{eqnarray}
where $A$ is a fresh atom, and $\Phi_0= \emptyset \subseteq EP(\Pi)$.
Then $\mathcal{A} = \{ I_0\}$ where, $I_0 = X\cup \bar{X} \cup \{ V\}$
is the single WJ answer set
of ${\Pi_0}^{\Phi_0}$, and for each epistemic negation $\naf
F$ occurring in $\Pi_0$, we have that $I\models F$. Thus
$\mathcal{A}$ satisfies condition 1 of an EWJ world view. Condition~2
(maximality) is not satisfied only if $\naf \neg A$ in
$\Phi'$ for any $\Phi'$ for $_{\mathcal{A}'}$ such that $\Phi'\supset \Phi_0$; but in this case, the rules (\ref{normal-tag-pi-rule})
are enabled, and $\mathcal{A}'$ amounts to a world view of
$\Pi$ in the proof of Theorem~\ref{th-EWJ-existence-normal-atomic},
with $A$ added in each answer set; the rules  
(\ref{sat-X}) do not fire, and $A$ becomes a fact (cf.\ (\ref{choose-A})).

To accommodate non-void $W$, i.e., to encode
evaluating the QBF (\ref{qbf4}), we let 
$$
\Pi_1 = \Pi_0 \cup \{ W_i \leftarrow \naf \bar{W_i};\ \ \bar{W_i}
\leftarrow \naf W_i \mid W_i \in W \} 
$$
where $\bar{W_i}$ is a fresh atom for $W_i$.

As $W_i,\bar{W}_i$ do not occur in rule heads of $\Pi_0$, for every 
EWJ world view $\mathcal{A}$ of $\Pi$ the set $\Phi_\mathcal{A}$ must contain exactly 
one of $\naf W_i$, $\naf \bar{W}_i$, i.e.,  $\Phi_\mathcal{A}$ encodes
a truth assignment $\sigma$ to $W$. Furthermore, $\naf \neg A \in
\Phi_\mathcal{A}$ holds iff the QBF $\exists X\forall Y \exists
Z.E(\sigma(W),X,Y,Z)$ evaluates to true. As every truth assignment
$\sigma$ to $W$, is encoded by some EWJ world view $\mathcal{A}$ of $\Pi_1$, it
follows that $F = A$ is true in some EWJ world view (tantamount
rational world view) of $\Pi$ iff
the QBF (\ref{qbf4}) evaluates to true.
\end{proof}

\noindent $\hookrightarrow$\  {Proof of Theorem~\ref{th-EWJ-simple-disjunctive-existence}}

\RthEWJsimpledisjunctiveexistence*
\begin{proof}
Membership in $\Sigma^p_3$ follows from
Lemma~\ref{lem-epi-rat-membership}.(i) and $\Sigma^p_2$ / $\Pi^p_2$-
completeness of brave / cautious inference for simple disjunctive
epistemic-free programs
(row~1 of Table~\ref{tab:complexity2}).

The $\Sigma^p_3$-hardness is shown by a minor adaptation of the proof
of $\Sigma^p_3$-hardness of deciding EFLP world view existence for
normal epistemic programs in Theorem~5 of \cite{ShenEiter16}. The
program $\Pi = \Pi_x\cup \Pi_Y\cup \Pi_Z \cup \Pi_\phi$  constructed there to encode the evaluation of a QBF
$\exists X\forall Y\exists Z.\phi$ is a simple normal program (with
atomic heads) except for $\Pi_Z = \{ Z_i \lor \bar{Z_i} \mid
Z_i\in Z\}$, which serves to generate all truth assignments to $Z$
in different FLP answer sets. We can equivalently rewrite
each $Z_i \lor \bar{Z_i}$ to $Z_i \mid \bar{Z_i}$, such that the 
DI-FLP world views of the resulting disjunctive program $\Pi'$
coincide with the EFLP world views of $\Pi$. Since $\Pi'$ is simple, 
the DI-FLP world views of $\Pi'$ coincide with its DI-WJ world views
(each disjunctive program reduct is a simple normal program), for which
WJ and FLP answer sets coincide); as $\Pi'$ has no variant heads, the 
rational world views of $\Pi'$ coincide then with the EFLP-world views of
$\Pi$; this establishes the result.  
\end{proof}

\noindent $\hookrightarrow$\  {Proof of Theorem~\ref{th-EWJ-simple-disjunctive-inference}}

\RthEWJsimpledisjunctiveinference*
\begin{proof}[Sketch]
Membership in $\Sigma^p_4$ follows from
Lemma~\ref{lem-epi-rat-membership}.(ii) and $\Sigma^p_2$ / $\Pi^p_2$-
completeness of brave / cautious inference for simple disjunctive
epistemic-free programs (row~1 of Table~\ref{tab:complexity2}).

The $\Pi^p_4$-hardness is shown by adapting the program
$\Pi_1$ in the proof of Theorem~6 of \cite{ShenEiter16}  which showed
$\Sigma^p_4$-hardness of deciding whether an atom $F$ is true in
some EFLP world view of a normal epistemic program. The program
$\Pi_1$ was constructed from the program $\Pi$ in the proof of Theorem~5 of
\cite{ShenEiter16}, which has an EFLP world view iff a given QBF 
\begin{equation}
\label{qbf:XYZ}
\exists X\forall Y\exists Z.\phi(X,Y,Z)
\end{equation}
evaluates to true, as follows:
\begin{eqnarray}
\Pi_1  &=& \{ H \leftarrow B^* \land A \mid H \leftarrow B \in \Pi\} 
\cup \\
        && \{ X_i \leftarrow \neg A;\ \ \bar{X_i} \leftarrow \neg A
\mid X_i \in X \} \cup \{ V \leftarrow \neg A \} \cup \\
 && \{ A \leftarrow \naf \neg A \} \cup \\
&& \{ W_i \leftarrow \naf \bar{W_i};\ \ \bar{W_i}
\leftarrow \naf W_i \mid W_i \in W \} 
\end{eqnarray}
where $A$ and all $\bar{W_i}$ are fresh atoms, 
and $B^*$ denotes $B$ with each $W_i$ (resp.\ $\neg W_i$) literal replaced by
$\bar{W_i}$ (resp.\ $W_i$). Then $A$ is true in some EFLP-world view of
$\Pi_1$ iff the QBF (\ref{qbf:XYZ}) evaluates to true. As in the proof of Theorem~\ref{th-EWJ-simple-disjunctive-existence}, 
we simply rewrite each $Z_i \lor \bar{Z_i}$ in $\Pi_1$ (i.e., the
rules in $\Pi_Z$) to $Z_i \mid \bar{Z_i}$; then the EFLP world views of
$\Pi_1$ correspond to the DI-WJ world views (tantamount, rational
world views) of the modified program $\Pi_2$. 
\end{proof}

\noindent $\hookrightarrow$\  {Proof of Theorem~\ref{th-EWJ-disjunctive-existence}}

\RthEWJdisjunctiveexistence*
\begin{proof}
Membership in $\Sigma^p_4$ follows from
Lemma~\ref{lem-epi-rat-membership}.(i) and $\Sigma^p_3$- / $\Pi^p_3$-
completeness of brave / cautious inference for disjunctive epistemic-free programs.

The $\Sigma^p_4$-hardness is shown by extending the encoding of
MINQASAT in Theorem~\ref{th-DJ-WJ-atomic-heads-hardness}
into brave inference of an atom $A$ from a epistemic-free disjunctive
program $\Pi$ w.r.t.\ rational answer sets, to the complement of the following
generalized problem A-MINQASAT: Given a QBF $\forall Y.E(W,X,Y)$, over
disjoint sets $W$, $X$, and $Y$, and an atom $A\in X$, decide whether for each assignment
$\sigma$ to $W$ it holds that the pair $(\forall Y.E(\sigma(W),X,Y), A)$ is a
yes-instance of MINQASAT.  By a simple generalization of the proof of Lemma~2 in
\cite{ShenE19}, which shows $\Sigma^p_3$ completeness of A-MINQASAT for
the case of void $W$, we obtain:

\begin{lemma}
\label{lem:A-MINQASAT}
A-MINQASAT is $\Pi^p_4$-complete.
\end{lemma}

We then adapt the MINQASAT encoding to the following program $\Pi'$,
using techniques from the proof of Theorem~5 in \cite{ShenEiter16}:
\begin{eqnarray}
   W_j & \gets &
   \parbox{0.4\textwidth}{$\naf \hat{W_j}, \quad W_j \gets \naf {W_j}$}
   W_j \in W \label{Gchoose-W} \\  
   X_i \mid \hat{X_i} & \leftarrow &    \parbox{0.4\textwidth}{~}  X_i \in X \label{Gchoose-X} \\   
   B & \leftarrow &  \parbox{0.4\textwidth}{E(X,Y)} B\in \hat{X}Y\cup\{p\} \label{Gaaa} \\
   p & \leftarrow & \neg p \label{Gkill}\\
   V & \leftarrow & \naf V, \naf \neg A \label{Gwv-kill}
\end{eqnarray}
Informally, the rules (\ref{Gchoose-W}) choose an assignment $\sigma$ to the
atoms in $W$, and the rules (\ref{Gchoose-X})-(\ref{Gkill}) single out
minimal assignments $\mu$ to $X$ such that $\forall
Y.E(\sigma(W),\mu(X),Y)$ evaluates to true. The rule
(\ref{Gwv-kill}) checks whether in every such minimal $\mu$ the atom $A$
is false; without loss of generality, we may assume that such a $\mu$
always exists. Thus, a rational world view for $\Pi'$ will exist  iff the
A-MINQASAT instance is a no-instance. This proves the result.
\end{proof}

\noindent $\hookrightarrow$\  {Proof of Theorem~\ref{th-EWJ-disjunctive-atomic-inference}}

\RthEWJdisjunctiveatomicinference*
\begin{proof} (Sketch)
Membership in $\Sigma^p_5$ follows from
Lemma~\ref{lem-epi-rat-membership}.(ii) and $\Sigma^p_3$ / $\Pi^p_3$-
completeness of brave / cautious inference for disjunctive
epistemic-free programs
(row~1 of Table~\ref{tab:complexity2}).

The $\Sigma^p_5$-hardness proof is analogous to the
$\Sigma^p_4$-hardness proof of inference from EFLP world views in
Theorem~6  of \cite{ShenEiter16}. We construct from the program $\Pi'$
in the proof of Theorem~\ref{th-EWJ-disjunctive-existence} the
following program $\Pi''$, which encodes  deciding the 
generalization E-A-MINQASAT of  A-MINQASAT, where
a pair of a QBF $(\forall Y.E(Z,W,X,Y),A)$ and an atom $A$ is a yes-instance iff 
for some assignment $\sigma$ to $Z$, the pair $(\forall
Y.E(\sigma(Z),W,X,Y),A)$ is a yes-instance of A-MINQASAT; thus 
\begin{eqnarray}
\Pi_0  &=& \{ H \leftarrow B \land A' \mid H \leftarrow B \in \Pi\}\label{tag-pi-rule}
\cup \\
        && \{ Z_j\leftarrow \neg A';\ \ \bar{Z_j} \leftarrow \neg A'
\mid Z_j \in Z \} \cup \{ V \leftarrow \neg A' \} \cup \\
 && \{ A' \leftarrow \naf \neg A' \},
\end{eqnarray}
where $A'$ is a fresh atom.  Then it holds that $A'$ is true in some rational
world view of $\Pi''$ iff $(\forall Y.E(Z,W,X,Y),A)$ is a
yes-instance of E-A-MINQASAT. The latter is $\Sigma^p_5$-complete to
decide, which can be readily shown by a generalization of problem A-MINQASAT.
\end{proof}

\subsection{Computing a witnessing strict partial order}
\label{WitnessingPorder}
\begin{theorem}
\label{th-comp-wellsupport-lfp-normal-order}
Given a model $I$ of an epistemic-free
normal program $\Pi$, computing the strict partial order
$\cal O$ for $I$ as in 
Algorithm~\ref{algo:well-support-with-partialorder-2} 
witnessing that $I$ is well-supported in $\Pi$
is \FPNPPar-complete. Moreover, 
\FPNPPar-hardness holds even if $I$ is well-supported and all rule
heads are atomic.
\end{theorem}
\begin{proof}
As for membership in \FPNPPar, by Theorem~\ref{th-comp-wellsupport-lfp-normal} we
can decide with an \NP{} oracle call whether
Algorithm~\ref{algo:well-support-with-partialorder-2} will output 
some partial order $\cal O$ or not. With an argument similar to the one for
Lemma~\ref{lem:guess}, we can show that deciding for each rule $r$ whether $head(r)$ is
put in $S_{j+1}$, is for each $0\leq j \leq |{\cal R}|$ in \coNP;
we may put equivalent heads (i.e.\ $head(r) \equiv head(r')$) into
$S_{j+1}$ without harm.
All these \coNP{} problems can be evaluated in parallel.
Once
$S_0,S_i,\ldots,S_i$ are known, we can determine by using a further round of
parallel \NP-oracle calls all precedences $p \prec q$ that have to be
added to $\cal O$.  Thus, with two round of parallel \NP\ oracle
queries, we can compute the strict partial order $\cal O$ for
$I$. Since two parallel rounds of such queries can be emulated with a
single round, it follows that computing $\cal O$ is in \FPNPPar.

The \FPNPPar-hardness is shown by a reduction from computing the
answers to given SAT instances $E_1,\ldots,E_n$ on pairwise disjoint
sets $X_1,\ldots, X_n$ of variables. 
Let $X = \bigcup_{i=1}^n X_i$ and $x \in X$ be some arbitrary variable. 
Construct the following program: 
\begin{align*}
\Pi = & \{  x \leftarrow ~\} \cup \{ p \leftarrow x; 
            \ \ p \leftarrow \neg E_i  \mid 1 \leq i \leq n, p \in X \}.
\end{align*}
Then $I = X$ is a model of $\Pi$, and moreover it is
well-supported; we have in Algorithm~\ref{algo:well-support-with-partialorder-2} 
$S_0 = \emptyset$, $S_1 = \{ x \} \cup
\bigcup_{\models \neg E_i} X_i$, and $S_2 = X = I$; 
Furthermore, for $p\in X$ such that $x\neq p$, we shall have $x \prec p$ in the constructed $\cal O$ 
iff $E_i$ is
unsatisfiable where $p\in X_i$. From $\cal O$, we can easily extract
the answers to the SAT instances $E_1,\ldots,E_n$. This proves
\FPNPPar-hardness under the stated restrictions.
\end{proof}

\begin{theorem}
\label{th-comp-wellsupport-lfp-disjunctive-order}
Given a model $I$ of an epistemic-free
disjunctive program $\Pi$, computing a head selection function $\sel$
and the strict partial order
$\cal O$ given $I$ and $\Pi^I_{\sel}$ as in 
Algorithm~\ref{algo:well-support-with-partialorder-2} 
witnessing that $I$ is well-supported in $\Pi$
is (i) MV\NP-complete if $\Pi$ is simple, and 
(i) MV$\Sigma^p_2$-complete if in general $\Pi$, where 
MV$\Sigma^p_2$-hardness holds even if all rule head formulas are atomic.
\end{theorem}

\begin{proof}
Membership in MV\NP{} respectively in
MV$\Sigma^p_2$ follows by a guess and check argument using
Algorithm~\ref{algo:well-support-with-partialorder-2} for the
checking part. For the hardness parts, we can use reductions from
computing some satisfying assignment to a Boolean
formula $E(X)$ on variables $X$, which is MV\NP-complete, and from computing some
assignment $\sigma$ to the variables $X$ of a QBF $\forall Y E(X,Y)$,
which is complete for MV$\Sigma^p_2$ (the completeness of these
problems can be established by Turing machine encodings). The program $\Pi$ constructed 
in the proof of Theorem~\ref{th-comp-wellsupport-disjunctive-simple} (resp.\ Theorem~\ref{th-comp-wellsupport-disjunctive-general}) encodes evaluating a SAT
instance $E$ (resp.\ a QBF $\exists X\forall Y\,E(X,Y)$).
Clearly, for a model $I$ that is well-supported in $\Pi$, any witnessing
pair $(\sel,{\cal O})$ of a selection function $\sel$ and a partial
order $\cal O$ computed by
Algorithm~\ref{algo:well-support-with-partialorder-2} corresponds to 
a satisfying assignment of $E(X)$ resp.\ a solution for $\forall
Y\,E(X,Y)$ and vice versa. Since all rule head formulas are atomic, this establishes the
hardness parts under the stated restriction.
\end{proof}

\end{document}